%% file: main.tex
\documentclass[sigconf, nonacm]{acmart}
\newcommand\vldbdoi{XX.XX/XXX.XX}
\newcommand\vldbpages{XXX-XXX}
\newcommand\vldbvolume{19}
\newcommand\vldbissue{1}
\newcommand\vldbyear{2026}
\newcommand\vldbauthors{\authors}
\newcommand\vldbtitle{\shorttitle} 
\newcommand\vldbavailabilityurl{URL_TO_YOUR_ARTIFACTS}
\newcommand\vldbpagestyle{plain} 

\usepackage{amsmath}    
\usepackage{bm}        
\usepackage{amsthm}    
\usepackage{amssymb}   

\usepackage{graphicx}               
\usepackage{float}                  
\usepackage{subfigure}              
\usepackage{multirow}               
\usepackage[ruled,linesnumbered]{algorithm2e} 
\usepackage{booktabs}               
\usepackage{nicefrac}               
\usepackage{microtype}              
\usepackage{xcolor}                 
\usepackage[most]{tcolorbox}        
\usepackage{wrapfig}                
\usepackage{picinpar}               
\usepackage{cutwin}                 
\usepackage{fancyhdr}               
\usepackage{url}                    
\usepackage{verbatim}               

\usepackage[utf8]{inputenc}         
\usepackage[T1]{fontenc}            

\usepackage{hyperref}               
\usepackage{cleveref}               
\usepackage{enumitem}

\title{Response Letter for PVLDBv19 Submission 240 "Theoretically and Practically Efficient Resistance Distance Computation on Large Graphs" (shepherd submission)}

\begin{document}

\input{command}

\title{Theoretically and Practically Efficient Resistance Distance Computation on Large Graphs}

\makeatletter
\newenvironment{breakablealgorithm}
{
	\begin{center}
		\refstepcounter{algorithm}
		\hrule height.8pt depth0pt \kern2pt
		\renewcommand{\caption}[2][\relax]{
			{\raggedright\textbf{\ALG@name~\thealgorithm} ##2\par}%
			\ifx\relax##1\relax 
			\addcontentsline{loa}{algorithm}{\protect\numberline{\thealgorithm}##2}%
			\else 
			\addcontentsline{loa}{algorithm}{\protect\numberline{\thealgorithm}##1}%
			\fi
			\kern2pt\hrule\kern2pt
		}
	}{
	\kern2pt\hrule\relax
\end{center}
}

\author{Yichun Yang}
\affiliation{%
  \institution{Beijing Institute of Technology}
  \city{Beijing}
  \country{China}
}\email{yc.yang@bit.edu.cn} 
\author{Longlong Lin}
\affiliation{%
  \institution{Southwest University}
  \city{Chongqing}
  \country{China}
}\email{longlonglin@swu.edu.cn} 
\author{Rong-Hua Li}
\affiliation{%
  \institution{Beijing Institute of Technology}
  \city{Beijing}
  \country{China}
}\email{lironghuabit@126.com} 
\author{Meihao Liao}
\affiliation{%
  \institution{Beijing Institute of Technology}
  \city{Beijing}
  \country{China}
}\email{mhliao@bit.edu.cn}
\author{Guoren Wang}
\affiliation{%
  \institution{Beijing Institute of Technology}
  \city{Beijing}
  \country{China}
}\email{wanggrbit@gmail.com}

\makeatother

\begin{abstract}

The computation of resistance distance is pivotal in a wide range of graph analysis applications, including graph clustering, link prediction, and graph neural networks. Despite its foundational importance, efficient algorithms for computing resistance distances on large graphs are still lacking. Existing state-of-the-art (SOTA) methods, including power iteration-based algorithms and random walk-based local approaches, often struggle with slow convergence rates, particularly when the condition number of the graph Laplacian matrix, denoted by $\kappa$, is large. To tackle this challenge, we propose two novel and efficient algorithms inspired by the classic Lanczos method: Lanczos Iteration and Lanczos Push, both designed to reduce dependence on  $\kappa$. Among them, Lanczos Iteration is a near-linear time global algorithm, whereas Lanczos Push is a local algorithm with a time complexity independent of the size of the graph. More specifically, we prove that the time complexity of Lanczos Iteration is $\tilde{O}(\sqrt{\kappa} m)$ ($m$ is the number of edges of the graph and $\tilde{O}$ means the complexity omitting the $\log$ terms) which achieves a speedup of $\sqrt{\kappa}$ compared to previous power iteration-based global methods. For Lanczos Push, we demonstrate that its time complexity is $\tilde{O}(\kappa^{2.75})$ under certain mild and frequently established assumptions, which represents a significant improvement of $\kappa^{0.25}$ over the SOTA random walk-based local algorithms. We validate our algorithms through extensive experiments on eight real-world datasets of varying sizes and statistical properties, demonstrating that Lanczos Iteration and Lanczos Push significantly outperform SOTA methods in terms of both efficiency and accuracy.

\end{abstract}

\maketitle

\pagestyle{\vldbpagestyle}
\begingroup\small\noindent\raggedright\textbf{PVLDB Reference Format:}\\
\vldbauthors. \vldbtitle. PVLDB, \vldbvolume(\vldbissue): \vldbpages, \vldbyear.\\
\href{https://doi.org/\vldbdoi}{doi:\vldbdoi}
\endgroup
\begingroup
\renewcommand\thefootnote{}\footnote{\noindent
This work is licensed under the Creative Commons BY-NC-ND 4.0 International License. Visit \url{https://creativecommons.org/licenses/by-nc-nd/4.0/} to view a copy of this license. For any use beyond those covered by this license, obtain permission by emailing \href{mailto:info@vldb.org}{info@vldb.org}. Copyright is held by the owner/author(s). Publication rights licensed to the VLDB Endowment. \\
\raggedright Proceedings of the VLDB Endowment, Vol. \vldbvolume, No. \vldbissue\ %
ISSN 2150-8097. \\
\href{https://doi.org/\vldbdoi}{doi:\vldbdoi} \\
}\addtocounter{footnote}{-1}\endgroup

\ifdefempty{\vldbavailabilityurl}{}{
\vspace{.3cm}
\begingroup\small\noindent\raggedright\textbf{PVLDB Artifact Availability:}\\
The source code, data, and/or other artifacts have been made available at \url{https://github.com/Ychun-yang/LanczosPush}.
\endgroup
}

\section{Introduction}\label{sec:intro}
Resistance distance (RD) computation is a fundamental task in graph data management, offering a powerful means to quantify the node similarities of a graph. Given an undirected graph $\mathcal{G}$, the resistance distance  $r_{\mathcal{G}}(s,t)$ of two vertices $s$ and $t$ is proportional to the expected number of steps taken by a random walk starting at $s$, visits $t$ and then comes back to $s$. Consequently, a small $r_{\mathcal{G}}(s,t)$ indicates a high similarity between $s$ and $t$ \cite{liao2023resistance,yang2023efficient}. By effectively quantifying node similarities, RD emerges as a highly versatile and potent tool, finding numerous real-world applications, including graph clustering ~\cite{shi2014clustering}, link prediction ~\cite{sarker2008proximity}, graph sparsification ~\cite{spielman2008graph}, maximum flow computation \cite{van2022faster,madry2016computing}, and graph neural networks ~\cite{black2023understanding,DBLP:conf/mir/YuLLWOJ24}. Therefore, designing fast algorithms for RD computation is a natural problem that has attracted extensive attention in recent years ~\cite{peng2021local,yang2023efficient,liao2023resistance,li2023new,liao2024efficient,dwaraknath2024towards,yang2025improved}.
\comment{
Given an undirected and unweighted graph $\mathcal{G}$, the Resistance Distance (RD) $r_{\mathcal{G}}(s,t)$ of two vertices $s$ and $t$ is defined as the potential difference between $s$ and $t$ when a unit electric flow is sent from $s$ to $t$. Intuitively, the small  $r_{\mathcal{G}}(s,t)$ indicates the high similarity between $s$ and $t$ since the potential difference is small when sending a unit flow from $s$ to $t$. Due to this nice structure and strong interpretability, RD has numerous applications, including the design of maximum flow algorithms ~\cite{van2022faster,madry2016computing}, graph clustering ~\cite{alev2017graph,saito2023multi}, graph sparsification ~\cite{spielman2008graph}, and alleviating oversquashing in  Graph Neural
Networks (GNNs) ~\cite{black2023understanding,topping2021understanding}. 
}


However, computing exact $r_{\mathcal{G}}(s,t)$ is equivalent to solving a linear system, which is only known to be solvable in $O(n^{2.37})$ time by the complex matrix multiplication ~\cite{williams2012matrices}, which is inefficient for massive graphs. Thus, many approximate algorithms have been developed recently for estimating  $r_{\mathcal{G}}(s,t)$ with absolute error guarantee ~\cite{peng2021local,liao2023resistance,liao2024efficient,yang2023efficient} (details in Definition \ref{def:ER-err}) to balance the efficiency and quality. Existing approximate methods can be roughly classified into two categories: (i) algorithms with nearly linear time complexity; (ii) local algorithms that do not depend on graph size.  For example, Cohen et al.~\cite{cohen2014solving} developed a global Laplacian solver capable of achieving a $\epsilon$-absolute error approximation within a time complexity of $O(m\log^{1/2} n \log \frac{1}{\epsilon})$, where  $n$ (resp., $m$) is the number of vertices (resp., edges) of the graph and $\epsilon$ is the absolute error parameter. However, though the Laplacian solver is theoretically very powerful, it remains practically inefficient for real-world datasets due to the complex data structures and large constant and $\log$ terms hidden in the complexity. Thus, the simpler Power Method ~\cite{ron2019sublinear,peng2021local,yang2023efficient} with an $\epsilon$-absolute error approximation is proposed, which can be implemented easier than the Laplacian solver. More specifically, the time complexity of Power Method ~\cite{ron2019sublinear,peng2021local,yang2023efficient} is  $O(\kappa m\log \frac{\kappa}{\epsilon})$, where $\kappa$ is the condition number of the normalized Laplacian matrix $\mathcal{L}$ (details in Sec. \ref{sec:preliminaries}). However, the Power Method remains computationally prohibitive for graphs with large condition number $\kappa$ due to its inherent linear dependency on $\kappa$ in time complexity. Even in the optimal case of  $\kappa=O(1)$ (i.e., expander graphs ~\cite{dwaraknath2024towards}), the time complexity of Power Method reduces to  $O(m\log \frac{1}{\epsilon})$. However, this reduction still has to perform $O(\log \frac{1}{\epsilon})$  matrix-vector multiplication operations, rendering it inefficient for massive graphs since such each operation needs to traverse the entire graph. 


To further boost the efficiency, researchers have recently turned to local algorithms\footnote{In this paper, the term "local" means the runtime of the algorithm independent or sublinear dependent on graph size $m$.} for estimating RD ~\cite{peng2021local}. These algorithms operate by examining only localized subgraphs rather than requiring full graph access, thereby significantly reducing computational overhead. Informally, we consider the well-known adjacency model ~\cite{ron2019sublinear}, which supports the following three types of queries on graph $\mathcal{G}$ in constant time: (i) degree query; (ii) neighbor query; (iii) jump query. The local algorithms aim to approximate $r_{\mathcal{G}}(s,t)$ by making as few queries as possible without searching the whole graph. For example, Peng et al. ~\cite{peng2021local} proposed the random walk algorithm for local computation of $r_{\mathcal{G}}(s,t)$. Yang et al. ~\cite{yang2023efficient} further improved this result by combining the Power Method with the random walk. However, the time complexity of these local algorithms are $\tilde{O}(\kappa^3/\epsilon^2)$\footnote{$\tilde{O}(.)$ means the complexity omitting the $\log$ term of $n,m$ and $\epsilon$.} for estimating the $\epsilon$-absolute error approximation of $r_{\mathcal{G}}(s,t)$. These results show that although the complexity of local algorithms does not depend on $n$ and $m$, it comes at the cost of higher dependence on the condition number $\kappa$ and parameter $\epsilon$. Unfortunately, the condition number $\kappa$ can be very large in real-world graphs (Table \ref{tab:dataset}). This makes the $\kappa^3$ bound very high. As shown in our experiments, we also find that existing algorithms for RD computation can only be efficiently implemented on datasets with a relatively small $\kappa$ (e.g., social networks), but their performance on datasets with a relatively large $\kappa$ (e.g., road networks) is unsatisfactory. 

Therefore, a natural and important question is proposed: Can we design efficient linear time and sublinear time (local) algorithms for resistance distance computation, while the dependence on the condition number $\kappa$ is weaker? In this paper, we provide a positive answer to this question. First, we utilize the classic Lanczos Iteration algorithm to provide a global RD algorithm. Specifically, the Lanczos Iteration algorithm (details in Sec. \ref{sec:LI}) search a sequence of orthogonal vectors $\mathbf{v}_1,...,\mathbf{v}_k$ lie in the $k$-dimension Krylov subspace $\mathcal{K}_k(\mathcal{A},\mathbf{v}_1)=span\langle\mathbf{v}_1, \mathcal{A}\mathbf{v}_1,..., \mathcal{A}^k\mathbf{v}_1\rangle$ with $k=\sqrt{\kappa }\log \frac{\kappa}{\epsilon}$, where $\kappa$ denotes the condition number of the normalized Laplacian matrix $\mathcal{L}$. We prove that one can then approximate $r_\mathcal{G}(s,t)$ with $\epsilon$-absolute error guarantee by these orthogonal vectors $\mathbf{v}_1,...,\mathbf{v}_k$ when setting $\mathbf{v}_1=\left(\frac{\mathbf{e}_s}{\sqrt{d_s}}-\frac{\mathbf{e}_t}{\sqrt{d_t}}\right)/ \sqrt{\frac{1}{d_s}+\frac{1}{d_t}}$, where $d_s,d_t$ denote the degree of the two specific nodes $s,t$. Next, by the implementation of Lanczos Iteration, $\mathbf{v}_1,...,\mathbf{v}_k$ can be computed in $k$ matrix-vector queries. Therefore, this yields an $O(\sqrt{\kappa}m\log \frac{\kappa}{\epsilon})$ time algorithm for RD estimation, improves upon the Power Method by reducing a $\sqrt{\kappa}$ dependency. Subsequently, to obtain a local algorithm with less than $\kappa^3$ complexity, we introduce a novel concept of the local Lanczos recurrence (details in Eq. (\ref{equ:lz_recurrence})). Using our novel techniques, we design the Lanczos Push algorithm (details in Sec. \ref{sec:LP}), which computes a sequence of sparse vectors $\hat{\mathbf{v}}_1,...,\hat{\mathbf{v}}_k$ that is close to the orthogonal vectors $\mathbf{v}_1,...,\mathbf{v}_k$ computed by Lanczos Iteration. As a consequence, our Lanczos Push algorithm estimates $r_\mathcal{G}(s,t)$ by only search only a small portion of the graph but still maintaining the acceleration effect by the Lanczos iteration. Under several mild and frequently-established  assumptions (also have been verified in our experiments), we theoretically prove that our Lanczos Push algorithm has a time complexity $\tilde{O}(\kappa^{2.75}C_1C_2/\epsilon)$ to obtain an absolute error approximation of the RD value, which is subcubic in $\kappa$. For the two numbers $C_1$ and $C_2$, we refer to Table \ref{tab:alg} and Sec. \ref{sec:LP} for details. On the other hand, we prove the lower bound of the time complexity is $\Omega(\kappa)$ for any algorithms to approximate $r_{\mathcal{G}}(s,t)$ with absolute error $\epsilon\leq 0.01$. This indicates that the time complexity of the local algorithms is indeed relevant to the condition number $\kappa$, which has not been emphasized in previous works for local RD computation\footnote{There is still a large gap between the upper and lower bounds of the dependence of $\kappa$, so we leave how to match the upper and lower bounds as future work.}. 

We conduct extensive experiments  to evaluate the effectiveness and scalability of the
proposed solutions. The empirical results show that the performance of the proposed Lanczos Iteration and Lanczos Push algorithm substantially outperform all baselines. In particular, Lanczos Iteration performs $5\times$ faster than the Power Method in social networks and $100\times$ faster than Power Method in road networks. Lanczos Push is $5\times$ to $10\times$ faster than all the state-of-the-art algorithms in social networks and $50\times$ faster than all other algorithms in road networks (including Lanczos Iteration). In short, our algorithms are both theoretically and practically efficient.

\comment{

\begin{table*}[t!]
	\centering
	\caption{A comparison of different algorithms for computing pairwise RD with $\epsilon$-absolute error guarantee} 
	\scalebox{1}{
		\begin{tabular}{c|c|c|c|c}
			\toprule
			\multicolumn{1}{c|}{Methods} & \multicolumn{1}{c|}{Algorithm}&\multicolumn{1}{c|}{Time Complexity}&
			\multicolumn{1}{c|}{Techniques}&\multicolumn{1}{c}{Assumptions}\\
			\midrule
			\multirow{3}{*}{Global}& Power Method & $O(\kappa m \log \frac{\kappa}{\epsilon})$ & \powermethod & $\times$\\
             & LapSolver ~\cite{sachdeva2014faster} & $O( m \log ^{1/2} n \log \frac{1}{\epsilon})$ & Precondition + Conjugate Gradient & $\times$\\
	       & Lanczos (section 4) & $O (\sqrt{\kappa} m \log \frac{\kappa}{\epsilon})$ & Lanczos Recurrence & $\times$\\
           \midrule
           \multirow{5}{*}{Local}& TP ~\cite{peng2021local}& $\tilde{O}(\kappa^4/\epsilon^2)$&	\rw & $\times$\\
			  & TPC ~\cite{peng2021local}& $\tilde{O}(\kappa^3/\epsilon^2)$& \rw & $\times$\\
                & GEER ~\cite{yang2023efficient}& $\tilde{O}(\kappa^3/(\epsilon^2 d^2))$& \powermethod+\rw &$\times$\\
                & Lanczos Push (section 4)& $\tilde{O}(\kappa^{2.75}n/\epsilon)$& Local Lanczos Recurrence & Assumption 1\\
                & Lanczos Push (section 4)& $\tilde{O}(\kappa^{2.75}/\epsilon)$& Local Lanczos Recurrence & Assumption 1 $\&$ 2\\
            \midrule
             \multirow{2}{*}{Others}& \push ~\cite{liao2023resistance}& $\times$ & \push & $\times$\\
             & \bipush ~\cite{liao2023resistance} & $\times$ &  \push+\rw & $\times$\\
            \midrule
             & Lower Bound (section 5) & $\Omega(\kappa)$ &  based on ~\cite{cai2023effective} & $\times$\\
            \bottomrule	
		\end{tabular}
	}
\end{table*}
}

\section{Preliminaries}\label{sec:preliminaries} 

\subsection{Notations and Concepts} \label{subsec:ER}

Consider an undirected and connected graph $\mathcal{G}=(\mathcal{V},\mathcal{E})$ with $|\mathcal{V}|=n$ vertices and $|\mathcal{E}|=m$ edges. In this paper, we primarily focus on unweighted graphs. However, all the algorithms and analysis can be directly generalized to weighed graphs. For any node $u\in \mathcal{V}$, its neighborhood is defined as $\mathcal{N}(u)=\{v|(u,v)\in \mathcal{E}\}$ and its degree is $d_u=|\mathcal{N}(u)|$. The degree matrix $\mathbf{D}$ is a $n\times n$ diagonal matrix with entries $\mathbf{D}_{i,i}=d_i$, while the adjacency matrix $\mathbf{A}$ satisfies $\mathbf{A}_{i,j}=1$ if $(i,j)\in \mathcal{E}$ and $\mathbf{A}_{i,j}=0$ otherwise. The Laplacian matrix is given by $\mathbf{L}=\mathbf{D}-\mathbf{A}$ with eigenvalues $0=\lambda_1<\lambda_2\leq \lambda_3\leq...\leq \lambda_n$ and the corresponding eigenvectors $\mathbf{u}_1,...,\mathbf{u}_n$. The Laplacian pseudo inverse is defined as $\mathbf{L}^\dagger=\sum_{i=2}^{n}{\frac{1}{\lambda_i}\mathbf{u}_i\mathbf{u}_i^T}$. Based on the definition of $\mathbf{L}^\dagger$,  Resistance Distance (RD) is formulated as follows.
\begin{definition}
    Given a pair of vertices $s$ and $t$ of a graph $\mathcal{G}$, the resistance distance of $s$ and $t$  is defined as $r_\mathcal{G}(s,t)=(\mathbf{e}_s-\mathbf{e}_t)^T\mathbf{L}^{\dagger}(\mathbf{e}_s-\mathbf{e}_t)$. Where $\mathbf{e}_s$ (resp., $\mathbf{e}_t$) is the one-hot vector that takes value $1$ at $s$ (resp., $t$) and $0$ elsewhere.
\end{definition}

Since computing exact RD requires $O(n^{2.37})$ time via complex matrix multiplication, prohibitively expensive for large graphs. Consequently, existing works have primarily focused on developing approximation algorithms with provable error guarantees to balance efficiency and quality. In this paper, following ~\cite{peng2021local,yang2023efficient,liao2023resistance}, we also focus on resistance distance with absolute error guarantee.

\begin{definition}\label{def:ER-err}
    Given graph $\mathcal{G}$ and a pair of vertices $s$ and $t$, given any small constant $\epsilon>0$, $\hat{r}_\mathcal{G}(s,t)$ is called $\epsilon$-absolute error approximation of $r_\mathcal{G}(s,t)$ iff $|r_\mathcal{G}(s,t)-\hat{r}_\mathcal{G}(s,t)|\leq \epsilon$.
\end{definition}

\subsection{Existing Methods and Their Defects}\label{sec:existing-methods}

For our analysis, we define the normalized adjacency matrix as  $\mathcal{A}=\mathbf{D}^{-1/2}\mathbf{A}\mathbf{D}^{-1/2}$ and define the normalized Laplacian matrix as $\mathcal{L}=\mathbf{I}-\mathcal{A}$. The condition number of  $\mathcal{L}$ is defined as follows ~\cite{dwaraknath2024towards}.

\begin{definition}
    The condition number $\kappa$ of the normalized Laplacian matrix $\mathcal{L}$ is defined as $\kappa\triangleq  \frac{2}{\mu_2}$, where $\mu_2$ is the second eigenvalue of $\mathcal{L}$, i.e., the smallest non-zero eigenvalue of $\mathcal{L}$.
\end{definition}

Most existing RD algorithms have time complexity relevant to $\kappa$. Specifically, we let $\mathbf{P}\triangleq \mathbf{AD}^{-1}$ be the probability transition matrix, and these existing algorithms are based on the following Taylor expansion formula ~\cite{peng2021local,yang2025improved}.
\begin{equation}\label{equ:er_laylor_expansion}
\begin{aligned}
r_\mathcal{G}(s,t)&=(\mathbf{e}_s-\mathbf{e}_t)^T\mathbf{L}^{\dagger}(\mathbf{e}_s-\mathbf{e}_t)\\
&=\frac{1}{2}(\mathbf{e}_s-\mathbf{e}_t)^T\sum_{k=0}^{+\infty}{\mathbf{D}^{-1}\left(\frac{1}{2}\mathbf{I}+\frac{1}{2}\mathbf{P}\right)^k}(\mathbf{e}_s-\mathbf{e}_t).
\end{aligned}
\end{equation}
Thus, we can set a sufficiently large truncation step $l$, and define the approximation $\overline{r}_\mathcal{G}(s,t)=\frac{1}{2}(\mathbf{e}_s-\mathbf{e}_t)^T\sum_{k=0}^{l}{\mathbf{D}^{-1}\left(\frac{1}{2}\mathbf{I}+\frac{1}{2}\mathbf{P}\right)^k}(\mathbf{e}_s-\mathbf{e}_t)$. To make $\overline{r}_\mathcal{G}(s,t)$ be the $\epsilon$-absolute error approximation of $r_\mathcal{G}(s,t)$, the truncation step  $l=O(\kappa\log \frac{\kappa}{\epsilon})$ ~\cite{peng2021local,yang2023efficient,yang2025improved}. Based on this interpretation, the basic Power Method and random walk-based methods can be easily derived. Table \ref{tab:alg} summarizes the state-of-the-art methods and our solutions.

\begin{table}[t!]
\small
	\caption{A comparison of different algorithms for computing pairwise RD with $\epsilon$-absolute error guarantee, where for Lanczos Push, $C_1=\max_{u=s,t;i\leq k}{\Vert \mathbf{D}^{1/2}T_i(\mathbf{P})\mathbf{e}_u\Vert_1}$, and $C_2=\max_{i\leq k}{(\Vert \mathbf{v}_i\Vert_1+\Vert \mathcal{A}\mathbf{v}_i^+\Vert_1+\Vert \mathcal{A}\mathbf{v}_i^-\Vert_1)}$, see Section 4 for details.} \vspace{-0.2cm}
	\scalebox{1}{
		\begin{tabular}{c|c|c}
			\toprule
			\multicolumn{1}{c|}{Methods} & \multicolumn{1}{c|}{Algorithm}&\multicolumn{1}{c}{Time Complexity}\\
			\midrule
			\multirow{4}{*}{Linear}& Power Method \cite{yang2023efficient} & $O(\kappa m \log \frac{\kappa}{\epsilon})$ \\
             & LapSolver ~\cite{sachdeva2014faster} & $O( m \log ^{1/2} n \log \frac{1}{\epsilon})$ \\
             & FastRD ~\cite{lu2025resistance} & $\tilde{O}(m/\epsilon^2)$ \\
	       & Lanczos (Section 3)& $O (\sqrt{\kappa} m \log \frac{\kappa}{\epsilon})$ \\
           \midrule
           \multirow{6}{*}{Sublinear (Local)}& TP ~\cite{peng2021local}& $\tilde{O}(\kappa^4/\epsilon^2)$\\
			  & TPC ~\cite{peng2021local}& $\tilde{O}(\kappa^3/\epsilon^2)$\\
                & GEER ~\cite{yang2023efficient}& $\tilde{O}(\kappa^3/(\epsilon^2 d^2))$\\
                & Algorithm 1 in ~\cite{yang2025improved}& $\tilde{O}(\kappa^3/(\epsilon\sqrt{d}))$ \\
                & BiSPER ~\cite{cui2025mixing} & $\tilde{O}(\kappa^{7/3}/\epsilon^{2/3})$ \\
                & Lanczos Push (Section 4)& $\tilde{O}(\kappa^{2.75}C_1C_2/\epsilon\sqrt{d})$\\
            \midrule
             \multirow{2}{*}{Others}& \push ~\cite{liao2023resistance}& $\times$\\
             & \bipush ~\cite{liao2023resistance} & $\times$ \\
            \midrule
             & Lower Bound (Section 5)& $\Omega(\kappa)$ \\
            \bottomrule	
		\end{tabular}
	}\label{tab:alg}
\end{table}

\stitle{Power Method (\powermethod)}
is one of the most classic and fundamental algorithms in numerical linear algebra. \powermethod is also widely used in the computation of PageRank ~\cite{page1999pagerank}, Heat Kernal PageRank ~\cite{yang2019efficient}, SimRank ~\cite{wang2021approximate}, and graph propagation ~\cite{wang2021approximate}. In our pairwise RD computation problem, we slightly modify the implementation of \powermethod and use it to compute the groundtruth RD value. See Algorithm \ref{algo:pm} for the pseudocode illustation. Specifically, the \powermethod algorithm is implemented as follows: (i) Initially, we define the approximation $\overline{r}_\mathcal{G}(s,t)=0$ and a vector $\mathbf{r}=\mathbf{e}_s-\mathbf{e}_t$ with iteration $k=0$; (ii) for the $k^{th}$ iteration of the algorithm with $k\leq l$, we perform $\overline{r}_\mathcal{G}(s,t) \leftarrow \overline{r}_\mathcal{G}(s,t)+\frac{\mathbf{r}(s)}{2d_s}-\frac{\mathbf{r}(t)}{2d_t}$ and $\mathbf{r}\leftarrow \left(\frac{1}{2}\mathbf{I}+\frac{1}{2}\mathbf{P}\right) \mathbf{r}$, then turn to the $(k+1)^{th}$ iteration; (iii) after $l$ iterations, we stop the algorithm and output $\overline{r}_\mathcal{G}(s,t)$. By Equ. (\ref{equ:er_laylor_expansion}), we can easily see that $\overline{r}_\mathcal{G}(s,t)=\frac{1}{2}(\mathbf{e}_s-\mathbf{e}_t)^T\sum_{k=0}^{l}{\mathbf{D}^{-1}\left(\frac{1}{2}\mathbf{I}+\frac{1}{2}\mathbf{P}\right)^k}(\mathbf{e}_s-\mathbf{e}_t)$ is the output after this process. Next, we prove that by setting $l=O(\kappa\log \frac{\kappa}{\epsilon})$, the approximation $\overline{r}_\mathcal{G}(s,t)$ is the $\epsilon$-absolute error approximation of $r_\mathcal{G}(s,t)$. Due to space limits, the missing proofs of this paper can be found in the full-version ~\cite{full-version}.

   \begin{algorithm}[t!]
\small
	\SetAlgoLined
	\KwIn{$\mathcal{G},l,s,t$}
    $\overline{r}_\mathcal{G}(s,t)=0$, $\mathbf{r}=\mathbf{e}_s-\mathbf{e}_t$\;
		\For{$i=0,1,2,...,l$}{
       $\overline{r}_\mathcal{G}(s,t) \leftarrow \overline{r}_\mathcal{G}(s,t)+\frac{\mathbf{r}(s)}{2d_s}-\frac{\mathbf{r}(t)}{2d_t}$\;
       $\mathbf{r}\leftarrow \left(\frac{1}{2}\mathbf{I}+\frac{1}{2}\mathbf{P}\right) \mathbf{r}$\;
       }
	\KwOut{$\overline{r}_\mathcal{G}(s,t)$ as the approximation of $r_\mathcal{G}(s,t)$}
	\caption{Power Method (\powermethod) for RD computation}\label{algo:pm}
\end{algorithm}

\begin{theorem}\label{thm:pm_guarantee}
    When setting $l=2\kappa\log \frac{\kappa}{\epsilon}$, the approximation error holds that $|\overline{r}_\mathcal{G}(s,t)-r_\mathcal{G}(s,t)|\leq \epsilon$. In addition, the time complexity of \ \powermethod is $O(lm)=O(\kappa m \log \frac{\kappa}{\epsilon})$.
\end{theorem}


\stitle{Random Walk (\rw)} is also a fundamental operator in various graph algorithms, including PageRank ~\cite{lofgren2013personalized,lofgren16bidirection,wang2017fora,wu2021unifying}, Heat Kernel PageRank ~\cite{chung2007heat,yang2019efficient} and local Laplacian Solver ~\cite{andoni2018solving}. For pairwise RD estimation, the random walk operator was first introduced in ~\cite{peng2021local}. The general idea of the algorithm is as follows: based on Eq. (\ref{equ:er_laylor_expansion}), to approximate $r_\mathcal{G}(s,t)$, we only need to approximate $\mathbf{e}_s\left(\frac{1}{2}\mathbf{I}+\frac{1}{2}\mathbf{P}\right)^i\mathbf{e}_t/d_s$, $\mathbf{e}_s\left(\frac{1}{2}\mathbf{I}+\frac{1}{2}\mathbf{P}\right)^i\mathbf{e}_s/d_s,\mathbf{e}_t\left(\frac{1}{2}\mathbf{I}+\frac{1}{2}\mathbf{P}\right)^i\mathbf{e}_s/d_t$, and $\mathbf{e}_t\left(\frac{1}{2}\mathbf{I}+\frac{1}{2}\mathbf{P}\right)^i\mathbf{e}_t/d_t$ for each $i\leq l$, which is exactly the probability of a length $i$ lazy random walk starts at $t$ and ends at $s$. Specifically, the length $i$ lazy random walk denotes a sequence $(v_0,v_1,...,v_i)$ such that for each $0\leq j< i$, $v_{j+1}$ is a random neighbour of $v_j$ with probability $1/2$ and $v_{j+1}=v_j$ with probability $1/2$. The details of the pseudocode are illustrated in algorithm ~\ref{algo:tp}. By setting the number of random walks $n_r=O(l^2/\epsilon^2)$ with $l=O(\kappa\log \frac{\kappa}{\epsilon})$ ~\cite{peng2021local}, we can obtain $\hat{r}_\mathcal{G}(s,t)$ with the $\epsilon$-absolute error approximation. Thus the total time complexity of this algorithm is $O(n_rl^2)=\tilde{O}(\kappa^4/\epsilon^2)$ since $l=O(\kappa\log \frac{\kappa}{\epsilon})$ and $n_r=O(l^2/\epsilon^2)$. Peng et al. also proposed an advanced algorithm for the estimation of $\mathbf{e}_s\left(\frac{1}{2}\mathbf{I}+\frac{1}{2}\mathbf{P}\right)^i\mathbf{e}_t/d_s$ (details in Algorithm 2 in ~\cite{peng2021local}). The time complexity of the advanced algorithm is $\tilde{O}(\kappa^3/\epsilon^2)$, resulting in a reduced dependency on $\kappa$. Later, Yang et al. \cite{yang2023efficient} further improved this algorithm by providing a tighter analysis of the variance. They derived a refined time complexity $\tilde{O}(\kappa^3/(\epsilon^2 d^2))$ under the same error guarantee, where $d=\min \{d_s,d_t\}$. However, none of the \rw based algorithms break the $\kappa^3$ bottleneck. 

   \begin{algorithm}[t!]
\small
	\SetAlgoLined
	\KwIn{$\mathcal{G},\epsilon,l,n_r,s,t$}
		\For{$i=0,1,2,...,l$}{
       perform $n_r$ independent lazy random walks of length $i$ starting at $s$, let $X_{i,s}$ (resp., $X_{i,t}$) be the number of walks ends at $s$ (resp., $t$) \;
		perform $n_r$ independent lazy random walks of length $i$ starting at $t$, let $Y_{i,s}$ (resp., $Y_{i,t}$) be the number of walks ends at $s$ (resp., $t$) \;
        $\hat{r}_\mathcal{G}(s,t)+=\frac{X_{i,s}}{2n_r d_s}-\frac{X_{i,t}}{2n_r d_t}+\frac{Y_{i,t}}{2n_r d_t}-\frac{Y_{i,s}}{2n_r d_s}$ \;}
	\KwOut{$\hat{r}_\mathcal{G}(s,t)$ as the approximation of $r_\mathcal{G}(s,t)$}
	\caption{Random Walk (\rw) for RD computation~\cite{peng2021local}}\label{algo:tp}
\end{algorithm}

\stitle{Other Methods.} Despite the above two classical methods, there are also related works based on different techniques and data structures. 
For example, the nearly linear time Laplacian Solver ~\cite{kyng2016approximate,gao2023robust} allows us 
to compute the $\epsilon$-approximation of RD in $O( m \log ^{1/2} n \log \frac{1}{\epsilon})$ time, 
independent of condition number $\kappa$. 
Using the Laplacian solver, 
one can also design $\tilde{O}(m/\epsilon^2)$ resistance distance sketches 
and query any single-pair RD in $\tilde{O}(\epsilon^{-2})$ time \cite{spielman2008graph,lu2025resistance}.
However, Laplacian Solver is hard to implement in practice and its time 
complexity is linearly dependent on $m$. 
Besides, Liao et al. ~\cite{liao2023resistance} introduced a novel landmark interpretation of RD
 and design several efficient \rw-based algorithms for the computation of RD based on their interpretation. 
However, their algorithms are heuristic and lack theoretical guarantees based on $n,m,\kappa$. 
Recently, \cite{yang2025improved,cui2025mixing} utilize bidirectional techniques to improve
 the efficiency of \rw for single pair RD computation. 
Specifically, Yang et al. \cite{yang2025improved} derive an $\tilde{O}(\kappa^3/(\epsilon\sqrt{d}))$ algorithm 
under $\epsilon$-absolute error guarantee and an $\tilde{O}(\kappa^3 \sqrt{d}/\epsilon)$ algorithm under 
relative error guarantee. Concurrent to this paper,
 Cui et. al \cite{cui2025mixing} derive an $\tilde{O}(\kappa^{7/3}/\epsilon^{2/3})$ algorithm under $\epsilon$-absolute error guarantee.
As an additional remark, while \bisper \cite{cui2025mixing} holds a theoretical advantage in asymptotic complexity, 
our Lanczos Push algorithm delivers superior practical performance on graphs with large condition numbers, such as road networks.
The reasons could be that the basic framework of the bidirectional methods ~\cite{yang2025improved,cui2025mixing} is PowerMethod,
which typically require $O(\kappa)$ iterations to converge,
 whereas the basic framework of the Lanczos Push algorithm is the Lanczos iteration,
 which often requires less than $O(\sqrt{\kappa})$ iterations to converge in practice.
 The stronger convergence of the Lanczos iteration leads to better practical performance
 of the proposed Lanczos Push algorithm (as shown in experiments), though slightly weaker theoretical worst case guarantee.

\section{A Novel Lanczos Iteration Method} \label{sec:LI}
We begin by introducing the classical Lanczos algorithm for computing the quadratic form $\mathbf{x}^Tf(\mathbf{A})\mathbf{x}$, where $\mathbf{x}\in \mathbb{R}^n$ is a vector, $f$ is a matrix function, $\mathbf{A}\in \mathbb{R}^{n\times n}$ is a positive semi-definite (PSD) matrix. In our case for RD computation, we choose $\mathbf{x}=\frac{\mathbf{e}_s}{\sqrt{d_s}}-\frac{\mathbf{e}_t}{\sqrt{d_t}}$, $\mathbf{A}=\mathcal{A}$ to be the normalized adjacency matrix, and $f(\mathcal{A})=(\mathbf{I}-\mathcal{A})^{\dagger}$. Now by the definition of RD, we have that:

\begin{equation}
\begin{aligned}
    r_\mathcal{G}(s,t)&=(\mathbf{e}_s-\mathbf{e}_t)^T\mathbf{L}^{\dagger}(\mathbf{e}_s-\mathbf{e}_t)\\
    &=(\mathbf{e}_s-\mathbf{e}_t)^T\mathbf{D}^{-1/2}(\mathbf{I}-\mathcal{A})^{\dagger}\mathbf{D}^{-1/2}(\mathbf{e}_s-\mathbf{e}_t)\\
    &=(\frac{\mathbf{e}_s}{\sqrt{d_s}}-\frac{\mathbf{e}_t}{\sqrt{d_t}})^T(\mathbf{I}-\mathcal{A})^{\dagger} (\frac{\mathbf{e}_s}{\sqrt{d_s}}-\frac{\mathbf{e}_t}{\sqrt{d_t}})\\
    &=\mathbf{x}^Tf(\mathcal{A})\mathbf{x}.
\end{aligned}
\end{equation}

The second equality holds beacuse $\mathbf{L}^{\dagger}=\mathbf{D}^{-1/2}\mathcal{L}^{\dagger}\mathbf{D}^{-1/2}=\mathbf{D}^{-1/2}(\mathbf{I}-\mathcal{A})^{\dagger}\mathbf{D}^{-1/2}$. This immediately allows us to use Lanczos iteration to compute the RD value. However, we notice that $f(\mathcal{A})=(\mathbf{I}-\mathcal{A})^{\dagger}$ is the pseudo inverse of the matrix $\mathbf{I}-\mathcal{A}$, which is not a common matrix function (eg., matrix inverse, matrix exponential). So we include one of the implementation of Lanczos algorithm ~\cite{musco2018stability} with slight modification in Algorithm ~\ref{algo:lanczos}. The main idea of the Lanczos iteration is to project the matrix $\mathcal{A}$ onto the tridiagonal matrix $\mathbf{T}$ such that $\mathbf{T}=\mathbf{V}^T\mathcal{A}\mathbf{V}$, where $\mathbf{V}=[\mathbf{v}_1,...,\mathbf{v}_k]$ be the $n\times k$ dimension column orthogonal matrix with $k\ll n$ that spans the $k$-dimension Krylov subspace: $\mathcal{K}_k(\mathbf{v}_1, \mathcal{A})=span\langle\mathbf{v}_1, \mathcal{A}\mathbf{v}_1,..., \mathcal{A}^k\mathbf{v}_1\rangle$ with $\mathbf{v}_1=\mathbf{x}/\Vert \mathbf{x} \Vert_2=\left(\frac{\mathbf{e}_s}{\sqrt{d_s}}-\frac{\mathbf{e}_t}{\sqrt{d_t}}\right)/ \sqrt{\frac{1}{d_s}+\frac{1}{d_t}}$. First, if we have already compute the column orthogonal matrix $\mathbf{V}$, we prove that for the vectors $\mathcal{A}^i\mathbf{v}_1$ with $i\leq k$, we have $\mathcal{A}^i\mathbf{v}_1=\mathbf{V}\mathbf{T}^i\mathbf{V}^T\mathbf{v}_1$. Formally we state the following Lemma.

\begin{lemma}\label{lem:lz_polynomial_correct}
    Given $\mathbf{V}=[\mathbf{v}_1,...,\mathbf{v}_k]$ is the $n\times k$ dimension orthogonal matrix spanning the subspace $\mathcal{K}_k(\mathbf{v}_1, \mathcal{A})=span\langle\mathbf{v}_1, \mathcal{A}\mathbf{v}_1,..., \mathcal{A}^k\mathbf{v}_1\rangle$ and $\mathbf{T}=\mathbf{V}^T\mathcal{A}\mathbf{V}$. Then for any polynomial $p_k$ with degree at most $k$, we have $p_k(\mathcal{A})\mathbf{v}_1=\mathbf{V}p_k(\mathbf{T})\mathbf{V}^T\mathbf{v}_1$.
\end{lemma}

\comment{
\begin{proof}
    Since $\mathcal{A}^i\mathbf{v}_1\in \mathcal{K}_k(\mathbf{v}_1, \mathcal{A})$ for $i\leq k$ and $\mathbf{VV}^T$ is the projection matrix onto the subspace $\mathcal{K}_k(\mathbf{v}_1, \mathcal{A})$, we have $\mathcal{A}^i\mathbf{v}_1=\mathbf{VV}^T\mathcal{A}^i\mathbf{v}_1$. Therefore
    \begin{align*}
        \mathcal{A}^i\mathbf{v}_1&=(\mathbf{VV}^T)\mathcal{A}(\mathbf{VV}^T)\mathcal{A}... \mathcal{A}(\mathbf{VV}^T)\mathbf{v}_1\\
        &=\mathbf{V}(\mathbf{V}^T\mathcal{A}\mathbf{V})...(\mathbf{V}^T\mathcal{A}\mathbf{V})\mathbf{V}^T\mathbf{v}_1=\mathbf{V}\mathbf{T}^i\mathbf{V}^T\mathbf{v}_1.
    \end{align*}
    holds for any $i\leq k$. Thus $p_k(\mathcal{A})\mathbf{v}_1=\mathbf{V}p_k(\mathbf{T})\mathbf{V}^T\mathbf{v}_1$ for any polynomial $p_k$ with degree at most $k$.
\end{proof}
}

Next, based on Lemma \ref{lem:lz_polynomial_correct} we show that $\mathbf{v}_1^T\mathbf{V}(\mathbf{I}-\mathbf{T})^{-1}\mathbf{V}^T\mathbf{v}_1$ is a good estimator of the quadratic form $\mathbf{v}_1^T(\mathbf{I}-\mathcal{A})^\dagger\mathbf{v}_1$. The proof is based on the classical analysis but we slightly modify the proof.

\begin{lemma}\label{lem:lz_quadratic_err}
    Let $\mathbf{V}\in \mathbb{R}^{n\times k}$ and $\mathbf{T}\in \mathbb{R}^{k\times k}$ defined in Lemma \ref{lem:lz_polynomial_correct}. Let $\mathcal{P}_k$ be the set of all polynomials with degree $\leq k$. Then, we have:
\begin{align*}
    |\mathbf{v}_1^T&\mathbf{V}(\mathbf{I}-\mathbf{T})^{-1}\mathbf{V}^T\mathbf{v}_1-\mathbf{v}_1^T(\mathbf{I}-\mathcal{A})^\dagger\mathbf{v}_1|\\
   & \leq 2 \min_{p\in \mathcal{P}_k}{\max_{x \in [\lambda_{min}(\mathcal{A}),\lambda_{2}(\mathcal{A})]}{|1/(1-x)-p(x)|}},
    \end{align*}
    
    where \ $\lambda_{min}(\mathcal{A})$ and $\lambda_{2}(\mathcal{A})$ denotes the smallest and second largest eigenvalue of $\mathcal{A}$, respectively.
\end{lemma}

Based on Lemma \ref{lem:lz_quadratic_err}, we can now approximate the RD value by the following formula: 
\begin{equation}
\begin{aligned}
    r_\mathcal{G}(s,t)&=(\frac{\mathbf{e}_s}{\sqrt{d_s}}-\frac{\mathbf{e}_t}{\sqrt{d_t}})^T(\mathbf{I}-\mathcal{A})^{\dagger} (\frac{\mathbf{e}_s}{\sqrt{d_s}}-\frac{\mathbf{e}_t}{\sqrt{d_t}})\\
    &=\left(\frac{1}{d_s}+\frac{1}{d_t}\right)\mathbf{v}_1^T (\mathbf{I}-\mathcal{A})^{\dagger}\mathbf{v}_1 \\
    & \approx \left(\frac{1}{d_s}+\frac{1}{d_t}\right)\mathbf{v}_1^T \mathbf{V}(\mathbf{I}-\mathbf{T})^{-1} \mathbf{V}^T\mathbf{v}_1 \\
    &=\left(\frac{1}{d_s}+\frac{1}{d_t}\right)\mathbf{e}_1^T (\mathbf{I}-\mathbf{T})^{-1} \mathbf{e}_1.
\end{aligned}
\end{equation}\label{equ:approx_formula}

The second equality holds because $\mathbf{v}_1=\left(\frac{\mathbf{e}_s}{\sqrt{d_s}}-\frac{\mathbf{e}_t}{\sqrt{d_t}}\right)/\left\Vert \frac{\mathbf{e}_s}{\sqrt{d_s}}-\frac{\mathbf{e}_t}{\sqrt{d_t}} \right\Vert_2$ and $\left\Vert \frac{\mathbf{e}_s}{\sqrt{d_s}}-\frac{\mathbf{e}_t}{\sqrt{d_t}} \right\Vert_2^2=\frac{1}{d_s}+\frac{1}{d_t}$. The last equality holds because $\mathbf{v}_1^T\mathbf{V}=\mathbf{e}_1^T$, by $\mathbf{v}_1$ is orthogonal to $\mathbf{v}_2,...,\mathbf{v}_k$. Now since $\mathbf{I}-\mathbf{T}$ is only a $k\times k$ dimension matrix with $k\ll n$, we can efficiently compute $\left(\frac{1}{d_s}+\frac{1}{d_t}\right)\mathbf{e}_1^T (\mathbf{I}-\mathbf{T})^{-1} \mathbf{e}_1$ in only $O(k^{2.37})$ time with fast matrix multiplication ~\cite{williams2012matrices}. Finally, all the things remain is how to efficiently compute these orthogonal vectors $\mathbf{v}_1,...,\mathbf{v}_k$. To reach this, in each iteration we perform the Lanczos recurrence (i.e. Line 3-7 in Algorithm ~\ref{algo:lanczos}):
\begin{equation}
    \beta_{i+1}\mathbf{v}_{i+1}=(\mathcal{A}-\alpha_i)\mathbf{v}_i-\beta_i\mathbf{v}_{i-1},
\end{equation}\label{equ:lz_recurrence}

where $\alpha_i$ and $\beta_{i+1}$ are scalars computed in Line 4 and Line 6 in Algorithm ~\ref{algo:lanczos}. For our analysis, we present a classical result of the output of Lanczos iteration (see eg.  Claim 4.1 in ~\cite{musco2018stability}). This will be used in our analysis later.
   \begin{algorithm}[t!]
\small
	\SetAlgoLined
	\KwIn{$\mathcal{G},s,t, k$}
    $\mathbf{v}_0=0$, $\mathbf{v}_1=\left(\frac{\mathbf{e}_s}{\sqrt{d_s}}-\frac{\mathbf{e}_t}{\sqrt{d_t}}\right)/ \sqrt{\frac{1}{d_s}+\frac{1}{d_t}}$, $\beta_1=0$\;
		\For{$i=1,2,...,k$}{
       $\mathbf{v_{i+1}}= \mathcal{A}\mathbf{v}_i-\beta_i \mathbf{v}_{i-1}$\;
       $\alpha_i=\langle \mathbf{v}_{i+1},\mathbf{v}_i\rangle$\;
       $\mathbf{v_{i+1}}= \mathbf{v_{i+1}} - \alpha_i \mathbf{v}_i$\;
       $\beta_{i+1}=\Vert \mathbf{v}_{i+1}\Vert_2$\;
       $\mathbf{v}_{i+1}=\mathbf{v}_{i+1}/\beta_{i+1}$\;}
       $\mathbf{T}=\begin{bmatrix}
       \begin{array}{cccc}
        \alpha_1 & \beta_2 &  & 0\\
        \beta_2 & \alpha_2 & \ddots & \\
         & \ddots  & \ddots & \beta_k \\
        0 & & \beta_k & \alpha_k 
        \end{array}
        \end{bmatrix}$, $\mathbf{V}=[\mathbf{v}_1,...,\mathbf{v}_k]$\;
        
	\KwOut{$\hat{r}_\mathcal{G}(s,t)=\left(\frac{1}{d_s}+\frac{1}{d_t}\right)\mathbf{e}_1^T (\mathbf{I}-\mathbf{T})^{-1}\mathbf{e}_1$ as the approximation of $r_\mathcal{G}(s,t)$}
	\caption{Lanczos iteration for RD computation}\label{algo:lanczos}
\end{algorithm}\vspace{-0.2cm}

\begin{proposition}{(Output guarantee)}\label{prop:lanczos_guarantee}
    The Lanczos iteration (Algorithm ~\ref{algo:lanczos}) computes $\mathbf{V}\in \mathbb{R}^{n\times k}$, $\mathbf{T}\in \mathbb{R}^{k\times k}$ and an additional vector $\mathbf{v}_{k+1}$, a scalar $\beta_{k+1}$ such that:
$$\mathbf{AV}=\mathbf{VT}+\beta_{k+1}\mathbf{v}_{k+1}\mathbf{e}_k^T.$$
Moreover, the column space of $\mathbf{V}$ spans the Krylov subspace $\mathcal{K}_k(\mathbf{v}_1, \mathcal{A})=span\langle\mathbf{v}_1, \mathcal{A}\mathbf{v}_1,..., \mathcal{A}^k\mathbf{v}_1\rangle$ and $\mathbf{T}=\mathbf{V}^T\mathcal{A}\mathbf{V}$.
\end{proposition}

Finally, based on Lemma \ref{lem:lz_quadratic_err} and Proposition \ref{prop:lanczos_guarantee} we prove the following Theorem which gives the error bound and time complexity of the Lanczos iteration (Algorithm ~\ref{algo:lanczos}). 

\begin{theorem}{(Error guarantee)}\label{thm:lanczos_err}
    Let $\hat{r}_\mathcal{G}(s,t)$ output by Lanczos iteration (Algorithm \ref{algo:lanczos}). Then the error satisfies
\begin{align*}
    |\hat{r}_\mathcal{G}(s,t)&-r_\mathcal{G}(s,t)|\leq \epsilon.
    \end{align*}
    When setting the iteration number $k=O(\sqrt{\kappa}\log \frac{\kappa}{\epsilon})$. Furthermore, the time complexity of Algorithm \ref{algo:lanczos} is $O(km+k^{2.37})=O (\sqrt{\kappa} m \log \frac{\kappa}{\epsilon})$ when $k^{1.37}\leq m$.
\end{theorem}

Therefore, according to Theorem \ref{thm:lanczos_err} we provided that the approximation $\hat{r}_\mathcal{G}(s,t)$ satisfies the $\epsilon$-absolute error guarantee after $k=\sqrt{\kappa}\log \frac{\kappa}{\epsilon}$ iterations of Lanczos recurrence. Recall that the iteration number of \powermethod is required to be $l=\kappa\log \frac{\kappa}{\epsilon}$, so Algorithm \ref{algo:lanczos} is $\sqrt{\kappa}$-times faster than \powermethod.

\section{A Novel Lanczos Push method} \label{sec:LP}

\subsection{The Lanczos Push Algorithm}
Next, we introduce our novel push-style local algorithm, called Lanczos Push, based on a newly-developed "purning" technique. Before introducing our techniques, we first explain the high level idea of our method. Recall that the Lanczos iteration search a sequence of orthogonal vectors $\mathbf{v}_1,...,\mathbf{v}_k$ to accelerate RD computation (i.e., only $\sqrt{\kappa}$ dependency on condition number $\kappa$, which improves upon PM by a $\sqrt{\kappa}$ factor). Now, to design a local algorithm for RD computation while maintaining the rapid convergence rate by Lanczos iteration, our hope is to search a sequence of $\mathbf{sparse}$ vectors $\hat{\mathbf{v}}_1,...,\hat{\mathbf{v}}_k$ that is similar to $\mathbf{v}_1,...,\mathbf{v}_k$ (though they can be no longer orthogonal). See Fig. \ref{fig:illustration_lzpush} as an illustration.

\begin{figure}
    \centering
    \includegraphics[scale=0.2]{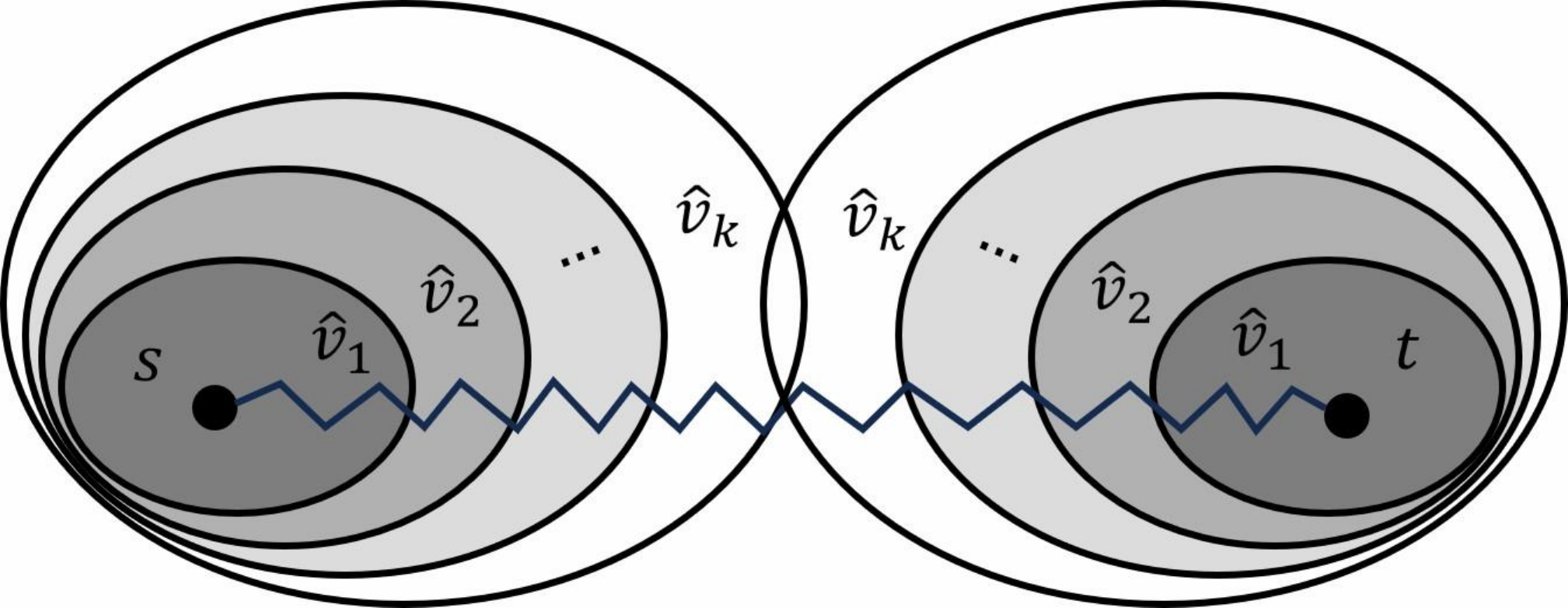} \vspace{-0.3cm}
    \caption{High level idea of our Lanczos Push algorithm for computing $r_\mathcal{G}(s,t)$; $\hat{\mathbf{v}}_1,...,\hat{\mathbf{v}}_k$ are sparse vectors.}\vspace{-0.5cm}
    \label{fig:illustration_lzpush}
\end{figure}

To specify our method, let's closely examine the time complexity of Lanczos iteration (Algorithm \ref{algo:lanczos}). Note that in Line 3 of Algorithm \ref{algo:lanczos}, we perform the matrix-vector multiplication $\mathcal{A}\mathbf{v}_i$, which requires $O(m)$ operations. In Line 4-7 of Algorithm \ref{algo:lanczos}, we invoke the inner product and addition/subtraction operations on $\mathbf{v}_i$, which requires $O(n)$ operations. Therefore, to transform Algorithm \ref{algo:lanczos} to a local algorithm, there are two main challenges: (i) "locally" implement the matrix-vector multiplication $\mathcal{A}\mathbf{v}_i$ in tach iteration. That is, we should approximate $\mathcal{A}\mathbf{v}_i$ by only invoking small portion of graph. (ii) maintaining $\mathbf{v}_i$ a sparse vector in each iteration. To address these challenges, we propose the Lanczos Push algorithm to locally approximate these vectors. To reach this end, first we define our sparse approximation of $\mathbf{v}_i$ and $\mathcal{A}\mathbf{v}_i$.

\begin{definition}\label{def:AMV(A,v_i)}
    We define $AMV(\mathcal{A},\hat{\mathbf{v}}_i)\in \mathbb{R}^n$ as the vector that approximates $\mathcal{A}\hat{\mathbf{v}}_i$. Specifically, for each node $v\in \mathcal{V}$, we define $AMV(\mathcal{A},\hat{\mathbf{v}}_i)(v)=\sum_{u\in \mathcal{N}(v)}{\frac{\hat{\mathbf{v}}_i(u)}{\sqrt{d_ud_v}}\mathbb{I}_{|\hat{\mathbf{v}}_i(u)|>\epsilon\sqrt{d_ud_v}}}$. Where we denote $\mathbb{I}_{|\hat{\mathbf{v}}_i(u)|>\epsilon\sqrt{d_ud_v}}$ as the indicator that takes value $1$ if and only if $|\hat{\mathbf{v}}_i(u)|>\epsilon\sqrt{d_ud_v}$.
\end{definition}

\begin{definition}\label{def:vi_constraint}
    We define $\hat{\mathbf{v}}_i|_{S_i}\in \mathbb{R}^n$ as the vector $\hat{\mathbf{v}}_i$ constraint on the subset $S_i\subset \mathcal{V}$. Specifically, $\hat{\mathbf{v}}_i|_{S_i}(u)=\hat{\mathbf{v}}_i(u)$ for $u\in S_i$, and $\hat{\mathbf{v}}_i|_{S_i}(u)=0$ otherwise. Where $S_i= \{u\in \mathcal{V}: |\hat{\mathbf{v}}_i(u)|> \epsilon d_u\}$.
\end{definition}

Based on Definition \ref{def:AMV(A,v_i)} and \ref{def:vi_constraint}, we propose the following subset Lanczos recurrence:

\begin{equation}\label{equ:subset_lz_recurrence}
\hat{\beta}_{i+1}\hat{\mathbf{v}}_{i+1}=AMV(\mathcal{A},\hat{\mathbf{v}}_i)-\hat{\alpha}_i\hat{\mathbf{v}}_i|_{S_i}-\hat{\beta}_i\hat{\mathbf{v}}_{i-1}|_{S_{i-1}}.
\end{equation}

In Eq. (\ref{equ:subset_lz_recurrence}), $AMV(\mathcal{A},\hat{\mathbf{v}}_i)$ is the sparse approximation of $\mathcal{A}\hat{\mathbf{v}}_i$ by Definition \ref{def:AMV(A,v_i)}, $\hat{\mathbf{v}}_i|_{S_i}$ is the sparse approximation of $\hat{\mathbf{v}}_i$ by Definition \ref{def:vi_constraint}, $\hat{\alpha}_i$ and $\hat{\beta}_{i+1}$ are scalars computed in the same way as Line 4 and Line 6 in Algorithm ~\ref{algo:lanczos}. Putting these things together, we provide the pseudocode in Algorithm \ref{algo:lanczos_local}. The idea of Algorithm \ref{algo:lanczos_local} is simple: to make Lanczos iteration implemented locally on graphs, we perform the subset Lanczos iteration (Eq. (\ref{equ:subset_lz_recurrence})). In each iteration, we perform the following two key operations: (i) For the approximation of the matrix-vector multiplication $\mathcal{A}\hat{\mathbf{v}}_i$, we only search on the node $u$ such that $\hat{\mathbf{v}}_i(u)$ non-zero. Next, for the neighbors $v\in \mathcal{N}(u)$, we only add the scores along the edge $(u,v)\in \mathcal{E}$ with $|\hat{\mathbf{v}}_i(u)|> \epsilon \sqrt{d_ud_v}$, see Line 4-8 in Algorithm \ref{algo:lanczos_local}. We denote $AMV(\mathcal{A},\hat{\mathbf{v}}_i)$ as the approximation of $\mathcal{A}\hat{\mathbf{v}}_i$ after this operation. (ii) We select a subset of important nodes $S_i= \{u\in \mathcal{V}: |\hat{\mathbf{v}}_i(u)|> \epsilon d_u\}$, and we perform addition/subtraction only constraint on subset $S_i$, see Line 9-11, Line 13-15 of Algorithm \ref{algo:lanczos_local}. These operations maintain the sparsity of the computation. \comment{In the next section, we will prove that after our sparse approximation, the time complexity of the proposed algorithm will become local, while the error propagation maintains stable. Our analysis for the stability result (Theorem \ref{thm:lanczos_push_error}) is partially inspired by ~\cite{musco2018stability}.}

\begin{algorithm}[t!]
\small
	\SetAlgoLined
	\KwIn{$\mathcal{G},s,t, k,\epsilon$}
    $\hat{\mathbf{v}}_0=0$, $\hat{\mathbf{v}}_1=\left(\frac{\mathbf{e}_s}{\sqrt{d_s}}-\frac{\mathbf{e}_t}{\sqrt{d_t}}\right)/ \sqrt{\frac{1}{d_s}+\frac{1}{d_t}}$, $\hat{\beta}_1=0$\;
		\For{$i=1,2,...,k$}{
        $S_i= \{u\in \mathcal{V}: |\hat{\mathbf{v}}_i(u)|> \epsilon d_u\}$\;
        \For{node $u$ such that $\hat{\mathbf{v}}_i(u)\neq 0$}{
            \For{$v\in \mathcal{N}(u)$ with $|\hat{\mathbf{v}}_i(u)|>\epsilon \sqrt{d_ud_v}$}{
                $\hat{\mathbf{v}}_{i+1}(v)=\hat{\mathbf{v}}_{i+1}(v)+\frac{1}{\sqrt{d_u d_v}}\hat{\mathbf{v}}_i (u)$\;
            }
        }
        \For{$u\in S_{i-1}$}{
       $\hat{\mathbf{v}}_{i+1}(u)= \hat{\mathbf{v}}_{i+1}(u)-\hat{\beta}_i \hat{\mathbf{v}}_{i-1}(u)$\;
       }
       $\hat{\alpha}_i=\langle \hat{\mathbf{v}}_{i+1},\hat{\mathbf{v}}_i\rangle$\;
       \For{$u\in S_i$}{
       $\hat{\mathbf{v}}_{i+1}(u)= \hat{\mathbf{v}}_{i+1}(u) - \hat{\alpha}_i \hat{\mathbf{v}}_i(u)$\;
       }
       $\hat{\beta}_{i+1}=\Vert\hat{\mathbf{v}}_{i+1}\Vert_2$\;
       $\hat{\mathbf{v}}_{i+1}=\hat{\mathbf{v}}_{i+1}/\hat{\beta}_{i+1}$\;}
       $\hat{\mathbf{T}}=\begin{bmatrix}
       \begin{array}{cccc}
        \hat{\alpha}_1 & \hat{\beta}_2 &  & 0\\
        \hat{\beta}_2 & \hat{\alpha}_2 & \ddots & \\
         & \ddots  & \ddots & \hat{\beta}_k \\
        0 & & \hat{\beta}_k & \hat{\alpha}_k 
        \end{array}
        \end{bmatrix}$, $\hat{\mathbf{V}}=[\hat{\mathbf{v}}_1,...,\hat{\mathbf{v}}_k]$\;
        
	\KwOut{$r'_\mathcal{G}(s,t)=\left(\frac{1}{d_s}+\frac{1}{d_t}\right)\hat{\mathbf{v}}_1^T \hat{\mathbf{V}} (\mathbf{I}-\hat{\mathbf{T}})^{-1}\mathbf{e}_1$ as the approximation of $r_\mathcal{G}(s,t)$}
	\caption{Lanczos Push for RD computation}\label{algo:lanczos_local}
\end{algorithm}

\begin{figure}
    \centering
    \includegraphics[scale=0.17]{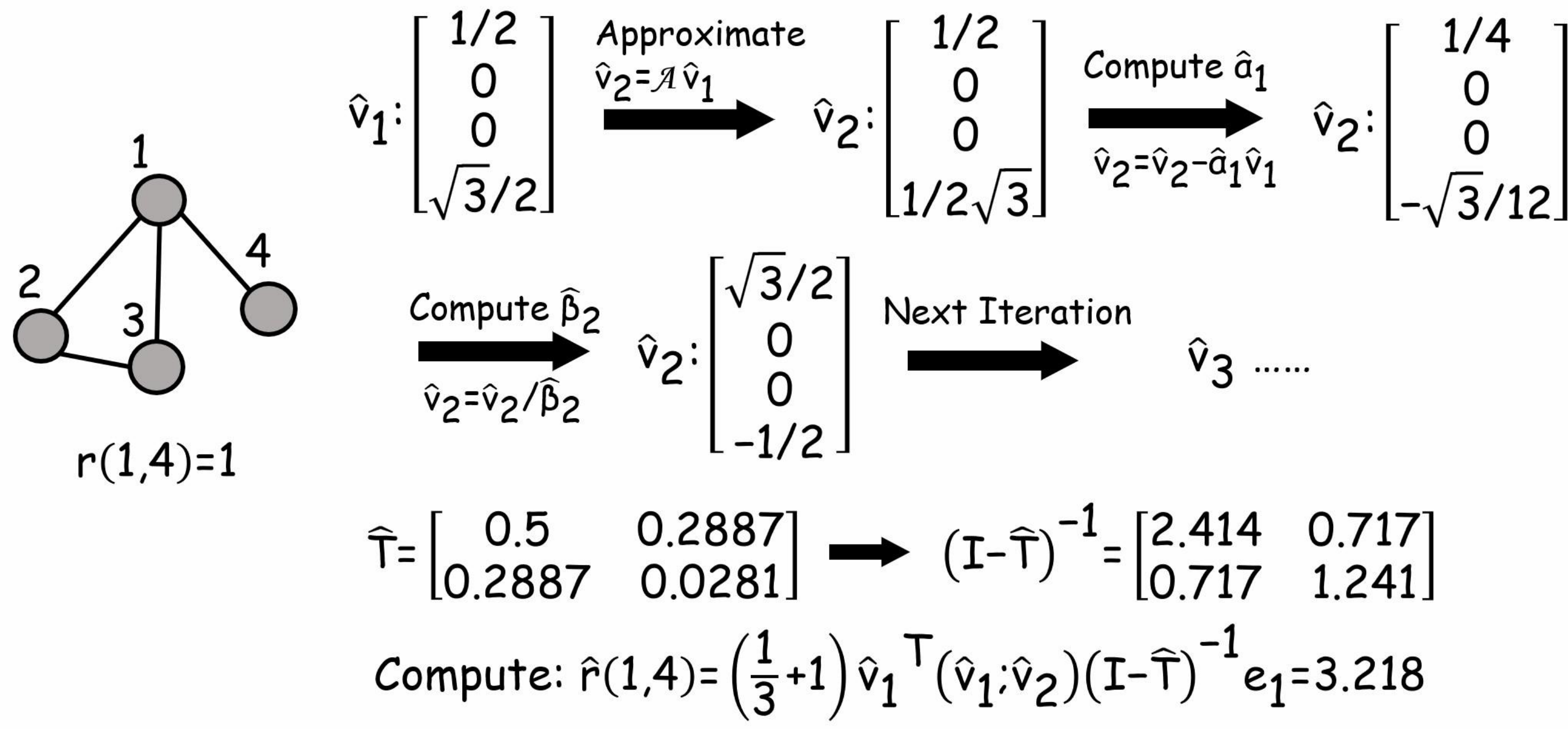} \vspace{-0.3cm}
    \caption{A running example of our Lanczos Push algorithm}\vspace{-0.5cm}
    \label{fig:example_lzpush}
\end{figure}

\stitle{A running example.} To illustrate the proposed Lanczos Push algorithm, we provide a running example here. We consider a toy graph with vertex set $\mathcal{V}=\{1,2,3,4\}$ and edge set $\mathcal{E}=\{(1,2),(1,3),(1,4),(2,3)\}$, as illustrated in Fig. \ref{fig:example_lzpush}. We focus on computing the resistance distance (RD) between nodes index $1$ and $4$, where the exact value is known to be $r(1,4)=1$ by the definition of resistance. For demonstration purposes, we set the algorithm parameters to $k=2$ and $\epsilon=0.25$ for Lanczos Push algorithm. We first initialize $\hat{\mathbf{v}}_1=\left(\frac{\mathbf{e}_s}{\sqrt{d_s}}-\frac{\mathbf{e}_t}{\sqrt{d_t}}\right)/ \sqrt{\frac{1}{d_s}+\frac{1}{d_t}}=[1/2,0,0,\sqrt{3}/2]^T$ and set $\hat{\beta}_1=0$ by Line 1 of Algorithm \ref{algo:lanczos_local}. We set $S_1=\{1,4\}$ by Line 3. Next, we invoke Lines 4-8 to approximate the matrix-vector multiplication $\mathcal{A}\hat{\mathbf{v}}_1$. Notably, the computations along the edge $(1,2)$ and $(1,3)$ are prevented because $|\hat{\mathbf{v}}_1(u)|/\sqrt{d_ud_v}=0.5/\sqrt{6}<0.25=\epsilon $ for $(u,v)=(1,2)$ and $(1,3)$. After the computation of Line 4-8, we obtain $\hat{\mathbf{v}}_2=[1/2,0,0,1/2\sqrt{3}]^T$. Then, we skip the computation from Lines 9-11 because we set $\hat{\beta}_1=0$ at initial step. Next, we compute $\hat{\alpha}_1=\langle \hat{\mathbf{v}}_{2},\hat{\mathbf{v}}_1\rangle=1/2$ by Line 12. We update $\hat{\mathbf{v}}_2=\hat{\mathbf{v}}_2-\hat{\alpha}_1\hat{\mathbf{v}}_1|_{S_1}=[1/4,0,0,-\sqrt{3}/12]^T$ by Line 13-15. We compute $\hat{\beta}_2=\Vert \hat{\mathbf{v}}_2\Vert_2=1/2\sqrt{3}$ by Line 16, and we finally update $\hat{\mathbf{v}}_2=\hat{\mathbf{v}}_2/\beta_2=[\sqrt{3}/2,0,0,-1/2]^T$ by Line 17. Subsequently, we turn to the next iteration. For iteration $i=2$, we can similarly compute $\hat{\alpha}_2= 0.0281$; we omit the details here because all the routines maintain the same. Finally, we compute the approximate RD value by $\hat{r}(1,4)=\left(\frac{1}{3}+1\right)\hat{\mathbf{v}}_1^T [\hat{\mathbf{v}}_1;\hat{\mathbf{v}}_2] (\mathbf{I}-\hat{\mathbf{T}})^{-1}\mathbf{e}_1=3.218$. Note that this result seems inaccurate mainly because we set a large threshold $\epsilon=0.25$ for better illustration. In real-world graphs, when setting the threshold $\epsilon$ smaller, our Lanczos Push algorithm successfully reduce computation costs while does not produce too much additional error. The theoretical guarantees are provided in the next subsection.

\comment{Therefore, it is easy to see that the operation times of the approximate matrix-vector multiplication $AMV(\mathcal{A},\hat{\mathbf{v}}_i)$ can be bounded by only $O(\frac{\Vert \mathcal{A}\hat{\mathbf{v}}_i\Vert_1}{\epsilon})$ , which is independent of $m$. Additionally, by our choose of $S_i$, the size of $S_i$ can be bounded by $|S_i|\leq O(\frac{\Vert \hat{\mathbf{v}}_i\Vert_1}{\epsilon})$, which is independent of $n$. Next, if we perform the subset Lanczos recurrence (Equ. (\ref{equ:subset_lz_recurrence})) in each iteration, we can make the non-zero entries of $\hat{\mathbf{v}}_i$, $supp(\hat{\mathbf{v}}_i)$ at most $O(\frac{\Vert \hat{\mathbf{v}}_i\Vert_1+\Vert \mathcal{A}\hat{\mathbf{v}}_i\Vert_1}{\epsilon})$). 
When assuming that $\Vert \hat{\mathbf{v}}_i\Vert_1,\Vert \mathcal{A}\hat{\mathbf{v}}_i\Vert_1=O(1)$, we can prove that the total operations perform by Algorithm \ref{algo:lanczos_local} in each iteration $i$ will be only $O(\frac{\Vert \hat{\mathbf{v}}_i\Vert_1+\Vert \mathcal{A}\hat{\mathbf{v}}_i\Vert_1}{\epsilon})=O(\frac{1}{\epsilon})$, independent of $n$ and $m$.}
 
\subsection{Analysis of the Lanczos Push Algorithm}
Here, we prove that our Lanczos Push algorithm (Algorithm \ref{algo:lanczos_local}) does not introduce significant errors compared to the Lanczos iteration, while drastically reducing the runtime. However, due to the complexity of the proofs, we provide only an outline here, the details are left in the full version ~\cite{full-version}. Overall, the theoretical analysis is devided into two steps: (i) the error bound of Algorithm \ref{algo:lanczos_local}, (ii) the local time complexity of Algorithm \ref{algo:lanczos_local}. For simplicity, we denote $C_1=\max_{u=s,t;i\leq k}{\Vert T_i(\mathbf{P})\mathbf{e}_u\Vert_1}$, where $T_i(x)$ is the first kind of Chebyshev polynomials (definition see Eq. 13 in Appendix). We denote $C_2=\max_{i\leq k}{(\Vert \hat{\mathbf{v}}_i\Vert_1+\Vert \mathcal{A}\hat{\mathbf{v}}_i^+\Vert_1+\Vert \mathcal{A}\hat{\mathbf{v}}_i^-\Vert_1)}$, where $\hat{\mathbf{v}}_i^+,\hat{\mathbf{v}}_i^-$ are the positive parts and negative parts of $\hat{\mathbf{v}}_i$, respectively.

\begin{theorem}\label{thm:lanczos_push_error}
   Given $\epsilon'\in (0,1)$ , setting iteration number $k=\sqrt{\kappa}\log \frac{\kappa}{\epsilon'}$, parameter $\epsilon=\tilde{\Omega}(\frac{\epsilon'}{\kappa^{2.25}C_1})$, the error between the approximation $r'_\mathcal{G} (s,t) $ output by Algorithm \ref{algo:lanczos_local} and accurate RD value $r_\mathcal{G} (s,t) $ can be bounded by:
    $$|r_\mathcal{G} (s,t)-r'_\mathcal{G} (s,t)| \leq \epsilon'.$$
\end{theorem}

\begin{theorem}\label{thm:lanczos_push_runtime}
The total operations performed by Algorithm \ref{algo:lanczos_local} can be bounded by $O(\frac{1}{\epsilon}\sum_{i\leq k}{(\Vert \hat{\mathbf{v}}_i\Vert_1+\Vert \mathcal{A}\hat{\mathbf{v}}_i^+\Vert_1+\Vert \mathcal{A}\hat{\mathbf{v}}_i^-\Vert_1)}+k^{2.37})$. When setting $k=\sqrt{\kappa}\log \frac{\kappa}{\epsilon'}$, $\epsilon=\tilde{\Omega}(\frac{\epsilon'}{\kappa^{2.25}C_1})$, this is $\tilde{O}(\kappa^{2.75}C_1C_2/\epsilon')$.
    \end{theorem}

The above theoretical analysis of the Lanczos Push algorithm is based on a simple assumption.
\begin{assumption}\label{assump:lzpush_error}
    We assume $\lambda(\hat{\mathbf{T}})\subset [\lambda_n(\mathcal{A}),\lambda_2(\mathcal{A})]$.
\end{assumption}

We notice that the requirement of Assumption \ref{assump:lzpush_error} is mild. In fact, from the proof of Lemma \ref{lem:lz_quadratic_err}, we always have $\lambda({\mathbf{T}})\subset [\lambda_n(\mathcal{A}),\lambda_2(\mathcal{A})]$ for the exact Lanczos iteration (i.e., Algorithm \ref{algo:lanczos} , also Lanczos Push algorithm when setting $\epsilon=0$). We also show that this assumption holds in various datasets in our experiments (details in Sec. \ref{sec:exp}).

\stitle{Understanding the theoretical result.} At a first glance, the result of Theorem \ref{thm:lanczos_push_error} and Theorem \ref{thm:lanczos_push_runtime} maybe not clear enough due to the existence of $C_1$ and $C_2$, we hereby provide some further discussions. First, we provide a simpler version of Theorem \ref{thm:lanczos_push_runtime} by proving the upper bound of $C_1$ and $C_2$.

\begin{corollary}\label{coro:lzpush_time_error}
    The runtime of Algorithm \ref{algo:lanczos_local} is $\tilde{O}(\kappa^{2.75}\sqrt{nm}/(\epsilon\sqrt{d}))$ to compute the $\epsilon$-absolute error approximation of $r_\mathcal{G} (s,t) $. Specifically, $C_1\leq \sqrt{m}$ and $C_2\leq 3\sqrt{n}$.
\end{corollary}

\comment{
\begin{proof}
   \textcolor{blue}{First, by the property of Chebyshev polynomials, we have $\max_{x\in [-1,1]}{|T_i(x)|}=1$ for any $i\in \mathbb{N}^+$. By the fact that the eigenvalues of the probability transition matrix $\mathbf{P}$ relies in $[-1,1]$, we have $\Vert T_i(\mathbf{P})\Vert_2\leq 1$. By Cauchy-Schwarz inequality, $\Vert T_i(\mathbf{P})\mathbf{e}_u\Vert_1\leq \sqrt{n}\Vert T_i(\mathbf{P})\mathbf{e}_u\Vert_2\leq \sqrt{n}$. As a result, $C_1\leq \sqrt{n}$.}

    \textcolor{blue}{In each iteration we maintain $\Vert \hat{\mathbf{v}}_i\Vert_2=1$. So again, by Cauchy-Schwarz inequality and by $\Vert \mathcal{A}\Vert_2\leq 1$, we have $\Vert \mathcal{A}\hat{\mathbf{v}}_i^+\Vert_1\leq \sqrt{n}\Vert \mathcal{A}\hat{\mathbf{v}}_i^+\Vert_2\leq \sqrt{n}\Vert \hat{\mathbf{v}}_i^+\Vert_2\leq \sqrt{n}\Vert \hat{\mathbf{v}}_i\Vert_2=\sqrt{n}$, and the same upper bound also holds for $\Vert \mathcal{A}\hat{\mathbf{v}}_i^-\Vert_1$. As a result, $C_2\leq 3\sqrt{n}$. So the runtime bound of Algorithm \ref{algo:lanczos_local} can be simplified by $\tilde{O}(\kappa^{2.75}n/\epsilon)$.}
\end{proof}
}

Corollary \ref{coro:lzpush_time_error} establishes that Lanczos Push achieves sublinear runtime dependence on graph size $m$ while maintaining sub-cubic dependence on the condition number $\kappa$. This result confirms our key intuition: the algorithm preserves faster $\kappa$-dependent convergence rate while operating in sublinear time in terms of $m$. Furthermore, our analysis provides the first theoretical justification for implementing Lanczos in a local manner. However, our experimental results consistently outperform the worst-case bounds in Corollary \ref{coro:lzpush_time_error}, suggesting these theoretical guarantees may not be tight. In particular, the empirical behavior of $C_1$ and $C_2$ demonstrates better scaling than simply using the Cauchy-Schwarz inequality in our proof. Moreover, our experiments (Exp-9, Section 6) reveal an intriguing pattern: $C_1$ and $C_2$ tend to grow slower on poor expansion graphs (e.g., road networks) and faster on good expansion graphs (e.g., social networks). This finding provides further evidence that Lanczos Push is particularly well-suited for poor expansion graphs, and we hope that there will be sharper theoretical bounds or further findings to explain this phenomenon in future studies.


\comment{
 \subsubsection{Sketch of the proof for Theorem \ref{thm:lanczos_push_error}.} 
 
Due to space limits we only provide a sketch, see our full version for the complete proof ~\cite{full-version}. Generally speaking, the proof of Theorem \ref{thm:lanczos_push_error} is split into three steps. 
 
 \stitle{Step 1.} For the first step, we observe that the error between $r'_\mathcal{G}(s,t)$ and $r_\mathcal{G}(s,t)$ can be expressed by the following formula:
 \begin{align*}
     |r'_\mathcal{G}(s,t)-r_\mathcal{G}(s,t)|/\left(1/d_s+1/d_t\right)=|\mathbf{v}_1^T (\mathbf{I}-\mathcal{A})^\dagger\mathbf{v}_1-\mathbf{v}_1^T\hat{\mathbf{V}}(\mathbf{I}-\hat{\mathbf{T}})^{-1}\mathbf{e}_1|.
 \end{align*}

Where this is further split by:

\begin{align*}
    |\mathbf{v}_1^T (\mathbf{I}-\mathcal{A})^\dagger \mathbf{v}_1-\mathbf{v}_1^T \hat{\mathbf{V}}(\mathbf{I}-\hat{\mathbf{T}})^{-1}\mathbf{e}_1| 
     &\leq |\mathbf{v}_1^T (\mathbf{I}-\mathcal{A})^\dagger \mathbf{v}_1 - \mathbf{v}_1^T p(\mathcal{A}) \mathbf{v}_1 | \\
   & + | \mathbf{v}_1^T p(\mathcal{A}) \mathbf{v}_1 -\hat{\mathbf{v}}_1^T \hat{\mathbf{V}}p(\hat{\mathbf{T}})\mathbf{e}_1  | \\
    &+ |  \hat{\mathbf{v}}_1^T \hat{\mathbf{V}} p(\hat{\mathbf{T}})\mathbf{e}_1 -  \mathbf{v}_1^T \hat{\mathbf{V}}(\mathbf{I}-\hat{\mathbf{T}})^{-1}\mathbf{e}_1| .
\end{align*}

For $p$ be some given polynomial of degree at most $k$. For the first term of the right hand side of the above inequality, we show that this is smaller than $\frac{\epsilon'}{8}$ by setting the iteration number $k=O(\sqrt{\kappa}\log (\frac{\kappa}{\epsilon'}))$. For the third term, by Assumption \ref{assump:lzpush_error} (i), we also prove this is smaller than $\frac{\epsilon'}{8}$. So the above inequality is further bounded by:
\begin{align*}
    |\mathbf{v}_1^T (\mathbf{I}-\mathcal{A})^\dagger \mathbf{v}_1-\hat{\mathbf{v}}_1^T \hat{\mathbf{V}}(\mathbf{I}-\hat{\mathbf{T}})^{-1}\mathbf{e}_1| 
    \leq| \mathbf{v}_1^T p(\mathcal{A}) \mathbf{v}_1 -\hat{\mathbf{v}}_1^T \hat{\mathbf{V}}p(\hat{\mathbf{T}})\mathbf{e}_1  | +\frac{\epsilon'}{4}
\end{align*}

Next we just need to focus on the analysis for the error term $| \mathbf{v}_1^T p(\mathcal{A}) \mathbf{v}_1 -\hat{\mathbf{v}}_1^T \hat{\mathbf{V}}p(\hat{\mathbf{T}})\mathbf{e}_1  | $.

\stitle{Step 2.} To bound the error term $| \mathbf{v}_1^T p(\mathcal{A}) \mathbf{v}_1 -\hat{\mathbf{v}}_1^T \hat{\mathbf{V}}p(\hat{\mathbf{T}})\mathbf{e}_1  | $, where $p$ is the given polynomial of degree at most $k$, we utilize the technique of Chebyshev polynomials. To reach this end, we first briefly introduce the definitions of the first and second kinds of Chebyshev polynomials here.

\begin{definition}
    The first kind of Chebyshev polynomials $\{T_l(x)\}_{l\geq 0}$ is generated by the following recurrence:
    \begin{align*}
    T_0(x)&=1, T_1(x)= x \\
    T_{l+1}(x)&=2xT_l(x)-T_{l-1}(x),  \ \ \  l\geq 1.
\end{align*}
\end{definition}
\begin{definition}
    The second kind of Chebyshev polynomials $\{U_l(x)\}_{l\geq 0}$ is generated by the following recurrence:
    \begin{align*}
    U_0(x)&=1, U_1(x)= 2x \\
    U_{l+1}(x)&=2xU_l(x)-U_{l-1}(x),  \ \ \  l\geq 1.
\end{align*}
\end{definition}

For our analysis, we choose $p(x)=\sum_{l=0}^{k}{c_lT_l(x)}$ be the linear combination of Chebyshev polynomials with coefficients $\sum_l{|c_l|}\leq \tilde{O}(\kappa)$. Therefore, we can bound the deviation between $ \mathbf{v}_1^T p(\mathcal{A})\mathbf{v}_1$ and $\hat{\mathbf{v}}_1^T\hat{\mathbf{V}}p(\hat{\mathbf{T}})\mathbf{e}_1$ by the following equation:

\begin{equation}\label{equ:lzpush_polynomial_error}
    \begin{aligned}
     \mathbf{v}_1^T p(\mathcal{A})\mathbf{v}_1 - \mathbf{v}_1^T\hat{\mathbf{V}}p(\hat{\mathbf{T}})\mathbf{e}_1&=\sum_{l=0}^{k}{c_l \mathbf{v}_1^T(\hat{\mathbf{V}}T_l(\hat{\mathbf{T}})\mathbf{e}_1-T_l(\mathcal{A})\mathbf{v}_1)}.
    \end{aligned}
\end{equation}

We define $d_l= \hat{\mathbf{V}}T_l(\hat{\mathbf{T}})\mathbf{e}_1-T_l(\mathcal{A})\mathbf{v}_1$ be the deviation constraint on the Chebyshev basis $T_l(x)$. So the error term can be now simply expressed by:
\begin{align*}
    |\mathbf{v}_1^T p(\mathcal{A})\mathbf{v}_1 - \mathbf{v}_1^T\hat{\mathbf{V}}p(\hat{\mathbf{T}})\mathbf{e}_1|&\leq \sum_{l=0}^{k}|c_l||\mathbf{v}_1^Td_l|.
\end{align*}
\stitle{Step 3.} Finally we focus on the bound of $d_l$. To this end, we prove the following Lemma.

\begin{lemma}\label{lem:bound_dl}
    $|d_l(u)|\leq 6 k^{5/2} \epsilon d_u$ for $u=s,t$ and $l\leq k$.
\end{lemma}

Putting these things altogether, we finally prove that the total error term is bounded by:
\begin{align*}
|\mathbf{v}_1^T p(\mathcal{A})\mathbf{v}_1 - \hat{\mathbf{v}}_1^T\hat{\mathbf{V}}p(\hat{\mathbf{T}})\mathbf{e}_1|
&\leq \tilde{O}(\kappa^{2.25} \epsilon \Delta).
\end{align*}

Where $\Delta=d_s=d_t$ be the degree of $s,t$ (without loss of generality we assume $s,t$ be the nodes with maximum degree, see the full version of the proof for more explanations). Therefore, by setting $\epsilon=\tilde{\Omega}(\frac{\epsilon'}{\kappa^{2.25}})$, we can make $|\mathbf{v}_1^T p(\mathcal{A})\mathbf{v}_1 - \hat{\mathbf{v}}_1^T\hat{\mathbf{V}}p(\hat{\mathbf{T}})\mathbf{e}_1|\leq \frac{\epsilon'}{4}\Delta$. Equivalently, we have:
\begin{align*}
    |\hat{r}_\mathcal{G}(s,t)-r_\mathcal{G}(s,t)|/(\frac{1}{d_s}+ \frac{1}{d_t}) & \leq 
     | \mathbf{v}_1^T p(\mathcal{A}) \mathbf{v}_1 -\hat{\mathbf{v}}_1^T \hat{\mathbf{V}}p(\hat{\mathbf{T}})\mathbf{e}_1  | +\frac{\epsilon'}{4}\\
      &\leq\frac{\epsilon'}{4}\Delta + \frac{\epsilon'}{4}.
\end{align*}

And therefore
\begin{align*}
    |\hat{r}_\mathcal{G}(s,t)-r_\mathcal{G}(s,t)| \leq (\frac{\epsilon'}{4}+\frac{\epsilon'}{4}\Delta)\frac{2}{\Delta} \leq \epsilon'.
\end{align*}

This finishes the proof.\hfill\qed
}
\comment{
\subsubsection{The proof for Theorem \ref{thm:lanczos_push_runtime}.} 
We provide the proof for the time complexity of the Lanczos Push. Recall that in each iteration $i\leq k$, we perform the subset Lanczos recurrence: $\hat{\beta}_{i+1}\hat{\mathbf{v}}_{i+1}=AMV(\mathcal{A},\hat{\mathbf{v}}_i)-\hat{\alpha}_i\hat{\mathbf{v}}_i|_{S_i}-\hat{\beta}_i\hat{\mathbf{v}}_{i-1}|_{S_{i-1}}$. So the time complexity analysis for $i^{th}$ step is split into two parts: (i) bound the operation numbers for computing $AMV(\mathcal{A},\hat{\mathbf{v}}_i)$; (ii) bound the operation numbers for computing $\hat{\alpha}_i,\hat{\beta}_i$ and other addition/subtraction operations.

\stitle{Step 1.} For the computation of $AMV(\mathcal{A},\hat{\mathbf{v}}_i)$ (Line 4-8 in Algorithm \ref{algo:lanczos_local}), recall we only perform the matrix-vector multiplication $\mathcal{A}\hat{\mathbf{v}}_i$ on edge $(u,v)$ with $|\hat{\mathbf{v}}_{i}(u)|>\epsilon \sqrt{d_ud_v}$ and $\Vert \mathcal{A}\hat{\mathbf{v}}_i \Vert_1 \leq C'$. For our analysis, we define $E_i=\{(u,v)\in \mathcal{E}:\hat{\mathbf{v}}_i(u)\geq \epsilon\sqrt{d_ud_v}\}$ be the subset of edges that has been searched for the computation of $AMV(\mathcal{A},\hat{\mathbf{v}}_i)$. Moreover, we split $\hat{\mathbf{v}}_i=\hat{\mathbf{v}}_i^+-\hat{\mathbf{v}}_i^-$, where $\hat{\mathbf{v}}_i^+$ and $\hat{\mathbf{v}}_i^-$ denotes the positive and negative parts of $\hat{\mathbf{v}}_i$, respectively. Similarly, we define $E_i^+=\{(u,v)\in \mathcal{E}:\hat{\mathbf{v}}_i^+(u)\geq \epsilon\sqrt{d_ud_v}\}$ and $E_i^-=\{(u,v)\in \mathcal{E}:\hat{\mathbf{v}}_i^-(u)\geq \epsilon\sqrt{d_ud_v}\}$. Clearly $|E_i|=|E_i^+|+|E_i^-|$. Next we bound the size of $E_i^+$ and $E_i^-$ respectively. We note that:
\begin{align*}
    \Vert \mathcal{A}\hat{\mathbf{v}}_i^+\Vert_1&=\sum_{u\in\mathcal{V}}{|\mathcal{A}\hat{\mathbf{v}}_i^+(u)|}=\sum_{u\in\mathcal{V}}{\left|\sum_{v\in \mathcal{N}(u)}{\frac{\hat{\mathbf{v}}_i^+(u)}{\sqrt{d_ud_v}}}\right|}\\
    &=\sum_{u\in\mathcal{V}}{\sum_{v\in \mathcal{N}(u)}{\frac{\hat{\mathbf{v}}_i^+(u)}{\sqrt{d_ud_v}}}}=2\sum_{(u,v)\in \mathcal{E}}{\frac{\hat{\mathbf{v}}_i^+(u)}{\sqrt{d_ud_v}}}\\
    &\geq 2\sum_{(u,v)\in E_i^+}{\frac{\hat{\mathbf{v}}_i^+(u)}{\sqrt{d_ud_v}}}\geq 2\epsilon |E_i^+|.
\end{align*}
So $|E_i^+|\leq O(\frac{\Vert \mathcal{A}\hat{\mathbf{v}}_i^+\Vert_1}{\epsilon})$. Similarly, $|E_i^-|\leq O(\frac{\Vert \mathcal{A}\hat{\mathbf{v}}_i^-\Vert_1}{\epsilon})$. So the total operations for computing $AMV(\mathcal{A},\hat{\mathbf{v}}_i)$ is $|E_i|=|E_i^+|+|E_i^-|\leq O(\frac{\Vert \mathcal{A}\hat{\mathbf{v}}_i^+\Vert_1+\Vert \mathcal{A}\hat{\mathbf{v}}_i^-\Vert_1}{\epsilon})$, which is further bounded by $O(\frac{1}{\epsilon})$ by our Assumption ~\ref{assump:lzpush_error} (iii).

\stitle{Step 2.} We observe that the number of the other operations (Line 9-17 in Algorithm \ref{algo:lanczos_local}) can be bounded by $O(|S_i|+|S_{i-1}| + supp\{\hat{\mathbf{v}}_{i+1}\})$, since we only work on the subsets $S_i,S_{i-1}$ and $ supp\{\hat{\mathbf{v}}_{i+1}\}$, where $ supp\{\hat{\mathbf{v}}_{i+1}\}$ denotes the set of non-zero entries of $\hat{\mathbf{v}}_{i+1}$. However, since we set $S_i= \{u\in \mathcal{V}: |\hat{\mathbf{v}}_i(u)|> \epsilon d_u\}$, we have $|S_i|\leq \frac{\Vert \hat{\mathbf{v}}_i\Vert_1}{\epsilon}$, which is further bounded by $O(\frac{1}{\epsilon})$ by Assumption \ref{assump:lzpush_error} (iii). Similarly, $|S_{i-1}|\leq \frac{\Vert \hat{\mathbf{v}}_{i-1}\Vert_1}{\epsilon}$ and further bounded by $O(\frac{1}{\epsilon})$. In addition, we observe that $ supp\{\hat{\mathbf{v}}_{i+1}\}\leq |E_i|+|S_i|+|S_{i-1}|\leq O(\frac{1}{\epsilon})$. Therefore, the total operations of Algorithm \ref{algo:lanczos_local} in each iteration $i$ is bounded by $O(|E_i|+|S_i|+|S_{i-1}| + supp\{\hat{\mathbf{v}}_{i+1}\})=O(\frac{1}{\epsilon})$. So the number of operations of Algorithm \ref{algo:lanczos_local} in total $k$ iterations is bounded by $O(\frac{k}{\epsilon})$. Finally, the computation of the final output $\left(\frac{1}{d_s}+\frac{1}{d_t}\right)\hat{\mathbf{v}}_1^T \hat{\mathbf{V}} (\mathbf{I}-\hat{\mathbf{T}})^{-1}\mathbf{e}_1$ takes $O(k^{2.37})$ time, so the total time complexity is $O(\frac{k}{\epsilon}+k^{2.37})$. This finishes the proof. \hfill\qed

}
\comment{
\stitle{Remark.} There is an additional remark that the Lanczos Push algorithm experimentally scales better than the $O(n)$ time complexity. Actually we can remove the $O(n)$ term in the theoretical analysis when assuming the following two additional conditions:

\begin{assumption}\label{assump2}
The following two additional proposition holds:

(i) $\Vert T_l(\mathbf{P}) \mathbf{e}_u \Vert_1\leq C$ for $u=s,t$ for some constant $C$ for $\forall l\leq k$.

(ii) $\Vert \hat{\mathbf{v}}_i \Vert_1,\Vert \mathcal{A}\hat{\mathbf{v}}_i^+ \Vert_1,\Vert \mathcal{A}\hat{\mathbf{v}}_i^- \Vert_1 \leq C'$ for some constant $C'$ for $\forall i\leq k$. Where $\hat{\mathbf{v}}_i^+,\hat{\mathbf{v}}_i^-$ are the positive parts and negative parts of $\hat{\mathbf{v}}_i$, respectively.
\end{assumption}

Based on the additional assumption, we prove a stronger result for the complexity of Lanczos Push. The details of the reasonability of the additional assumption are provided in Sec. \ref{sec:exp}.

\begin{theorem}\label{thm:lzpush_err_improve}
    When assumption \ref{assump2} holds, Algorithm \ref{algo:lanczos_local} outputs an approximation $\hat{r}_\mathcal{G}(s,t)$ that satisfies the $\epsilon$-absolute error guarantee in $\tilde{O}(\kappa^{2.75}/\epsilon)$ runtime.
\end{theorem}
}

\section{Lower Bounds} ~\label{sec:lower-bounds}
In this section, we establish the lower bound of $\kappa$ dependency for the approximation of $r_\mathcal{G}(s,t)$. Our proof is based on the former result proposed by Cai et al. ~\cite{cai2023effective}. Below, we briefly introduce the result established in ~\cite{cai2023effective}. First, we consider two $d$-regular expander graphs $H_1$, $H_2$ (that is, $H_1$, $H_2$ has condition number $\kappa=O(1)$). Specifically, we choose $H_1$, $H_2$ to be two copies of Ramanujan graph ~\cite{morgenstern1994existence} on $n/2$ vertices (three regular graph with $\kappa=O(1)$). Next, we connect a node $s\in H_1$ and $t\in H_2$ by an edge $(s,t)$, generating the  graph $\mathcal{G}_1$. Clearly, $r_{\mathcal{G}_1}(s,t)=1$. For the construction of $\mathcal{G}_2$, we select a random edge $(u,u') \in H_1$ and $(v,v')\in H_2$. we remove the edges $(u,u')$ and $(v,v')$, and add two new edges $(u,v)$ and $(u',v')$. Then, ~\cite{cai2023effective} prove that after this operation, the resulting graph $\mathcal{G}_2$ has the effective resistance $r_{\mathcal{G}_2}(s,t)\leq 0.99$. However, any algorithm that distinguishes $\mathcal{G}_1$ and $\mathcal{G}_2$ with positive probability requires at least $\Omega(n)$ queries. For better understanding, we summarize this result into the following theorem.

\begin{figure}
    \centering
    \includegraphics[scale=0.2]{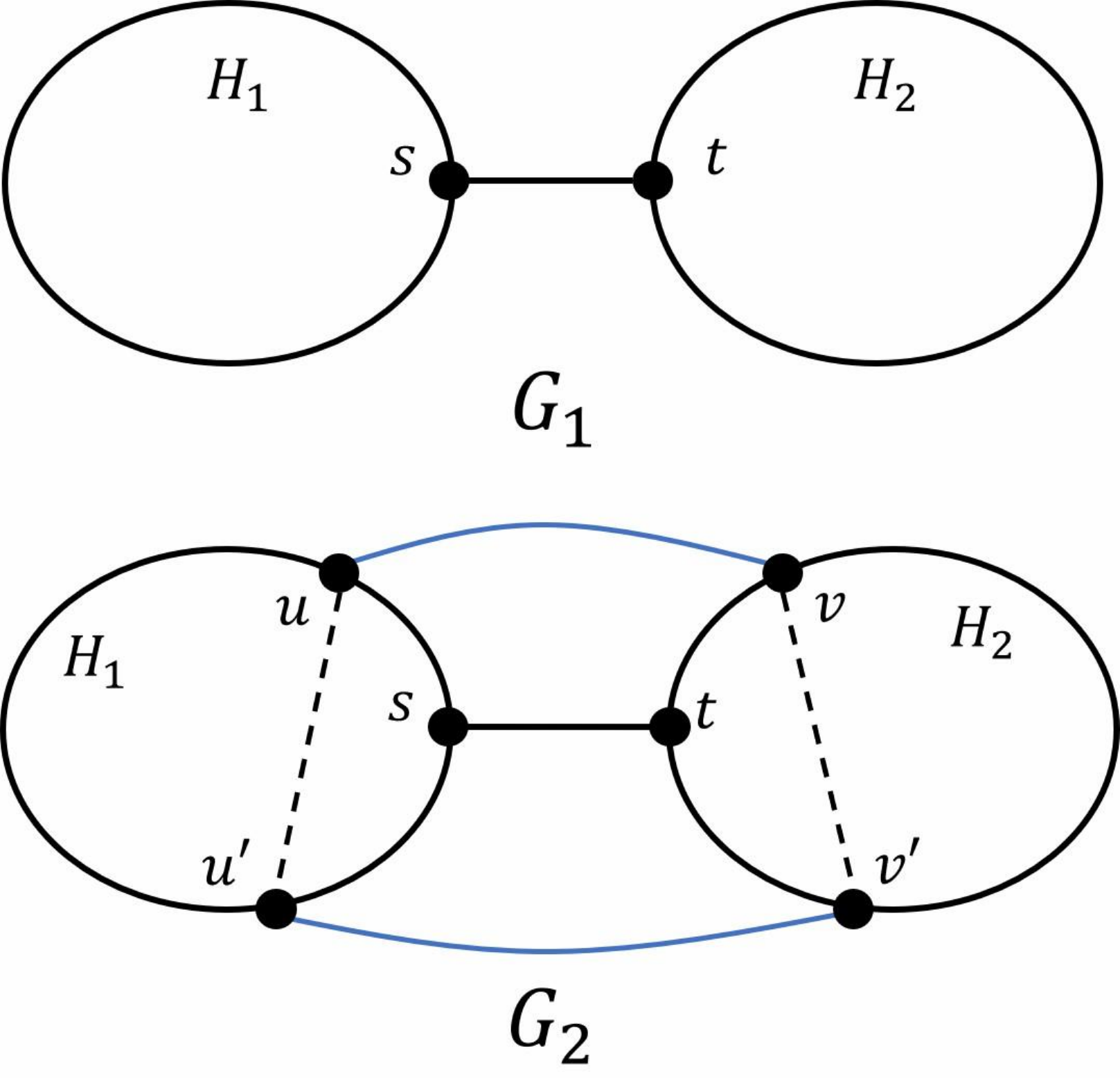} \vspace{-0.3cm}
    \caption{The construction of the lower bound}\vspace{-0.5cm}
    \label{fig:construction_lower_bound}
\end{figure}

\begin{theorem}{(Theorem 1.1 in~\cite{cai2023effective})}\label{thm:lower_bound_1}
  If $\mathcal{G}_1$ and $\mathcal{G}_2$ are constructed in the manner described above, then we have $|r_{\mathcal{G}_1}(s,t)-r_{\mathcal{G}_2}(s,t)|\geq 0.01$. But, any (randomized) algorithm that distinguish $\mathcal{G}_1$ and $\mathcal{G}_2$ with probability $\geq 0.6$ requires at least $\Omega (n)$ queries.
\end{theorem}
Based on Theorem \ref{thm:lower_bound_1}, we can prove that the condition number of $\mathcal{G}_1$ is $\kappa=\Theta(n)$, which is stated in the following lemma.

\begin{lemma}\label{lem:kappa_G1}
   If $\mathcal{G}_1$ is constructed in the manner described above, we have $\kappa (\mathcal{G}_1)= \Theta(n)$.
\end{lemma}

\comment{
\begin{proof}
   The proof is partially inspired by Lemma 4.1 in ~\cite{andoni2018solving}. Recall that the condition number is defined as $\kappa=\frac{2}{\lambda_2}$, where $\lambda_2$ is the second smallest eigenvalue of the normalized Laplacian matrix. First, the upper bound of $\lambda_2$ is given by Cheeger's inequality (Theorem \ref{thm:cheeger}). By the definition of  conductance, we have $\frac{2}{3n}=\phi (V(H_1))\geq\phi_{\mathcal{G}_1}\geq \frac{1}{2}\lambda_2$, where the first equality holds because Remanujan graph is $3$ regular graph and $(s,t)$ is the only cut edge between $V(H_1)$ and $V(H_2)$. Therefore, we have $\lambda_2\leq \frac{4}{3n}=O(\frac{1}{n})$. Next, we prove that $\lambda_2\geq \Omega(\frac{1}{n})$. To reach this end, we consider the Laplacian matrix of graph $\mathcal{G}_1$. By the construction, we have
$L_{\mathcal{G}_1}=L_{H_1}+L_{H_2}+\delta_{s,t}\delta_{s,t}^T$, where $\delta_{s,t}$ is the indicator vector that takes value $1$ at $s$, $-1$ at $t$ and $0$ otherwise. Subsequently, for any vector $x\perp \mathbf{1}$, we split $x$ by $x=x_1+x_2$, where $x_1$ only takes value in $u\in V(H_1)$ and $x_2$ only takes value in $u\in V(H_2)$ respectively. Since $x\perp \mathbf{1}$, we define $m=\frac{2}{n}\sum_{u\in V(H_1)}{x_1(u)}=-\frac{2}{n}\sum_{u\in V(H_2)}{x_2(u)}$. Next, we let $x_1'=x_1-m\mathbb{I}_{V(H_1)}$ and $x_2'=x_2+m\mathbb{I}_{V(H_2)}$ ($\mathbb{I}_{V(H_1)}$ is the indicator vector that takes value $1$ for $u\in V(H_1)$ and $0$ otherwise). Now we suppose the opposite that for any small constant $c$, there exist sufficiently large $n$ such that $\lambda_2(L_{\mathcal{G}_1})\leq \frac{c}{n}$. Then there exists some $x\perp \mathbf{1}$ with $\Vert x \Vert_2=1$, such that $x^T L_{\mathcal{G}_1}x\leq \frac{c}{n}$. Therefore, we can obtain the following equation.
    \begin{equation}
        x^T L_{\mathcal{G}_1}x=x_1^T L_{H_1}x_1+ x_2^T L_{H_2} x_2 + (x(s)-x(t))^2.
    \end{equation}
    Then, the three terms of the right hand side of the above equation all $\leq \frac{c}{n}$. Using the fact that $H_1$ is an expander ($\lambda_2(L_{H_1})\geq \frac{1}{2}$ by $H_1$ Ramanujan graph), we have:
\begin{equation}
    \frac{1}{2}\Vert x_1'\Vert_2^2\leq x_1'^TL_{H_1}x_1'=x_1^TL_{H_1}x_1\leq \frac{c}{n}.
\end{equation}
Thus, $\Vert x_1'\Vert_2\leq \sqrt{\frac{2c}{n}}$. And similarly, $\Vert x_2'\Vert_2\leq \sqrt{\frac{2c}{n}}$. However, since $\Vert x\Vert_2=1$, we have:
\begin{equation}
    1=\Vert x\Vert_2\leq \Vert x_1'\Vert_2+\Vert x_2'\Vert_2+m\sqrt{n}.
\end{equation}
Therefore, we have $m\geq \frac{1}{2\sqrt{n}}$ when setting $c\leq \frac{1}{32}$. However, on the other hand, we have the following inequation.
\begin{equation}\small
    (x(s)-x(t))^2=(x_1'(s)+m-(x_2'(t)-m))^2\geq(2m-2\sqrt{\frac{2c}{n}})^2\geq \frac{1}{4n}.
\end{equation}
which contradicts to $(x(s)-x(t))^2\leq \frac{c}{n}$ since $c\leq \frac{1}{32}$. Therefore, there must exist some constant $c$, such that $\lambda_2(L_{\mathcal{G}_1})\geq \frac{c}{n}$ for any sufficiently large $n$. Therefore, for the normalized Laplacian matrix, we have $\lambda_2(\mathcal{L}_{\mathcal{G}_1})\geq \frac{1}{4}\lambda_2(L_{\mathcal{G}_1})=\Omega(\frac{1}{n})$. Putting it together, we have $\lambda_2=\lambda_2(\mathcal{L}_{\mathcal{G}_1})=\Theta (\frac{1}{n})$. This is equivalent to $\kappa=\Theta (n)$.
\end{proof}
}

Combining with Theorem \ref{thm:lower_bound_1} and Lemma \ref{lem:kappa_G1}, we can prove that given any $\epsilon< 0.01$, any algorithm computes the $\epsilon$-approximate RD value $r_\mathcal{G}(s,t)$ requires at least $\Omega(\kappa)$ queries. 

\begin{theorem}\label{thm:kappa_lower_bound}
    Any algorithm that approximates $r_\mathcal{G}(s,t)$ with absolute error $\epsilon\leq 0.01$ with success probability $\geq 0.6$ requires $\Omega(\kappa)$ queries.
\end{theorem}

\comment{
\begin{proof}
    This statement holds immediately by combining with Theorem \ref{thm:lower_bound_1} and Lemma \ref{lem:kappa_G1}. If there is an algorithm that approximates $r_\mathcal{G}(s,t)$ with absolute error $\epsilon\leq 0.01$, then this algorithm distinguishes $\mathcal{G}_1$ from $\mathcal{G}_2$ (since the $s,t$-RD value of $\mathcal{G}_1$ and $\mathcal{G}_2$ differ at least $0.01$, by Theorem \ref{thm:lower_bound_1}). However, by Theorem \ref{thm:lower_bound_1}, any algorithm that distinguishes $\mathcal{G}_1$ from $\mathcal{G}_2$ requires $\Omega(n)$ queries, and $\kappa=\Theta(n)$ by Lemma \ref{lem:kappa_G1}. So any algorithm that requires $\Omega(\kappa)$ queries to distinguish $\mathcal{G}_1$ from $\mathcal{G}_2$ .
\end{proof}
}

This indicates that the time complexity of the local
algorithms is indeed relevant to the condition number $\kappa$, which has
not been emphasized in previous works for local RD computation.  We anticipate this $\Omega(\kappa)$ is not tight. Unfortunately, to the best of our knowledge, there are no better constructions, thus we leave how to match the upper and lower bounds as an interesting future direction. 

\section{Experimental Evaluation}\label{sec:exp}

\begin{table}[t!]
\small
	\centering
	\caption{Statistics of datasets} \vspace{-0.2cm}
        \label{tab:dataset}
	{
		\label{Datasets for vertex classification}
		\centering
		
		\begin{tabular}{cccccc}
			\toprule
			Datasets&$n$&$m$& $m/n$  & condition number $\kappa$ \\
			\midrule
			\dblp & 317,080 & 1,049,866  & 3.31 &717.62\\
			\youtube&1,134,890&2,987,624 &2.63 & 907.85 \\
			\livejournal&4,846,609&42,851,237&8.84&2785.51\\
			\orkut&3,072,626&11,718,5083&38.13&381.46\\
            \midrule
            \powergrid& 4,941 &6,594 & 1.33 & 7299.27\\
            \roadca&1,965,206 &2,766,607 & 1.41 & 64516.12\\
            \roadpa&1,088,092  & 1,541,898 & 1.42 & 24390.24\\
            \roadtx&1,379,917 & 1,921,660 & 1.39 & 28985.5\\
            \midrule
            \er & $10^2\sim 10^6$ & $n\log n$ & $\log n$ & $O(1)$\\
            \ba & $10^2\sim 10^6$ & $n\log n$ & $\log n$ & $O(1)$\\
			\bottomrule
		\end{tabular}
	}\vspace{-0.4cm}
\end{table}

\subsection{Experimental Setup}
We use 8 publicly available datasets\footnote{All the $8$ datasets can be downloaded from http://snap.stanford.edu/ and http://konect.cc/networks/} 
of varying sizes (Table~\ref{tab:dataset}), including $4$ social networks (\dblp, \youtube, \livejournal, \orkut) which serve as standard benchmarks 
for various graph algorithms \cite{liao2023resistance,wang2021approximate,wang2023singlenode,yang2023efficient,wu2021unifying} and $4$ datasets that 
represent hard cases (due to large $\kappa$) for previous RD algorithms ~\cite{liao2024efficient}. Among them, \roadca, \roadpa, and \roadtx are road networks, 
and \powergrid is an infrastructure network, statistically similar to road networks. Additionally, we artificially generate the classic Erdős-Rényi (\er) and  
Barabási-Albert (\ba) graphs with size $n$ varying from $10^2\sim 10^6$ and $m=n\log n$ for the scalability testing. To compute the condition number of these 
datasets, we use the standard Power Method to compute the second largest eigenvalue $\mu_2$ of the normalized adjacency matrix, and by definition, the condition 
number is $\kappa=\frac{2}{1-\mu_2}$. We set the threshold of the Power Method to $10^{-9}$ to make the approximation of $\mu_2$ sufficiently accurate. To evaluate 
the approximation errors of different algorithms, we compute \textit{ground-truth}  RD value using Power Method (\powermethod) with a sufficiently large truncation 
step $N=20000$. We compare our proposed algorithms, Lanczos iteration (\lanczos) and Lanczos Push (\lzpush) with 6 state-of-the-art (SOTA) baselines: 
\powermethod ~\cite{yang2023efficient}, \geer ~\cite{yang2023efficient}, \push ~\cite{liao2023resistance}, \bipush ~\cite{liao2023resistance}, also two very recent 
studies \bisper ~\cite{cui2025mixing} and \fastrd ~\cite{lu2025resistance}. Since the algorithms proposed in ~\cite{yang2023efficient,liao2023resistance} outperform 
the algorithms proposed in ~\cite{peng2021local} in all datasets, we no longer compare our algorithms with the algorithms proposed in ~\cite{peng2021local}. 
We randomly generate $50$ source nodes and $50$ sink nodes as query sets and report the average performance across them for different algorithms. 
By Definition \ref{def:ER-err}, we use absolute error (denoted as Absolute Err) to evaluate the estimation error of different algorithms. 
We conduct all experiments on a Linux 20.04 server with an Intel 2.0 GHz CPU and 128GB memory. 
All algorithms are implemented in C++ and compiled using GCC 9.3.0 with -O3 optimization. 
The source code is available at  \url{https://github.com/Ychun-yang/LanczosPush}.

\vspace{-0.2cm}
\subsection{Overall Empirical Results}

\begin{figure*}[t!]
    \vspace{-0.3cm}
    \centering
 	\subfigure[\dblp]{
            \centering
		\includegraphics[scale=0.25]{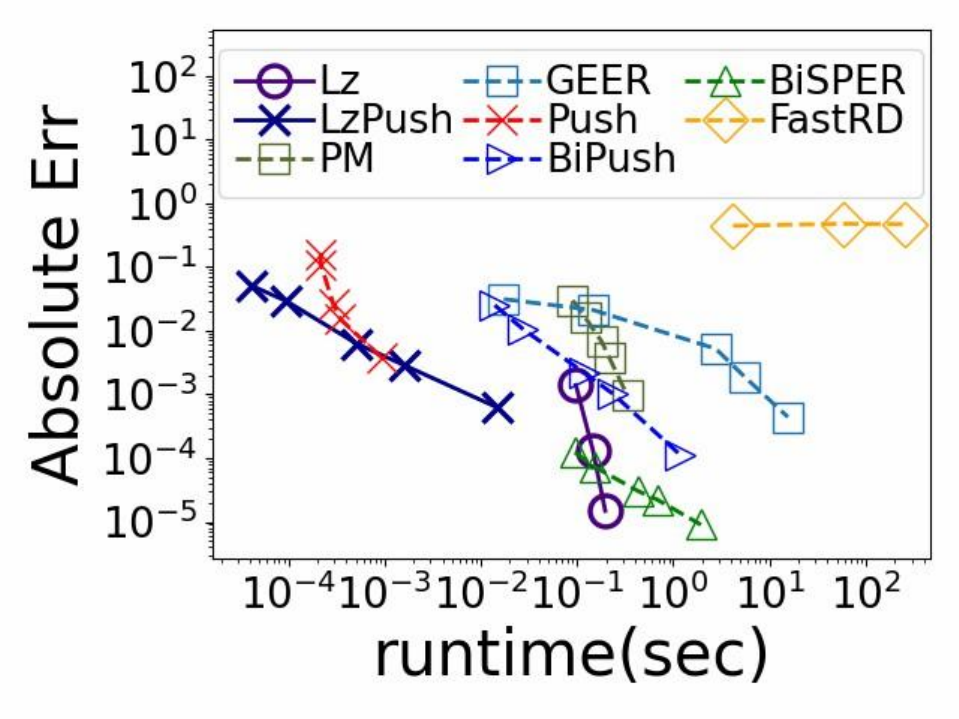}\hspace{-5mm} \label{1}
	}
	\quad
        \subfigure[\youtube]{
        \centering
		\includegraphics[scale=0.25]{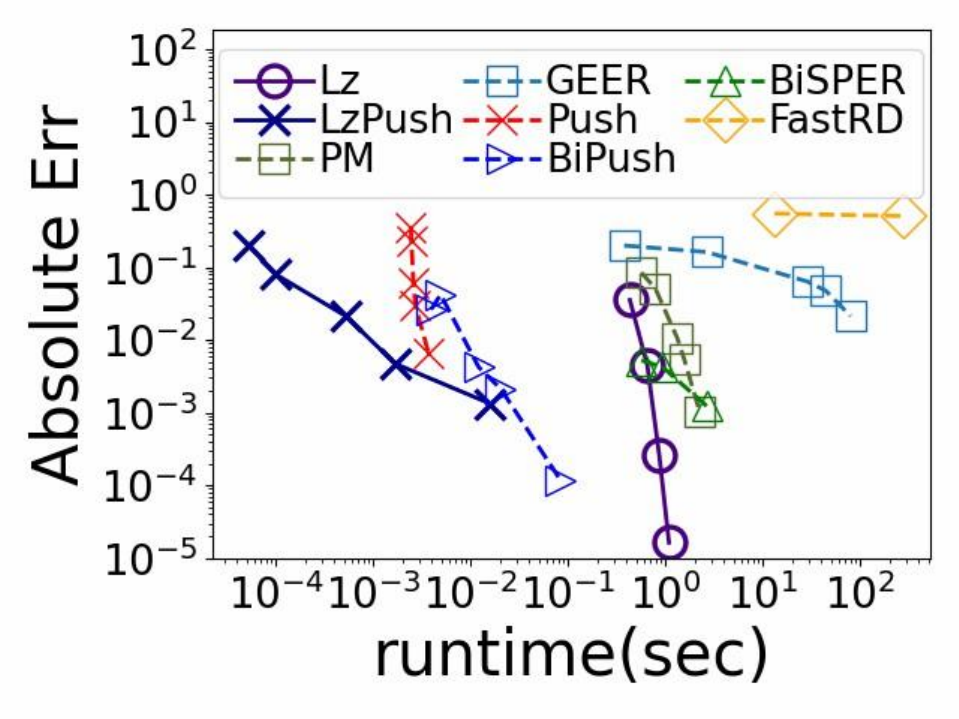}\hspace{-5mm}  \label{2}
	}
	\quad
	\subfigure[\livejournal]{
        \centering
		\includegraphics[scale=0.25]{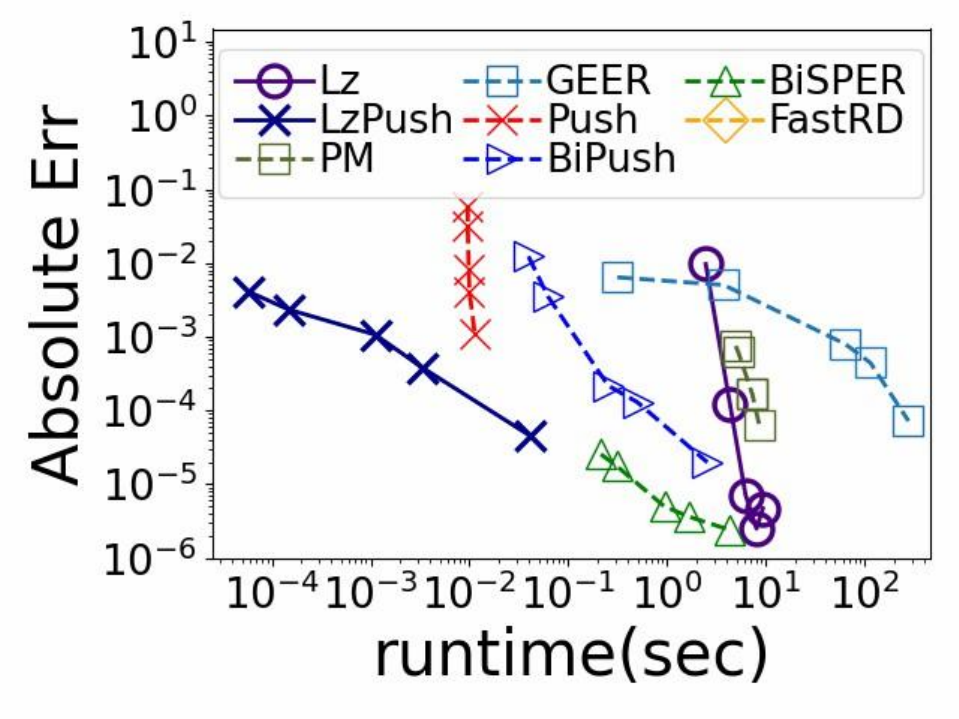}\hspace{-5mm}  \label{3}
	}
	\quad
	\subfigure[\orkut]{
        \centering
		\includegraphics[scale=0.25]{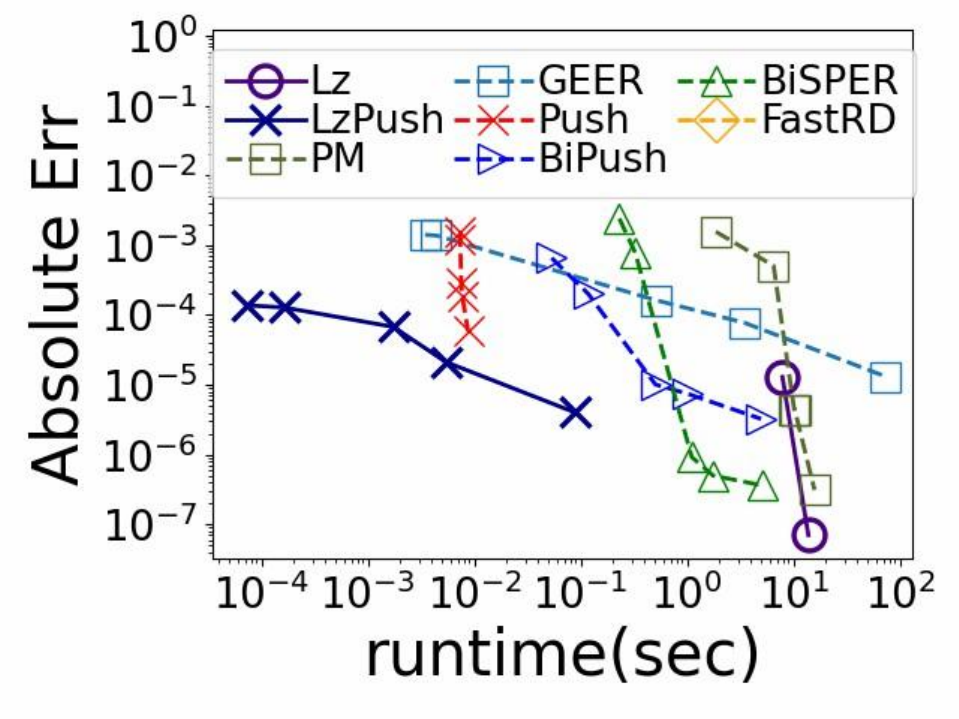}\label{4}
	}
	\caption{Query time of different algorithms for graphs with small condition number $\kappa$}\label{fig:small-condition}
\vspace{-0.2cm}
\end{figure*}

\begin{figure*}[t!]
    \vspace{-0.3cm}
    \centering
 	\subfigure[\powergrid]{
            \centering
		\includegraphics[scale=0.25]{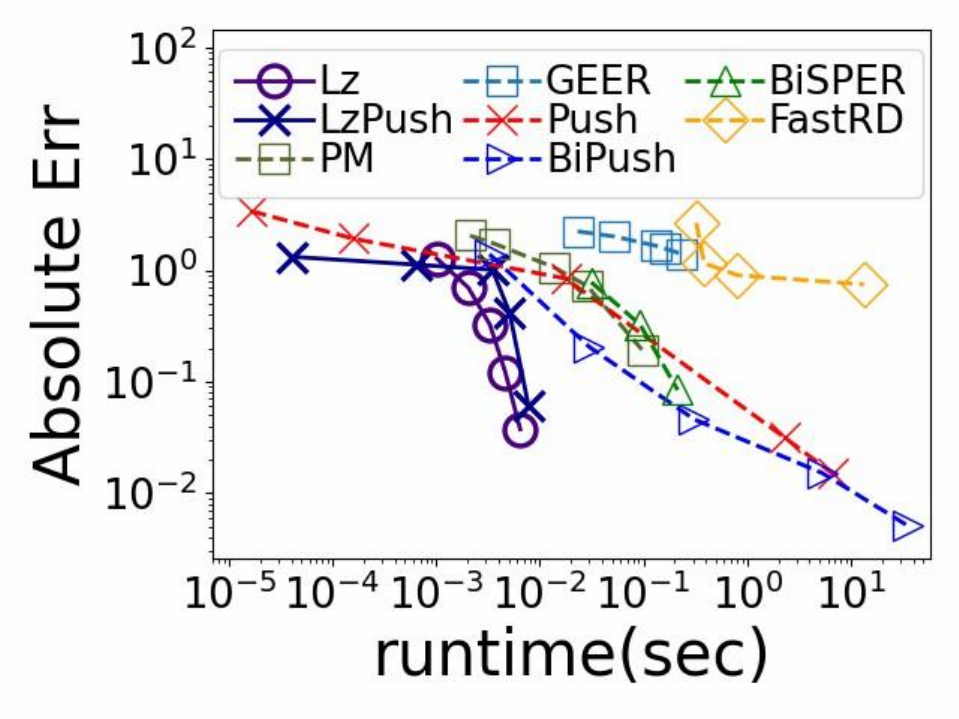}\hspace{-5mm} \label{1}
	}
	\quad
        \subfigure[\roadca]{
        \centering
		\includegraphics[scale=0.25]{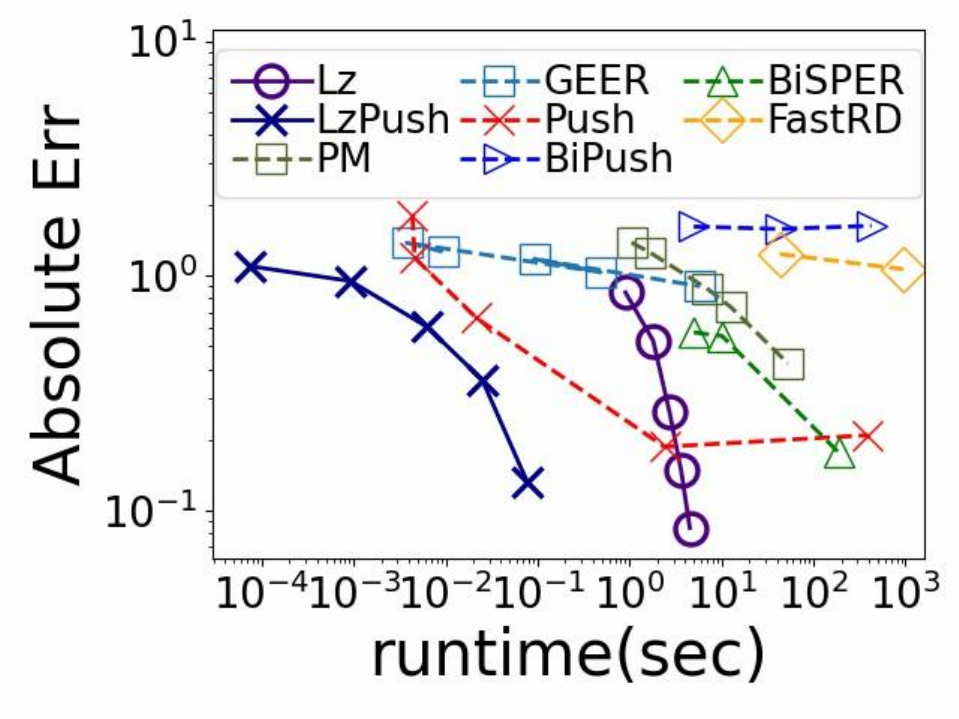}\hspace{-5mm}  \label{2}
	}
	\quad
	\subfigure[\roadpa]{
        \centering
		\includegraphics[scale=0.25]{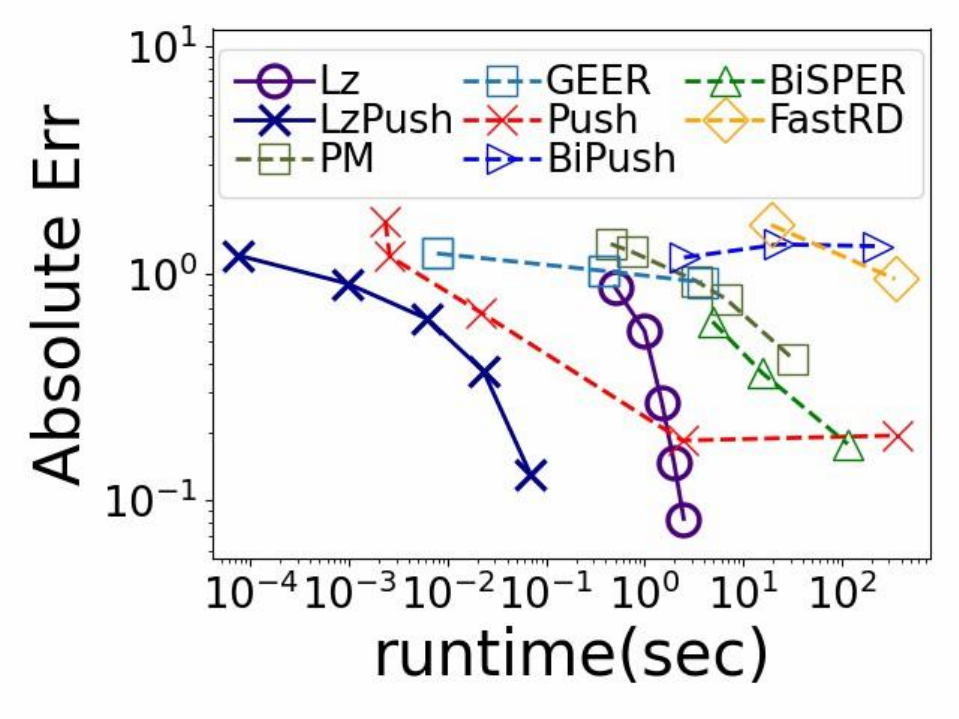}\hspace{-5mm}  \label{3}
	}
	\quad
	\subfigure[\roadtx]{
        \centering
		\includegraphics[scale=0.25]{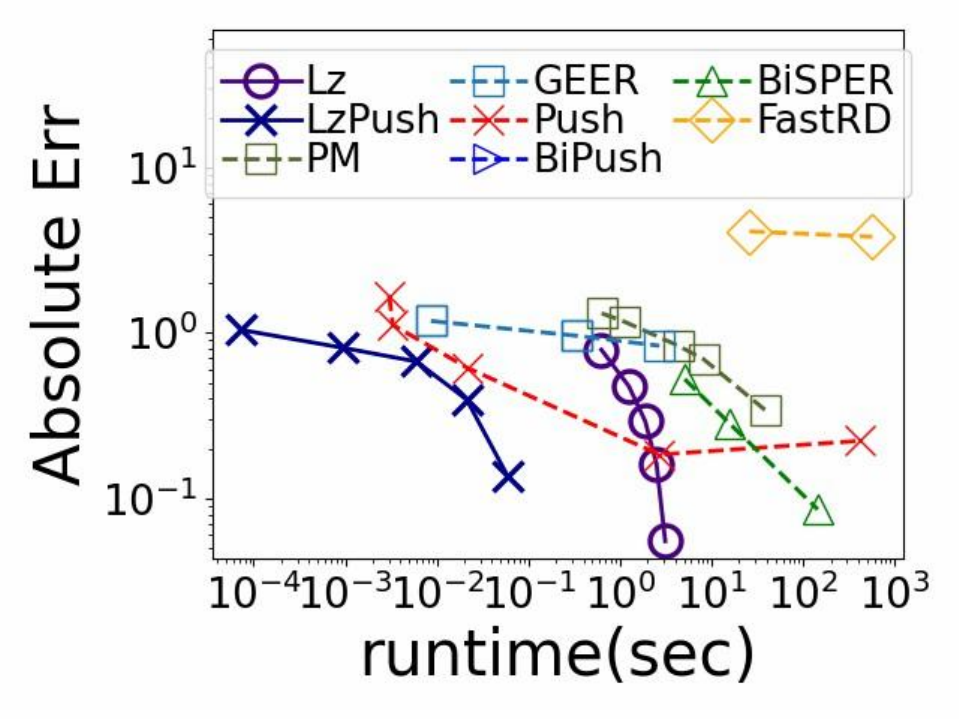}\label{4}
	}
	\caption{Query time of different algorithms for graphs with large condition number $\kappa$}\label{fig:large-condition}
\vspace{-0.2cm}
\end{figure*}

\stitle{Exp-1: Query time of different algorithms on unweighted graphs.}
In this experiment, we compare the proposed algorithms, \lanczos and \lzpush, with 6 SOTA algorithms proposed in the previous 
studies ~\cite{yang2023efficient,liao2023resistance,cui2025mixing,lu2025resistance}: 
\powermethod, \geer, \push, \bipush, \bisper and \fastrd. For \powermethod, \geer, \push, \bipush, 
following ~\cite{yang2023efficient,liao2023resistance} we vary the parameter $\epsilon$ from $10^{-1}$ to $10^{-5}$. 
For \bisper and \fastrd, following ~\cite{cui2025mixing,lu2025resistance} we vary $\epsilon$ from $0.9$ to $10^{-3}$. 
For \fastrd, since its storage is unacceptable for million-size graphs under high precision computation, 
so we only compare \fastrd under large $\epsilon$ for large datasets, e.g., $\epsilon =0.9,0.5,0.1$, following the setting as \cite{lu2025resistance}.
 For the datasets with small condition number (i.e. the four social networks), 
we set the parameters for our algorithms as follows for fair comparison: for \lanczos we vary the iteration 
step $k=[5,10,15,20,30]$; for \lzpush we vary $\epsilon=[0.1,0.05,0.01,0.005,0.001]$ and fix $k=20$. 
Fig. ~\ref{fig:small-condition} reports the query time of different algorithms for RD computation, under the Absolute Err metric. 
We make the following observations: (i) The proposed algorithm \lzpush is the most efficient and is $5\times$ to $10\times$ faster 
than all the other competitors across all datasets. (ii) Since the time complexity of \lzpush is independent of the number of vertices 
and edges, for the datasets of varying sizes, we observe that the performance of \lzpush remains almost the same. 
This is consistent with our theoretical analysis (\lzpush is a local algorithm). On the other hand, 
the performance of \lanczos and \powermethod is dependent on $m$, with their query time increasing on larger datasets (e.g., \livejournal, \orkut).

For datasets with large condition numbers (i.e. the three road networks and \powergrid), we set the parameters as follows for fair comparison: for \lanczos we vary the iteration step $k=[20,40,60,80,100]$; for \lzpush we vary $\epsilon=[10^{-3},5*10^{-4},10^{-4},5*10^{-5},10^{-5}]$ and fix $k=100$. Fig. ~\ref{fig:large-condition} reports the query time of different algorithms for RD computation under the Absolute Err metric. We make the following observations: (i) Our proposed algorithm \lzpush, is still the most efficient, while all the previous local algorithms (\geer, \push, \bipush, \bisper) fail to quickly output high-quality results (e.g.,  Absolute Err $<10^{-1}$) on three road networks. The reasons are as follows: Since the condition number $\kappa$ of the road networks is large, the random walk sampling based method requires setting a very large truncation step $l$ and very large sampling times $n_r=\tilde{O}(l^2/\epsilon ^2)$. This makes \geer inefficient. Similar observation also holds for \bisper, though the bidirectional techniques improve the efficiency of random walks. For \push and \bipush, these algorithms are heuristic, and their time complexity is related to hitting time $h(s,v)+h(t,v)$, where $v$ is a pre-selected landmark node. However, on road networks $h(s,v)$ can be $O(n)$, this makes \push and \bipush very inefficient, especially for high precision computation. Most notably, on three road networks, our \lzpush algorithm is over $1000\times$ faster than \powermethod and $50\times$ faster than \lanczos. (ii) The performance of \lzpush is similar across the four datasets, while the performance of \lanczos and \powermethod is related to $m$. Specifically, \lanczos and \powermethod are faster on small graph \powergrid, but slower on three large road networks. This is consistent with the results for the social networks.

\begin{figure}
    \centering
    \subfigure[\wdblp]{
		\includegraphics[scale=0.25]{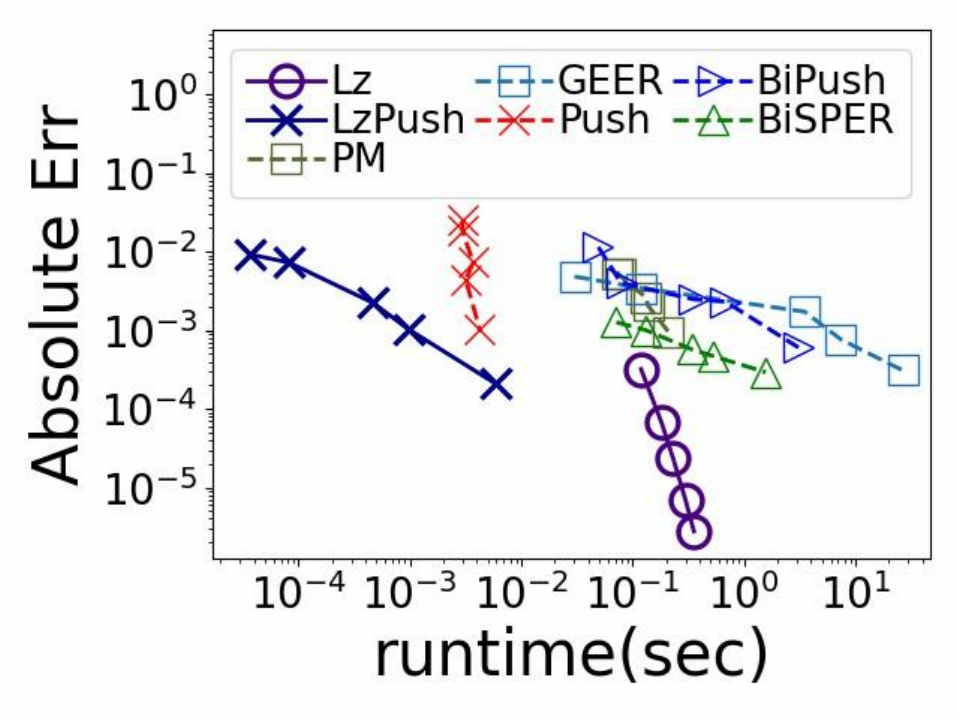}}
     \subfigure[\wlivejournal]{
		\includegraphics[scale=0.25]{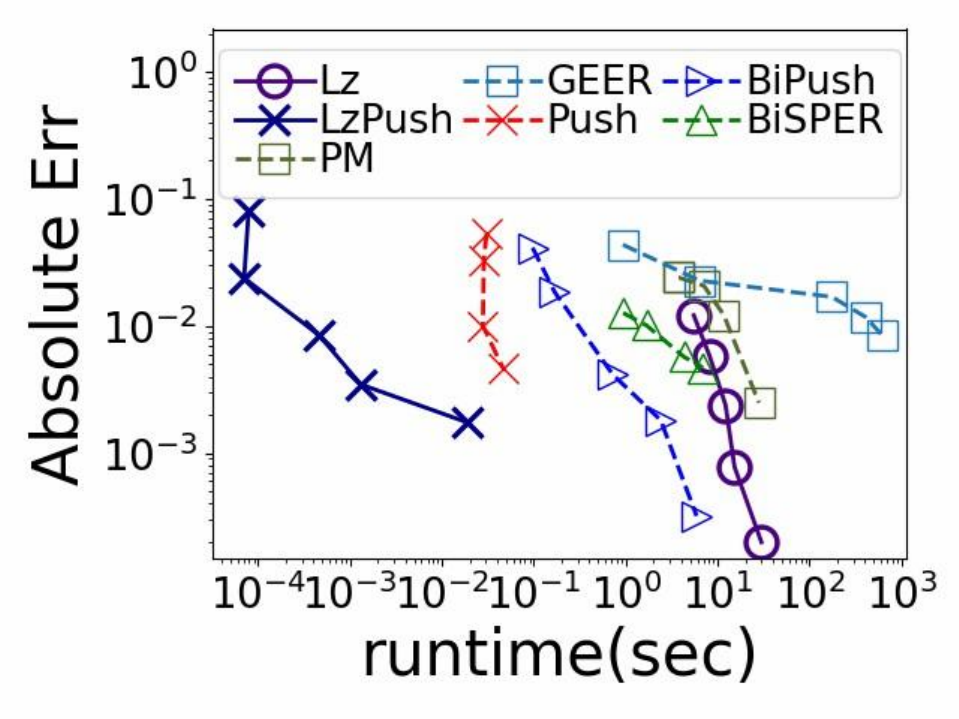}}
    \caption{Query time of different algorithms for weighted graphs}\label{fig:weight-graph}\vspace{-0.2cm}
\end{figure}

\stitle{Exp-2: Query time of different algorithms on weighted graphs.}
In this experiment, we evaluate our proposed algorithms against existing baselines on weighted graphs. Following the common practice in prior work ~\cite{wang2022edge}, we generate weighted versions of unweighted graphs by assigning each edge $e \in \mathcal{E}$ a weight equal to the number of triangles containing $e$\footnote{If $e$ is not contained in any triangle, we set $w(e)=1$ to ensure the network is connected.}. Fig. \ref{fig:weight-graph} shows the results on Dblp and LiveJournal datasets. Similar results can also be observed on the other datasets. As shown in Fig. \ref{fig:weight-graph}, the results demonstrate consistent behavior with our unweighted graph experiments: our \lzpush algorithm remains the most efficient algorithm among all compared methods.

\begin{figure}
    \centering
    \subfigure[\er]{
		\includegraphics[scale=0.20]{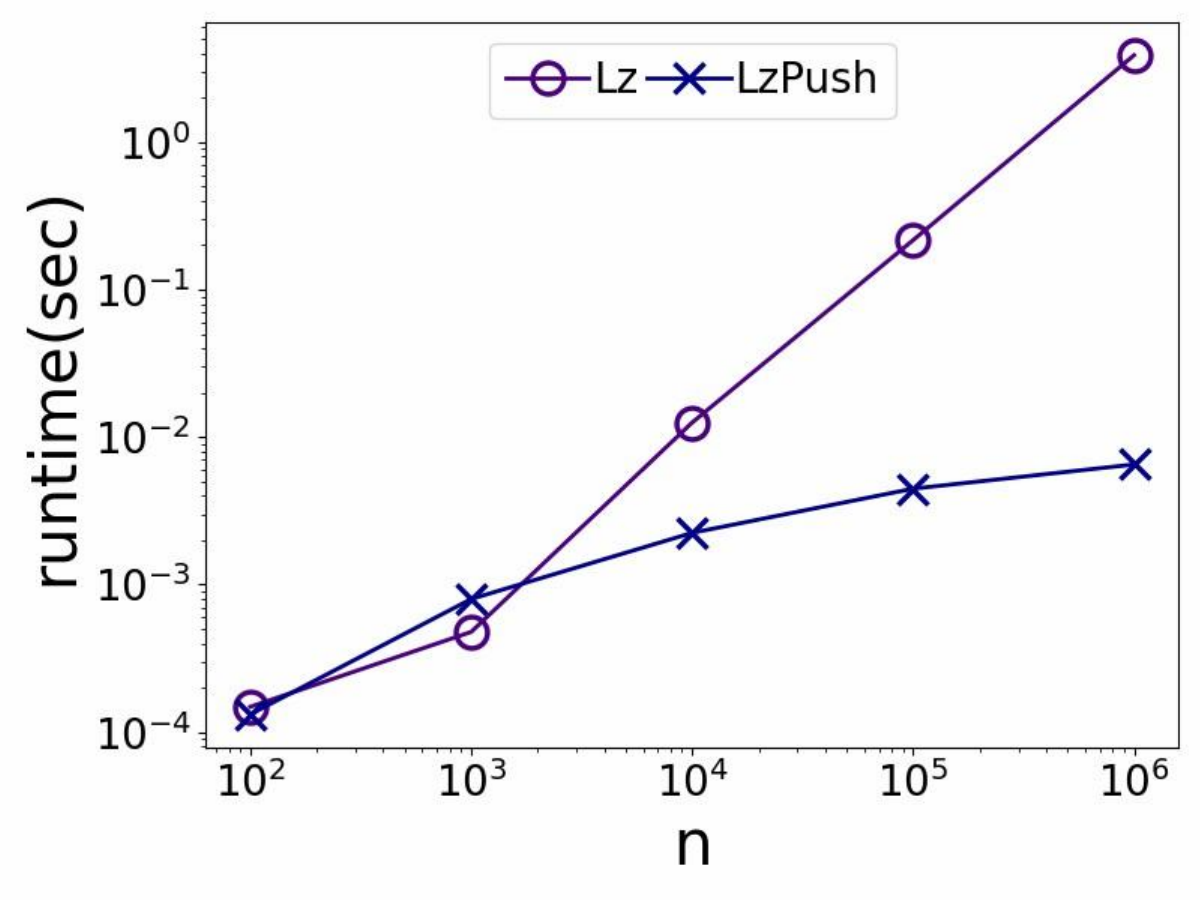}}
     \subfigure[\ba]{
		\includegraphics[scale=0.20]{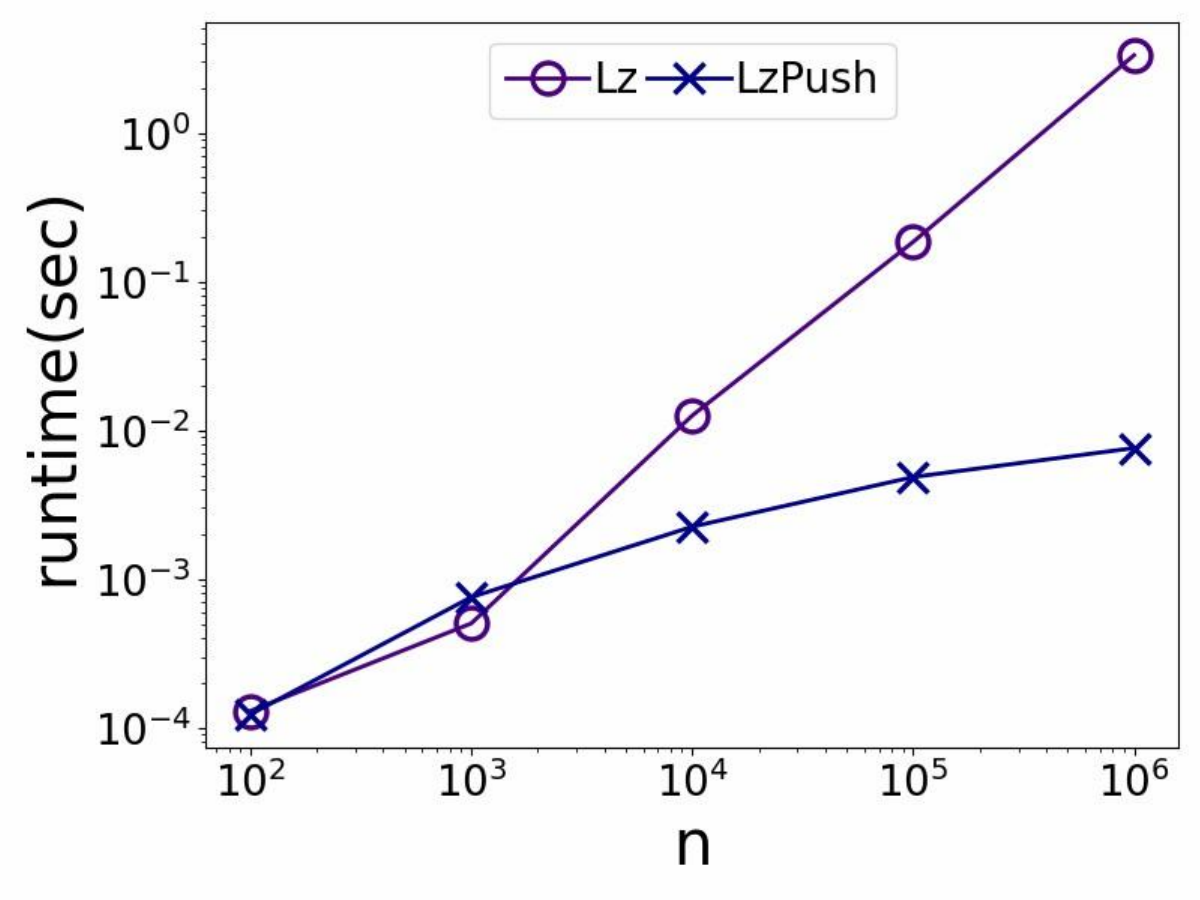}}
    \caption{Scalability testing on synthetic graphs}\label{fig:scalability}\vspace{-0.2cm}
\end{figure}

\stitle{Exp-3: Scalability testing on synthetic graphs.} In this experiment, we compare the scalability between our proposed \lanczos and \lzpush algorithms. We use the classic \er and \ba graphs with varying $n$ from $10^2$ to $10^6$ and $m=n\log n$. We set the iteration number $k=20$, threshold $\epsilon=0.005$, and compare the runtime between \lanczos and \lzpush. Fig. \ref{fig:scalability} shows the result. As can be seen, the runtime of \lanczos grows almost linear with respect to the graph size $n$, while the runtime of \lzpush is sublinear to the growth of $n$. These results further indicate that \lanczos is a nearly linear global algorithm and  \lzpush is a local algorithm, which is consistent with our theoretical analysis. Thus, these findings provide strong evidence that the proposed \lanczos and \lzpush have excellent scalability over massive graphs.

\begin{figure}
    \centering
    \subfigure[low precision]{
		\includegraphics[scale=0.20]{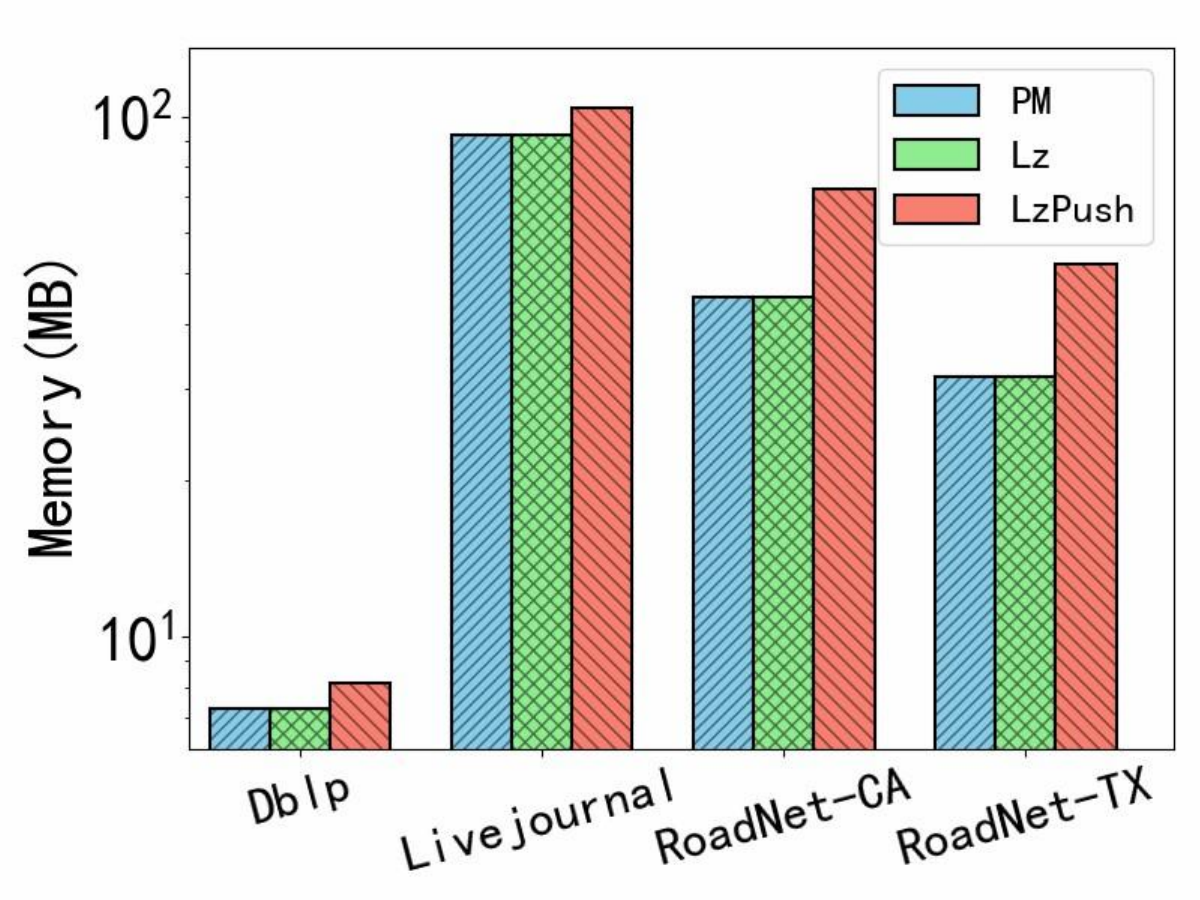}}
     \subfigure[high precision]{
		\includegraphics[scale=0.20]{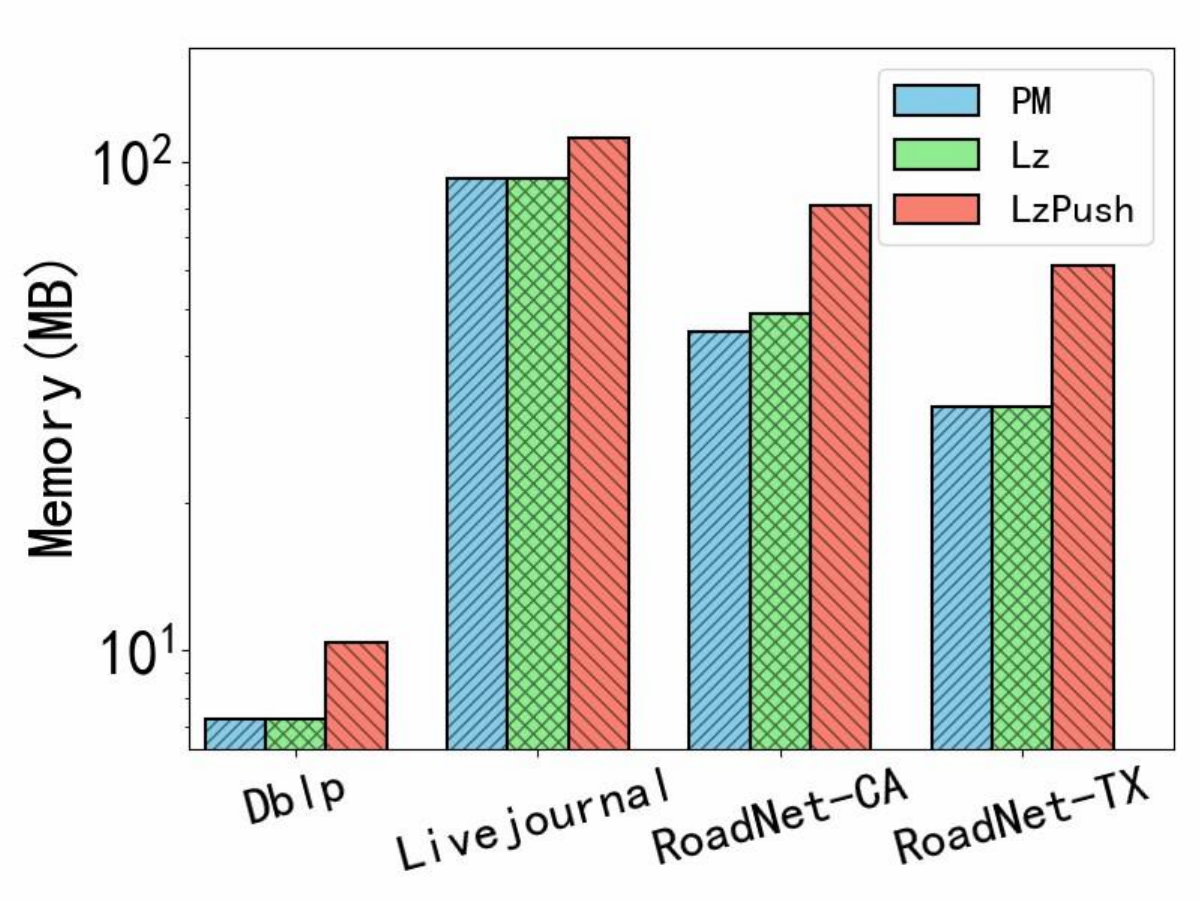}}
    \caption{Memory usage testing on different datasets}\label{fig:memory}\vspace{-0.2cm}
\end{figure}

\stitle{Exp-4: Memory usage testing on various datasets.}
In this experiment, we evaluate the memory usage of the proposed algorithms \lanczos and \lzpush with \powermethod under both low-precision case (Fig. \ref{fig:memory} (a)) and high-precision case (Fig. \ref{fig:memory} (b)). Specifically, for the low-precision case, we set $k=5,\epsilon=10^{-1}$ for social networks and $k=20,\epsilon=10^{-3}$ for road networks; for the high-precision case, we set $k=30, \epsilon=10^{-3}$ for social networks and $k=100,\epsilon=10^{-5}$ for road networks. From Fig. \ref{fig:memory}, we have the following observations: (1) Both \lanczos and \lzpush exhibit linear memory dependence on graph size $n$, comparable to \powermethod's $3n$ requirement (storing three vectors), with \lzpush showing marginally higher usage than \lanczos; The memory usage of \lzpush is slightly higher than \lanczos. (2) Memory consumption remains nearly identical across different precision settings, as our \lanczos implementation only maintains three vectors ($\mathbf{v}_{i-1},\mathbf{v}_i,\mathbf{v}_{i+1}$) per iteration rather than the full matrix $\mathbf{V}=[\mathbf{v}_0,...,\mathbf{v}_k]$. The modest additional memory requirement of \lzpush stems from storing the candidate set $S_i$ and non-zero entries of $\mathbf{v}_i$ during iterations, though the overall space complexity remains $O(n)$. Consequently, our algorithms maintain memory efficiency even for large-scale datasets, as their $O(n)$ storage requirement is dominated by the $O(m)$ space needed for the graph data itself.

\begin{figure*}[t!]
    \vspace{-0.3cm}
    \centering
 	\subfigure[\dblp]{
            \centering
		\includegraphics[scale=0.25]{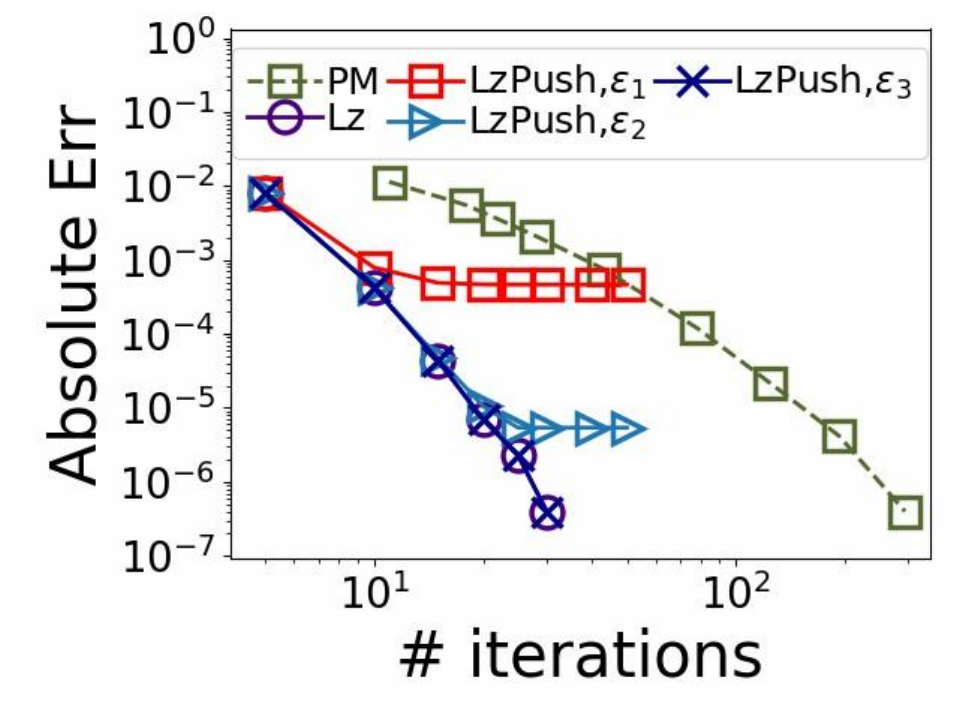}\hspace{-5mm} \label{1}
	}
	\quad
        \subfigure[\livejournal]{
        \centering
		\includegraphics[scale=0.25]{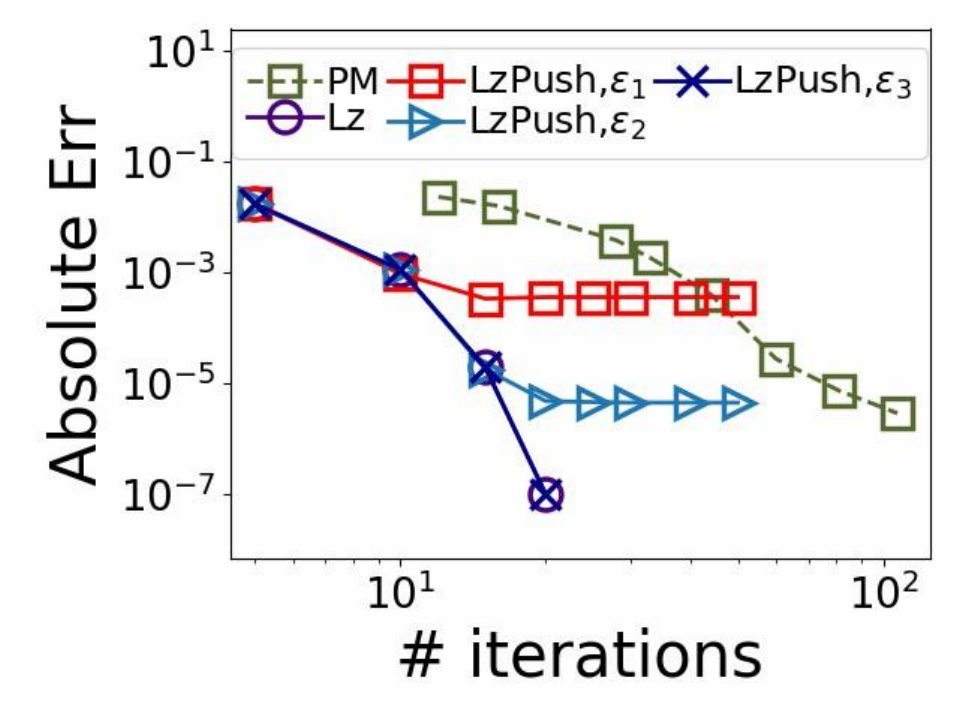}\hspace{-5mm}  \label{2}
	}
	\quad
	\subfigure[\powergrid]{
        \centering
		\includegraphics[scale=0.25]{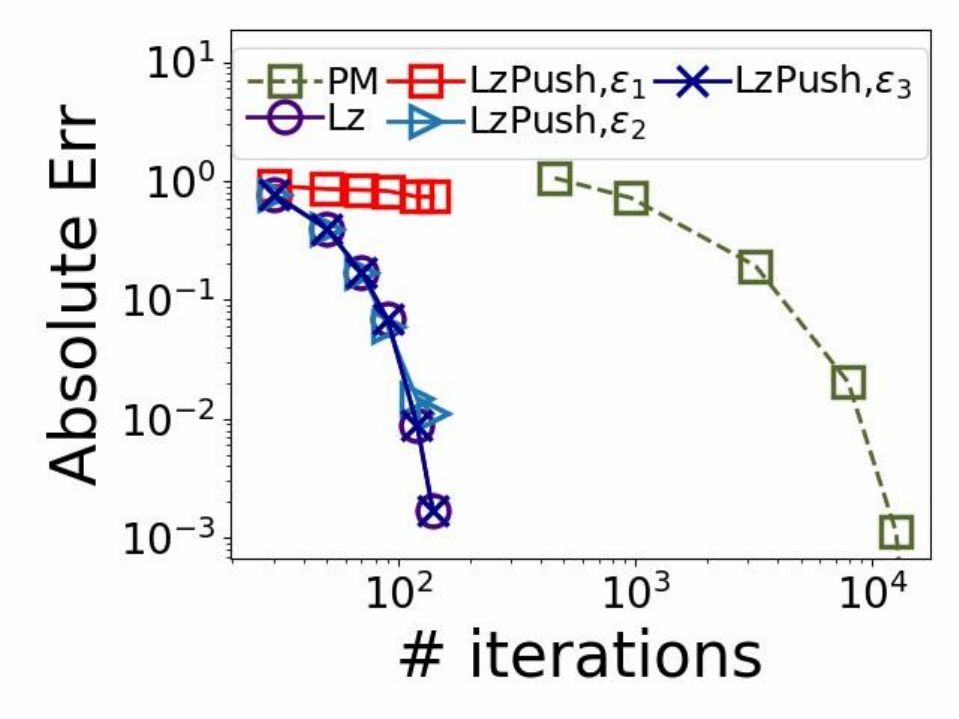}\hspace{-5mm}  \label{3}
	}
	\quad
	\subfigure[\roadca]{
        \centering
		\includegraphics[scale=0.25]{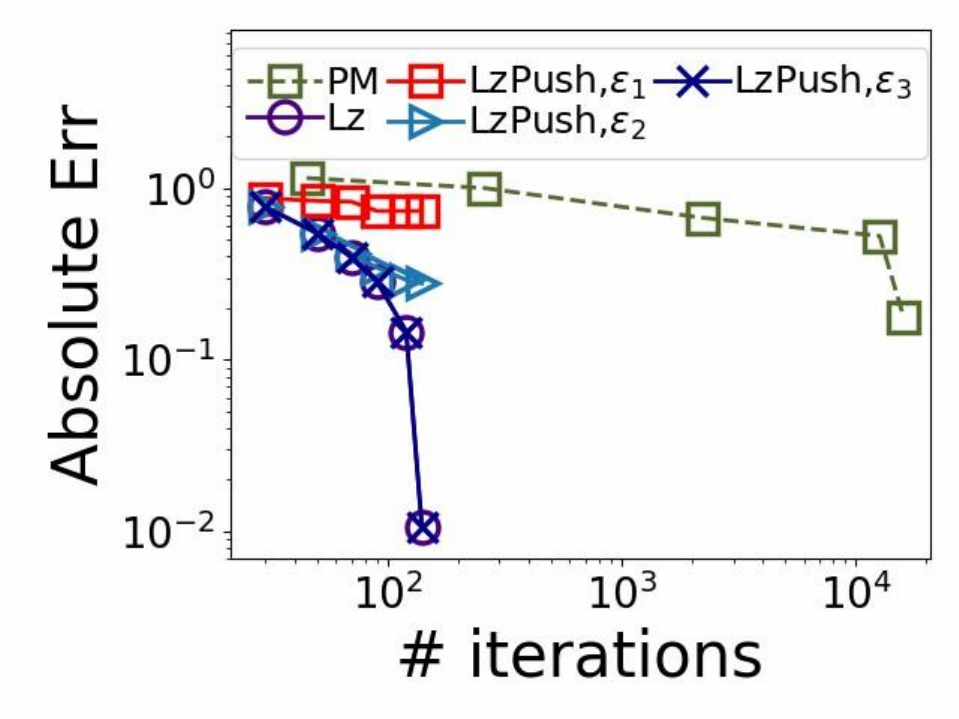}\label{4}
	}
	\caption{Stability testing of \lanczos and \lzpush with varying iteration number $k$ and $\epsilon$}\label{fig:itr}
\vspace{-0.2cm}
\end{figure*}

\stitle{Exp-5: Stability testing of \lanczos and \lzpush under high precision.}
In this experiment, we examines the error stability of our proposed \lanczos and \lzpush algorithms. Specifically, we vary the iteration number $k$ of \lanczos and \lzpush from $5$ to $150$, and set $\epsilon_1=10^{-3}$, $\epsilon_2=10^{-5}$, $\epsilon_3=10^{-7}$ for \lzpush to match different absolute errors (See Fig. \ref{fig:itr}). The results demonstrate that with sufficiently large $k$ and small $\epsilon$, both Lanczos algorithms achieve high-precision solutions while converging $\sqrt{\kappa}$ times faster than \powermethod. \footnote{Comparing with previous studies ~\cite{peng2021local,yang2023efficient,liao2023resistance,liao2024efficient,cui2025mixing}, the $10^{-7}$ absolute error for social networks and $10^{-3}$ absolute error for road networks is already very high precision for RD computation. This is beacuse \powermethod requires over $10^{3}$ iterations for social networks and $10^4$ iterations for road networks to reach this error guarantee for the ground-truth value computation.} 
As a result, \lanczos and \lzpush are both stable under high precision computation. Our findings align with previous studies ~\cite{Meurant2006lanczossurvey,musco2018stability}: Lanczos methods remain numerically stable for matrix function computation. In our case, the stability for RD computation stems from this phenomenon, since $\mathbf{L}^\dagger(\mathbf{e}_s-\mathbf{e}_t)$ is matrix function. As a result, our algorithms do not suffer unstable problem for high precision computation. Moreover, there is an additional observation that when $k$ increases, the absolute error produced by \lanczos drops exponentially, but the absolute error produced by \lzpush does not change drastically under $\epsilon=10^{-3}$ and $10^{-5}$. These results suggest that the efficiency of \lanczos mainly depends on $k$, but the error produced by \lzpush mainly depends on $\epsilon$ for larger $k$, which is consistent with our theoretical analysis.


\begin{figure}
    \centering
    \subfigure[\dblp, Absolute Err]{
		\includegraphics[scale=0.25]{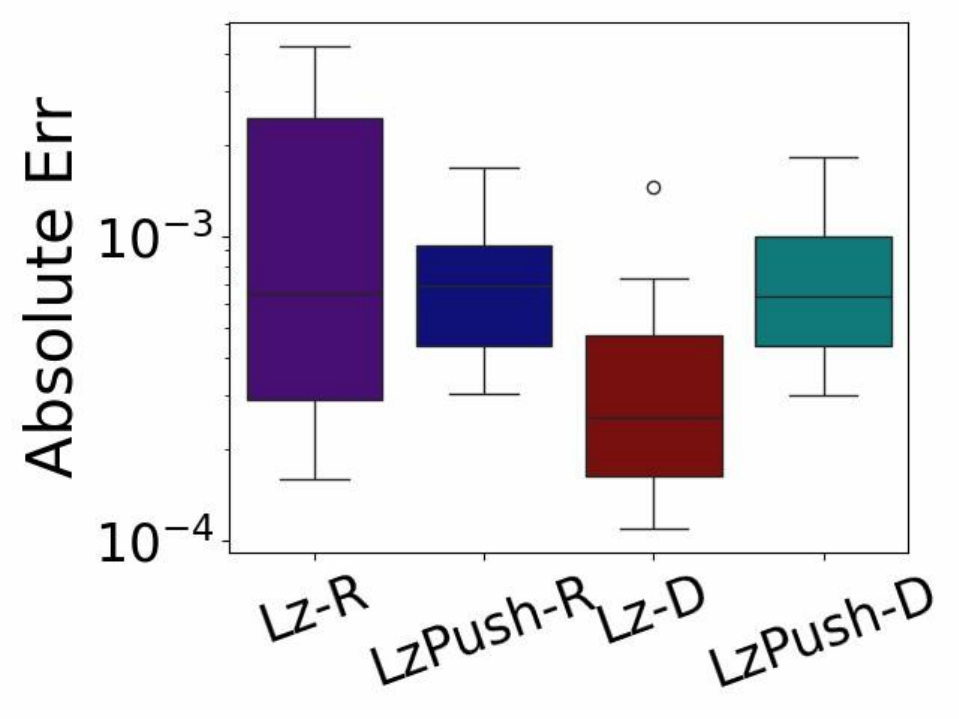}}
     \subfigure[\dblp, runtime]{
		\includegraphics[scale=0.25]{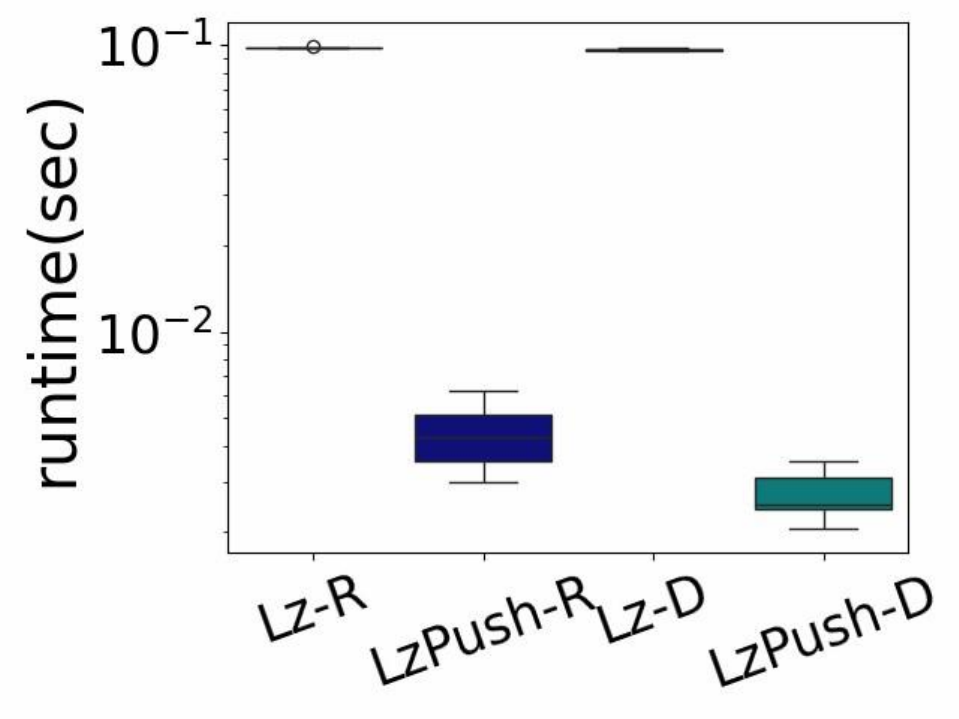}}

    \caption{The performance of \lanczos and \lzpush with different 
query node select strategies. \kw{Lz-R} and \kw{LzPush-R} denote \lanczos and \lzpush selecting source node $s$ and sink node $t$ uniformly, respectively. \kw{Lz-D} and \kw{LzPush-D} represent \lanczos and \lzpush selecting source node $s$ and sink node $t$ with the highest degree, respectively.}\label{fig:query-distribution}\vspace{-0.2cm}
\end{figure}

\begin{figure}
    \centering
    \subfigure[\er, Absolute Err]{
		\includegraphics[scale=0.2]{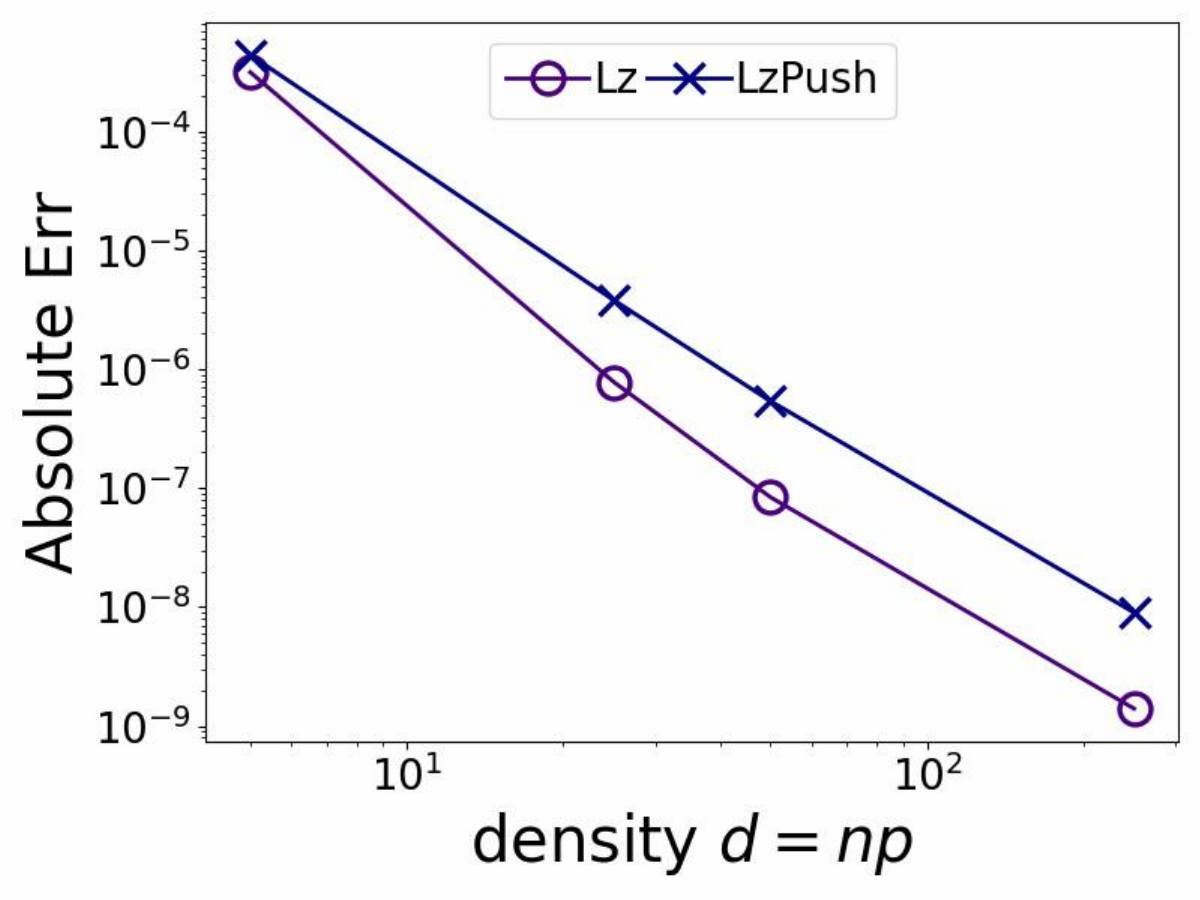}}
     \subfigure[\er, runtime]{
		\includegraphics[scale=0.2]{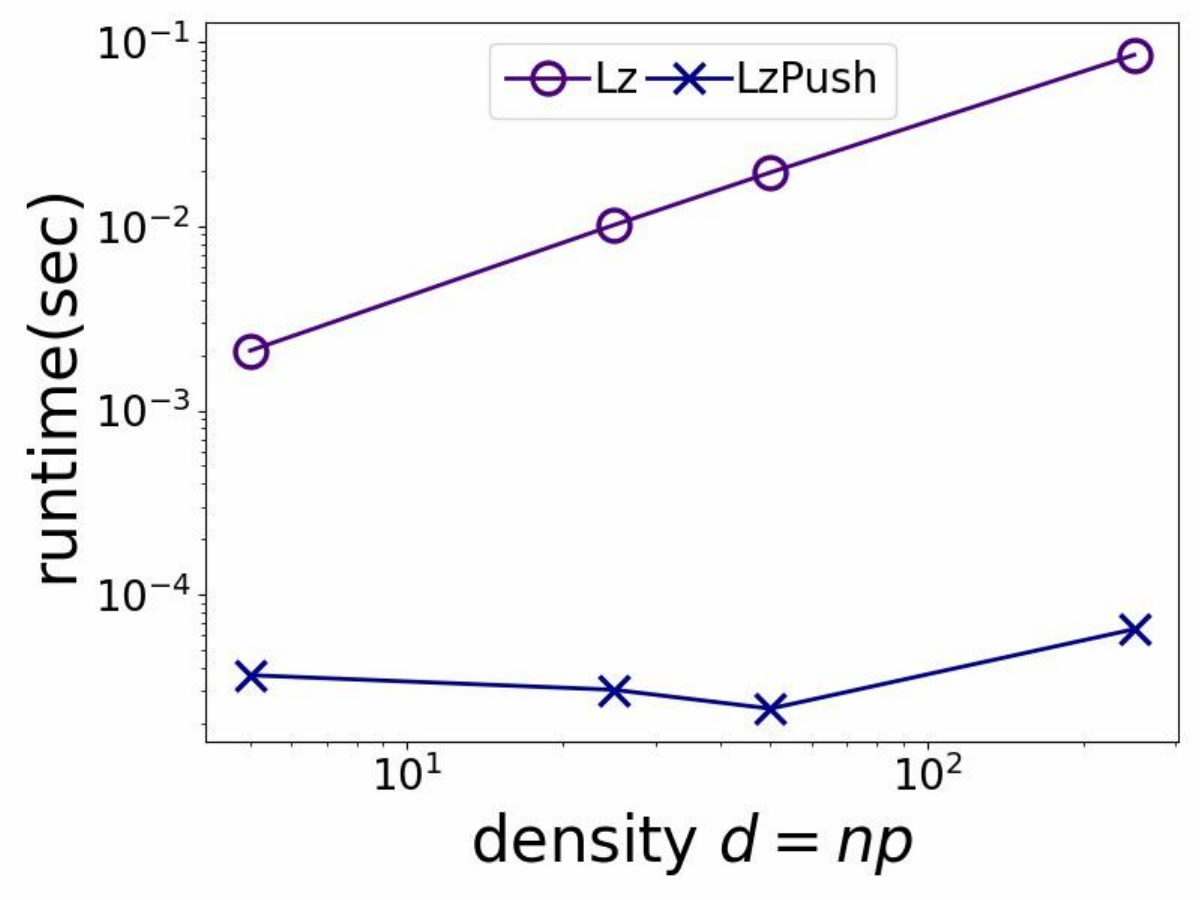}}

    \caption{The performance of \lanczos and \lzpush on $\mathbb{G}(n,p)$ with varying density, where the density is defined as $d=np$.}\label{fig:density}\vspace{-0.2cm}
\end{figure}

\stitle{Exp-6: Stability testing of \lanczos and \lzpush under various density and node selection strategies.}
This experiment evaluates the stability of \lanczos and \lzpush in terms of: (i) various node selection strategies; (ii) density of the datasets. Specifically, for the node selection strategies, we explores two strategies for selecting source node $s$ and sink node $t$: (i) choosing the 10 nodes with the highest degrees, and (ii) randomly selecting 10 nodes from the graph. 
For \lanczos and \lzpush, we set the parameter $\epsilon=10^{-3}$ and $k=15$. Fig. ~\ref{fig:query-distribution} presents the results using box plots to depict the distribution of query qualities and query times on \dblp. As can be seen, we have the following observations:  (i) The error produced by \lanczos and \lzpush is insensitive to the node selection strategies. (ii) For different source/sink node selection strategies, the query time and Absolute Err do not differ significantly. The overall observation is that the performance of both \lanczos and \lzpush do not heavily depend on the strategy of the $s,t$ node selection. For the density testing, we consider \er random graph $\mathbb{G}(n,p)$ with $n=1000$ nodes. We define the density $d=np$ and we vary $d=[5,25,50,250]$ to generate datasets with different density. We set $k=5$ and $\epsilon=0.1$ for \lanczos and \lzpush. Fig. \ref{fig:density} depicts the results. We have the following observations: the runtime of \lanczos grows linearly with the density of the graph, while \lzpush is insensitive to the density. Moreover, the absolute value drops when the density $d$ becomes larger. This is mainly because the accurate RD value is small for dense graphs.

\comment{
\stitle{Exp-5: The performance of \lzpush and \lanczos with varying the iteration number $k$.} In this experiment, we fix $\epsilon=0.005$ for \lzpush, vary the iteration number $k$ from $5$ to $30$ for \dblp and $20$ to $120$ for \roadca for both \lanczos and \lzpush algorithms. Since other datasets and parameters have similar trends, we ignore them for brevity. We compare the runtime and Absolute Err between \lanczos and \lzpush. Fig. \ref{fig:vary_k} shows this result. As can be seen, when $k$ increases, the absolute error produced by \lanczos drops exponentially, but the absolute error produced by \lzpush does not change drastically. On the other hand, the runtime of \lanczos and \lzpush both increase for larger $k$, and the runtime of \lzpush grows slightly faster for than \lanczos. These results suggest that the efficiency of \lanczos mainly depends on $k$, but the error produced by \lzpush does not mainly depend on $k$ for the fixed $\epsilon$, which is consistent with our theoretical analysis. 

\begin{figure}
    \centering
    \subfigure[\dblp,Absolute Err]{
		\includegraphics[scale=0.2]{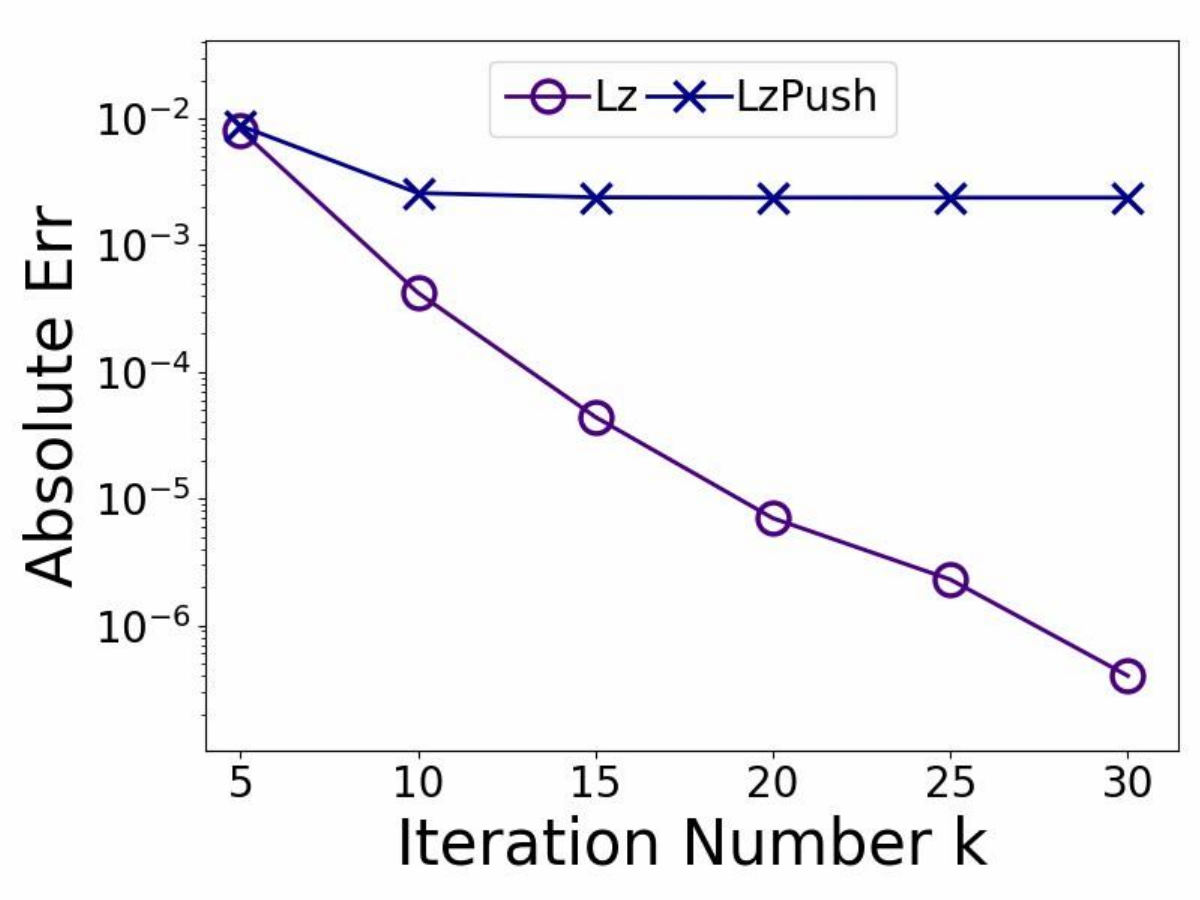}}
     \subfigure[\dblp,runtime]{
		\includegraphics[scale=0.2]{parameter_testing/dblp_varyk_time1.pdf}}

      \subfigure[\roadca,Absolute Err]{
		\includegraphics[scale=0.2]{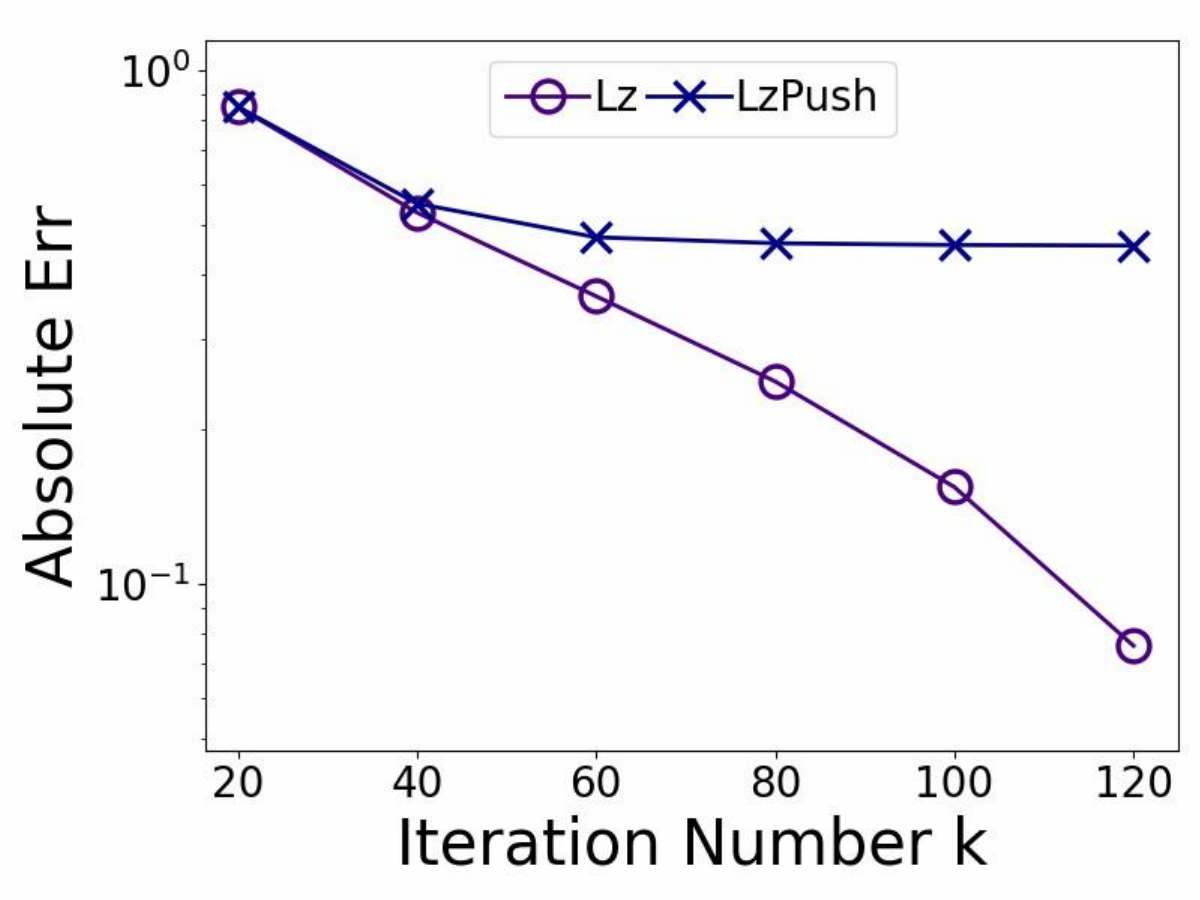}}
     \subfigure[\roadca,runtime]{
		\includegraphics[scale=0.2]{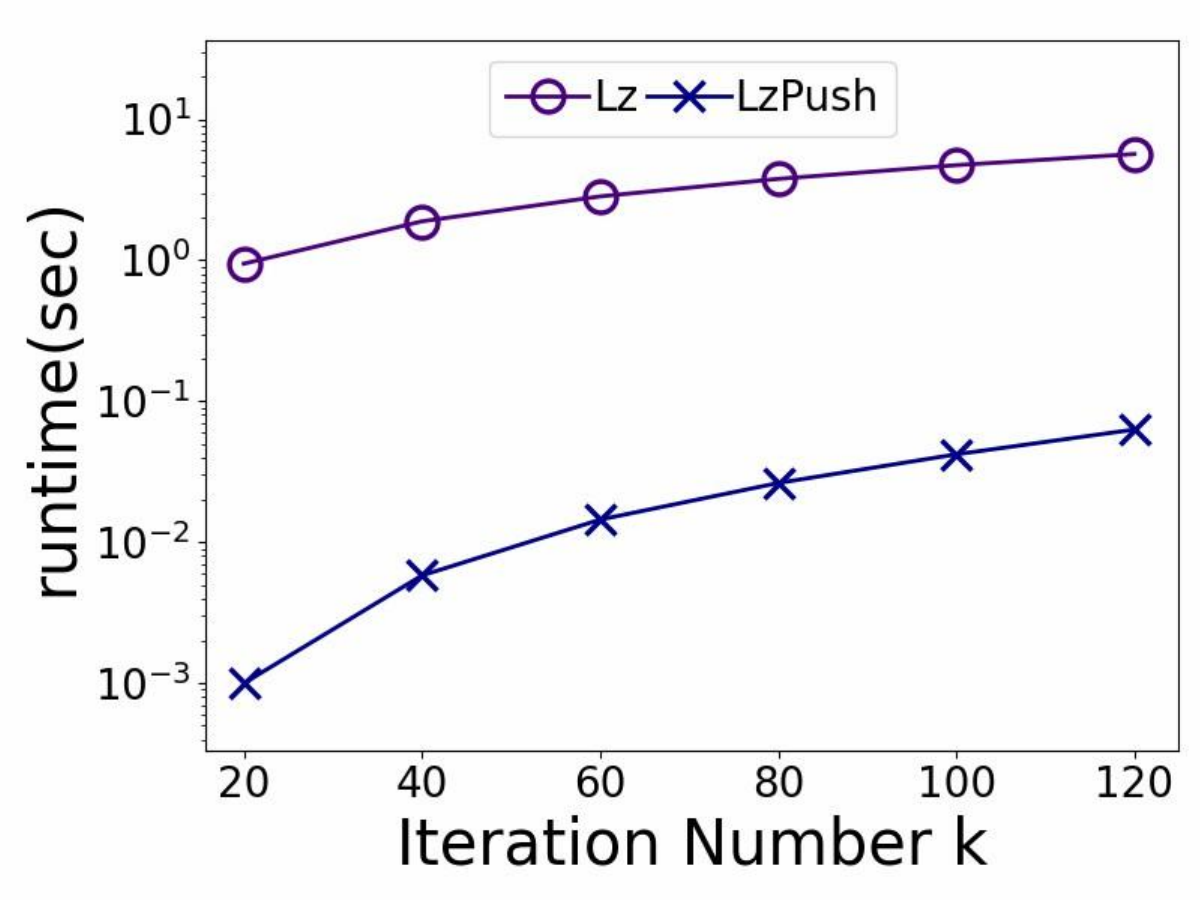}}
    \caption{The performance of \lzpush and \lanczos with varying the iteration number $k$.}\label{fig:vary_k}\vspace{-0.2cm}
\end{figure}
}

\stitle{Exp-7: The performance  of \lzpush with varying $k$ and $\epsilon$.} In this experiment, we explore the performance of \lzpush under different setting of $k$ and $\epsilon$ on two typical datasets \dblp and \roadca. Other datasets are ignored due to they have similar trends. Fig. \ref{fig:heat_plot} shows the heatmap to illustrate this result. As can be seen, we have the following observations: (1) we observe each row from Fig. \ref{fig:heat_plot} (a) and \ref{fig:heat_plot} (c). The conclusion is that on both two datasets, for each fixed $\epsilon$, the absolute error produced by \lzpush does not drastically drop when $k$ increases. However, for a specific pre-selected $k$ (e.g., $k=20$ for \dblp and $k=120$ for \roadca), the different setting of $\epsilon$ actually affect the performance of \lzpush. Specifically, the absolute error drops when $\epsilon$ becomes smaller. (2) the runtime of \lzpush (i.e., Fig. \ref{fig:heat_plot} (b) and \ref{fig:heat_plot} (d)) both influenced by $\epsilon$ and $k$. Specifically, larger $k$ and smaller $\epsilon$ resulting in longer runtime of \lzpush. Thus, these results align with our theoretical analysis.

\begin{figure}\vspace{-0.3cm}
    \centering
    \subfigure[\dblp,Absolute Err]{
		\includegraphics[scale=0.2]{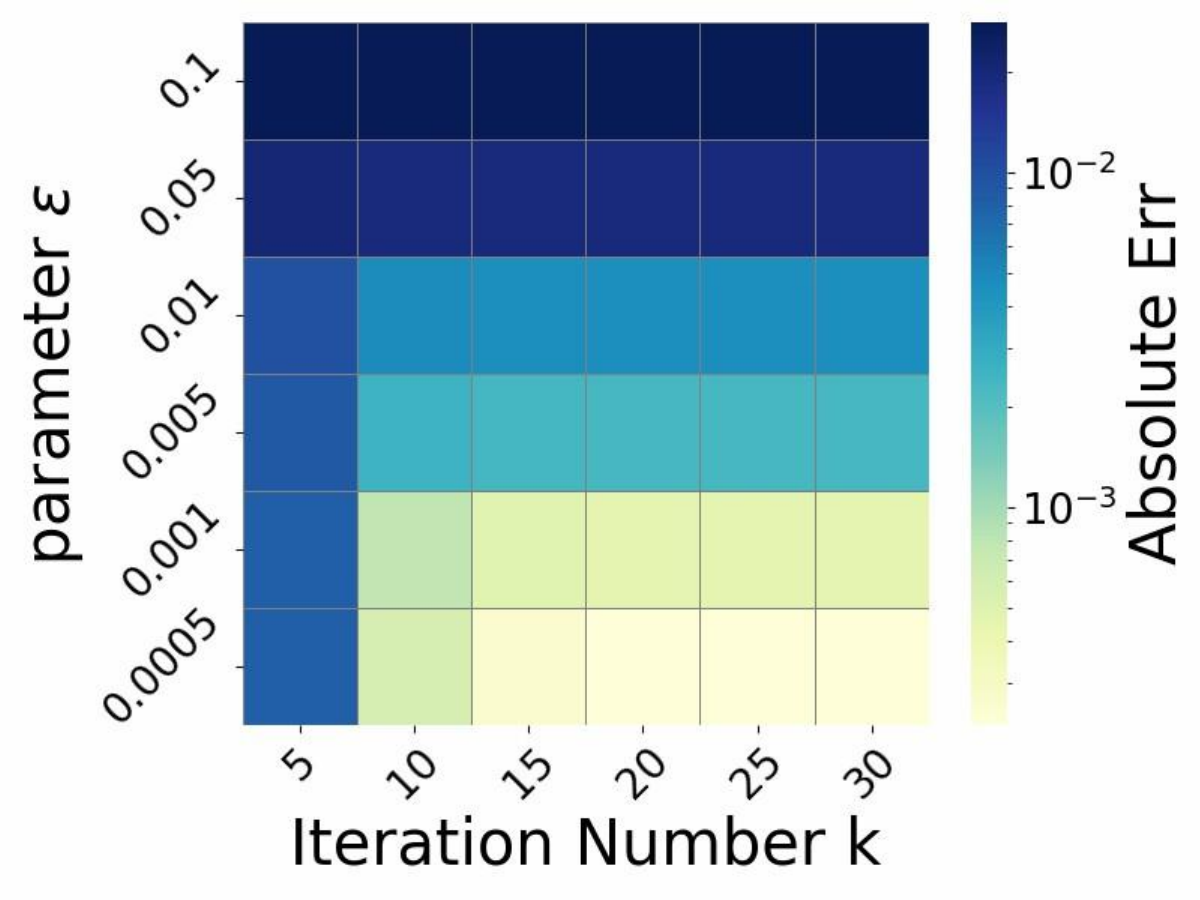}}
     \subfigure[\dblp,runtime]{
		\includegraphics[scale=0.2]{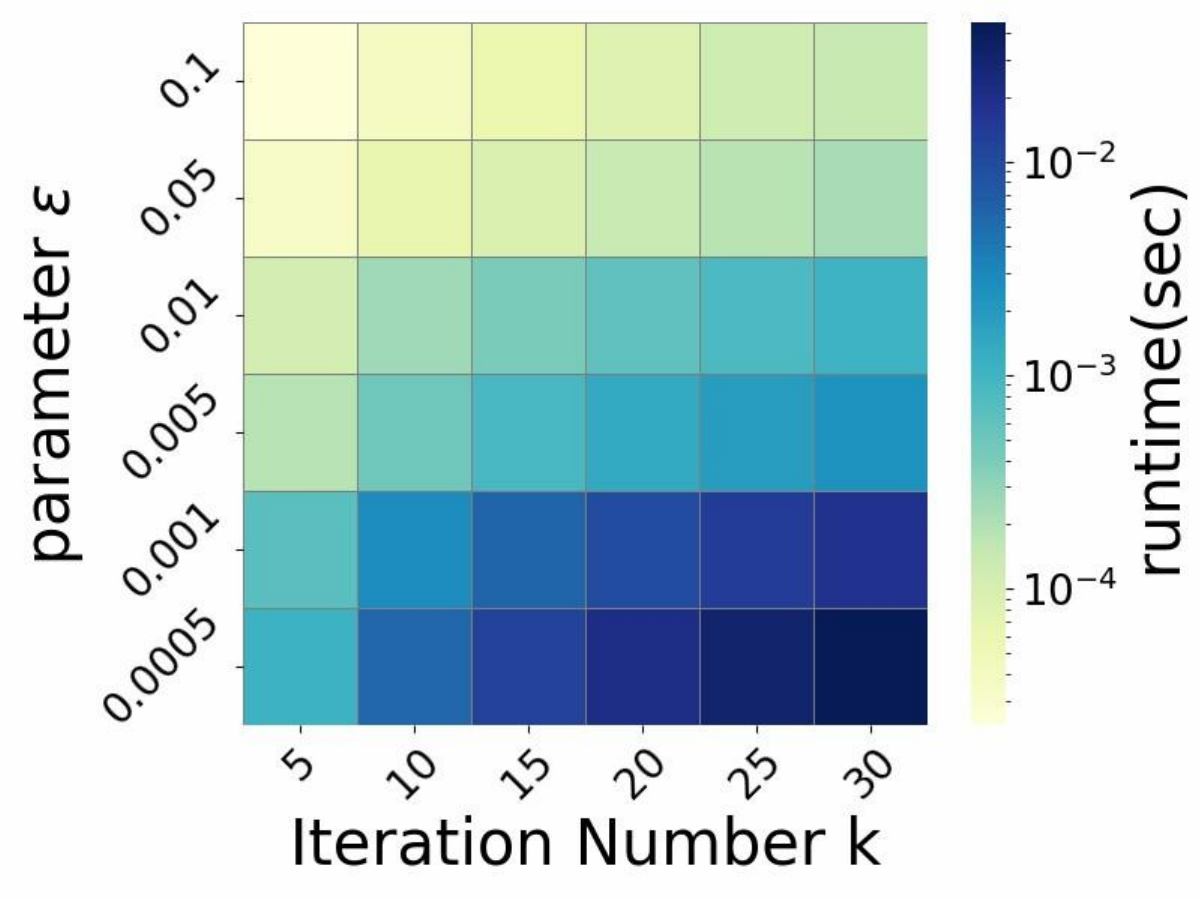}}

      \subfigure[\roadca,Absolute Err]{
		\includegraphics[scale=0.2]{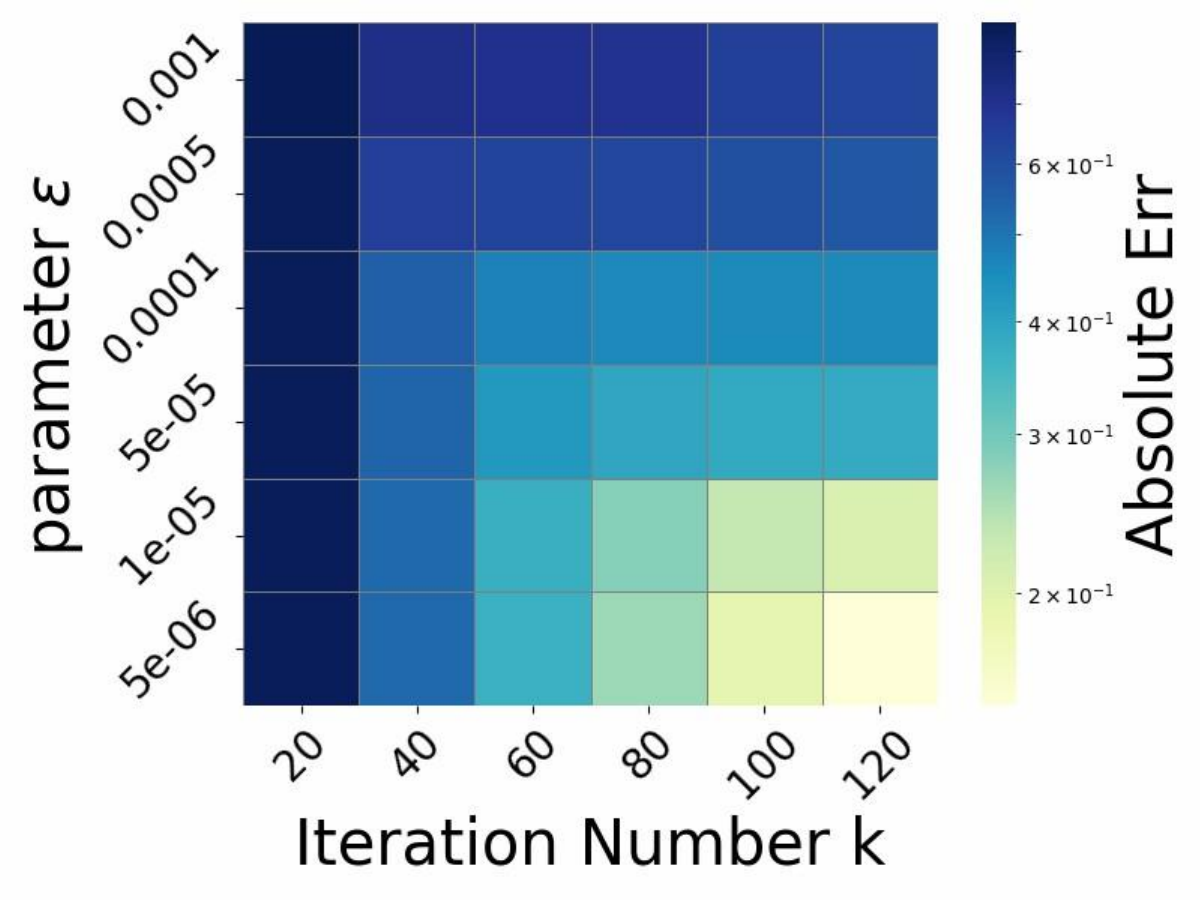}}
     \subfigure[\roadca,runtime]{
		\includegraphics[scale=0.2]{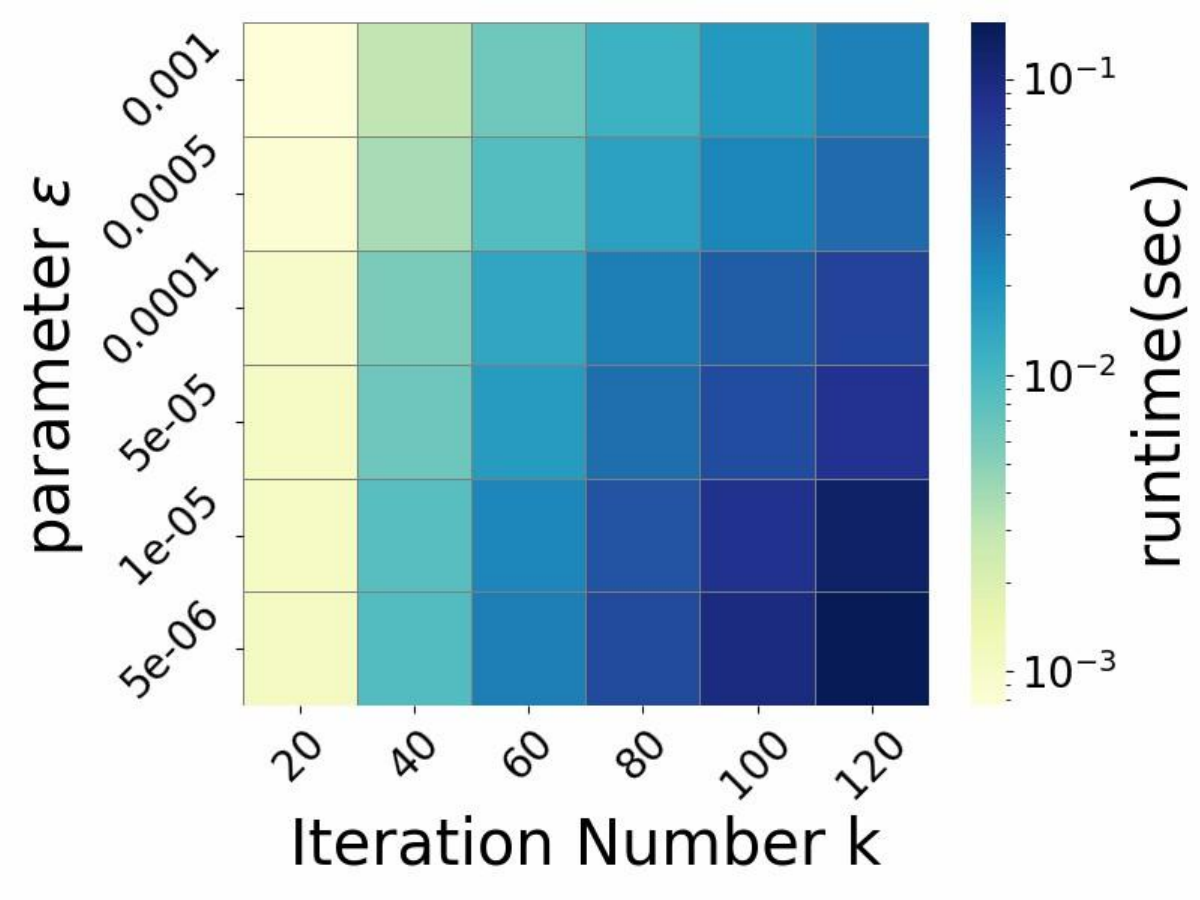}}
    \caption{The performance  of \lzpush with varying $\epsilon$ and $k$}\label{fig:heat_plot}\vspace{-0.2cm}
\end{figure}

\stitle{Exp-8: Testing the reasonability of Assumption \ref{assump:lzpush_error}.} In this experiment, we test the we test whether Assumption \ref{assump:lzpush_error} holds in numerical experiments, which requires that $\lambda(\hat{\mathbf{T}})\subset [\lambda_n(\mathcal{A}),\lambda_2(\mathcal{A})]$. For clearer visualization, we examine the eigenvalues of $\mathbf{I}-\mathbf{T}$ and find that $\lambda(\mathbf{I}-\hat{\mathbf{T}})\subset [\lambda_2(\mathcal{L}),\lambda_n(\mathcal{L})]$. For social networks we fix $k=15$, for road networks we fix $k=100$ while varying $\epsilon$ across different values. Fig. \ref{fig:assump_eigen} show that shows that the eigenvalues of $\mathbf{I}-\hat{\mathbf{T}}$ always fall within the range of $\lambda(\mathcal{L})$ with the extremal eigenvalues approaching $\lambda_2(\mathcal{L})$ and $\lambda_n(\mathcal{L})$ as $\epsilon$ decreases. These results support the reasonability of Assumption \ref{assump:lzpush_error} and suggest that the Lanczos Push algorithm may also be useful for approximating extremal eigenvalues.

\begin{figure}\vspace{-0.3cm}
    \centering
    \subfigure[Social Networks]{
		\includegraphics[scale=0.2]{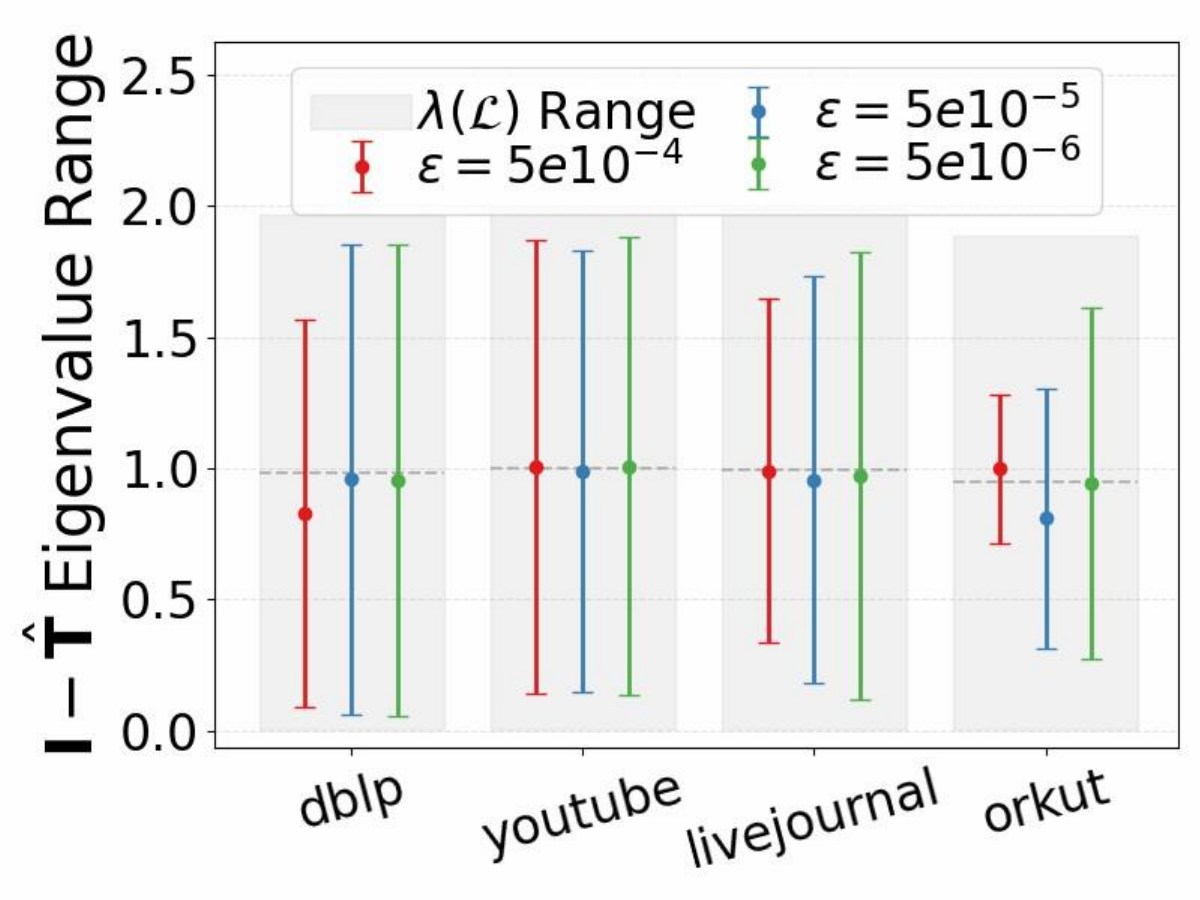}}
     \subfigure[Road Networks]{
		\includegraphics[scale=0.2]{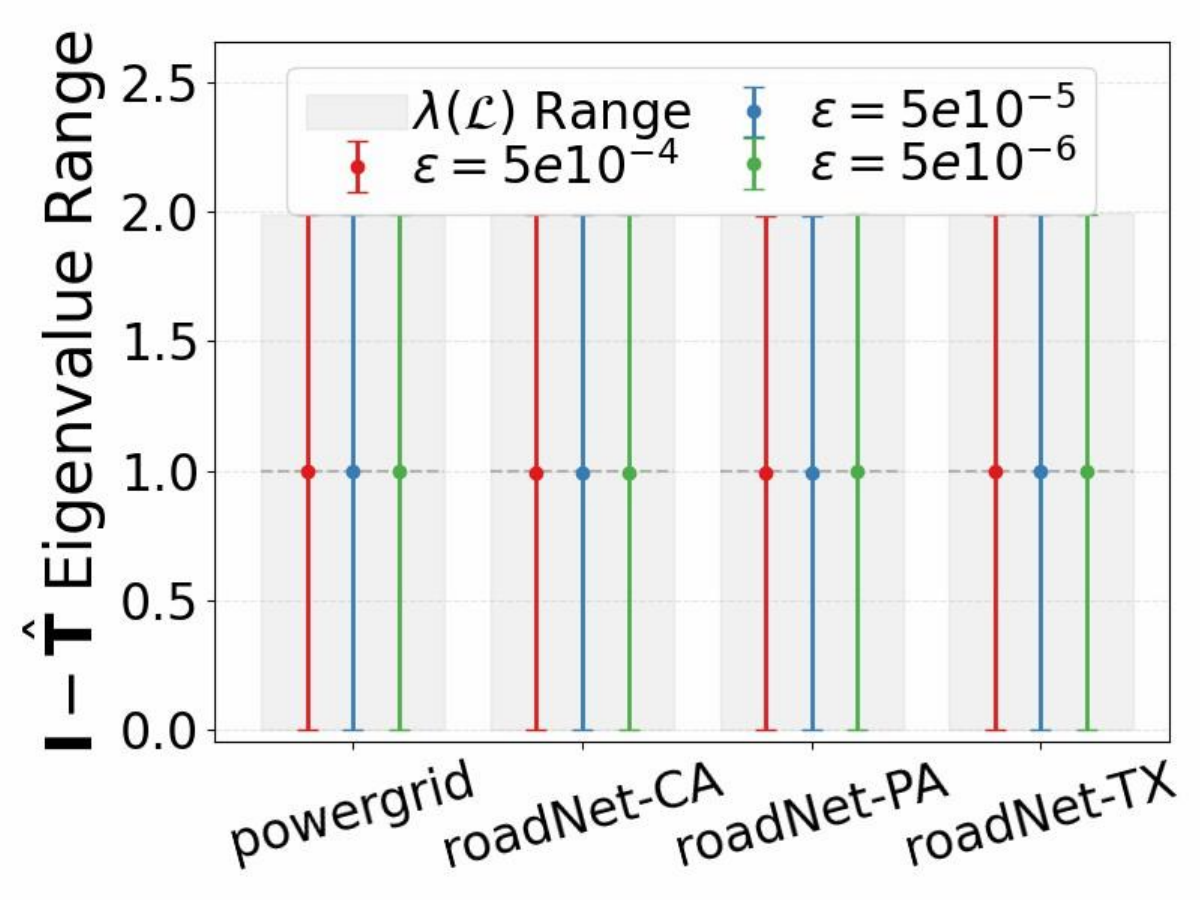}}
    \caption{Testing the eigenvalue range of $\mathbf{I}-\hat{\mathbf{T}}$ for different datasets with varing $\epsilon$}\label{fig:assump_eigen}\vspace{-0.2cm}
\end{figure}

\stitle{Exp-9: Testing the value of $C_1$ and $C_2$.} Recall that in our theoretical analysis, the runtime bound for \lzpush is $\tilde{O}(\kappa^{2.75}C_1C_2/\epsilon)$ with two additional numbers $C_1=\max_{u=s,t;i\leq k}{\Vert T_i(\mathbf{P})\mathbf{e}_u\Vert_1}$ and $C_2=\max_{i\leq k}{(\Vert \hat{\mathbf{v}}_i\Vert_1+\Vert \mathcal{A}\hat{\mathbf{v}}_i^+\Vert_1+\Vert \mathcal{A}\hat{\mathbf{v}}_i^-\Vert_1)}$. In this experiment we test the value of these two numbers with varying $k$ and $\epsilon$. Specifically, for $C_1$ we vary $k$ from $5$ to $30$ for social networks and $20$ to $120$ for road networks, with results shown in Fig. \ref{fig:assump_cheby}. We have the following observations: (i) $\Vert T_i(\mathbf{P})\mathbf{e}_u\Vert_1$ increases with both graph size and $k$, but experimentally scales better than its theoretical worst case $C_1\leq \sqrt{n}$ in Corollary \ref{coro:lzpush_time_error}. (ii) On graphs with poor expansion (road networks), the growth rate of $C_1$ is slower than the datasets with good expansion (social networks). For $C_2$, we fix $k=15$ for social networks, $k=100$ for road networks with varying $\epsilon$, with results shown in Fig. \ref{fig:assump_vi}. The observation is similar to $C_1$: $\Vert \hat{\mathbf{v}}_i\Vert_1,\Vert \mathcal{A}\hat{\mathbf{v}}_i\Vert_1$ increase with larger graph size and smaller $\epsilon$, and the growth rate on road networks is slower. These observations further demonstrate that \lzpush performs better on road networks comparing with previous algorithms.

\begin{figure}\vspace{-0.3cm}
    \centering
    \subfigure[Social Networks]{
		\includegraphics[scale=0.25]{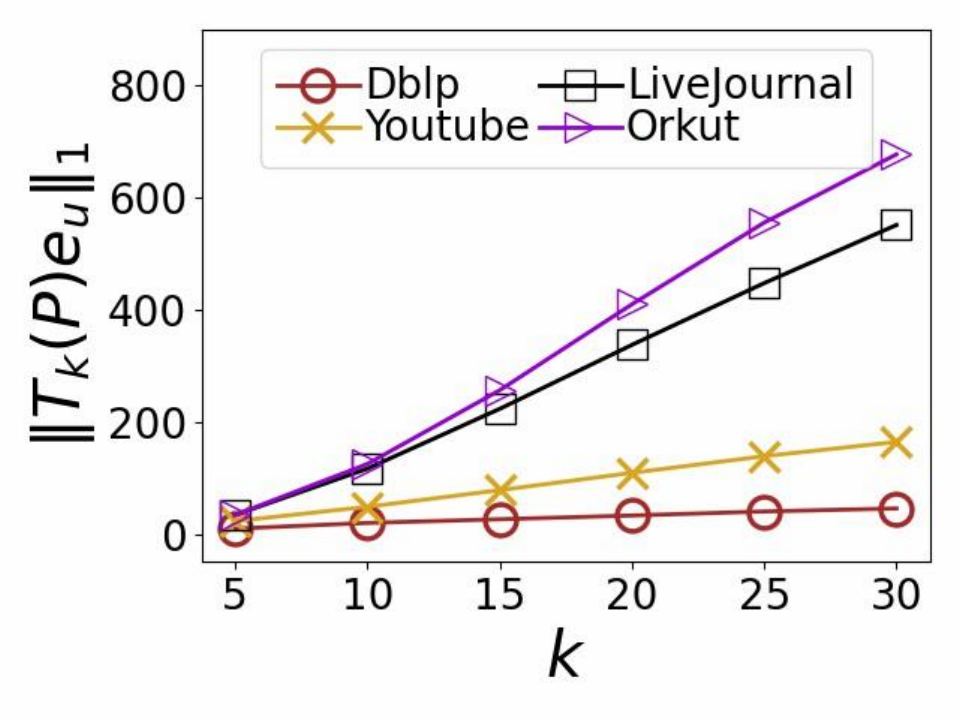}}
     \subfigure[Road Networks]{
		\includegraphics[scale=0.25]{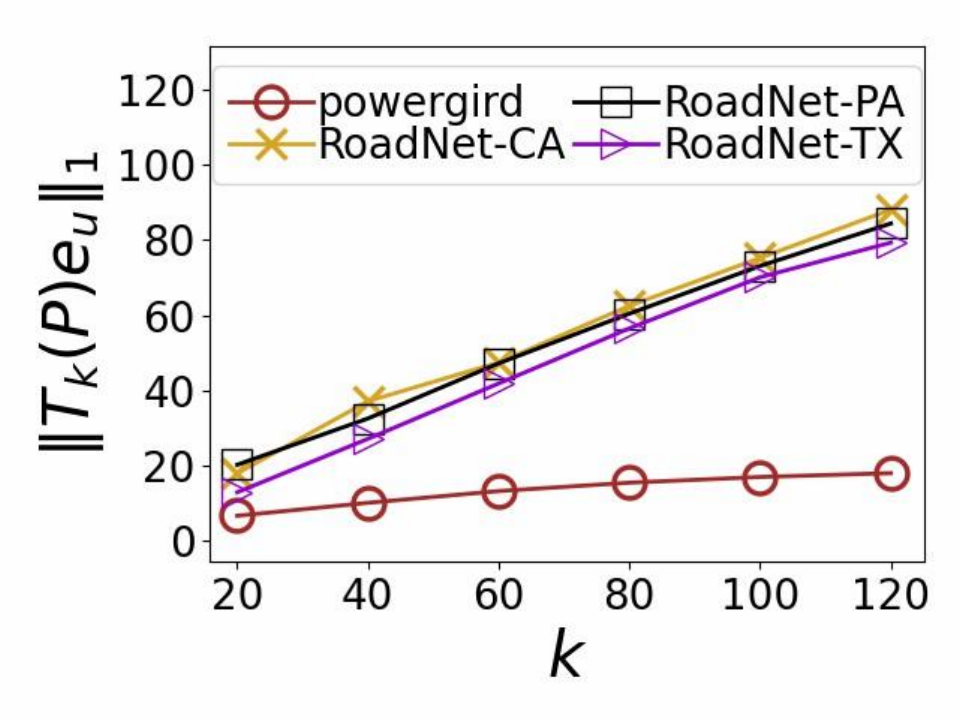}}
    \caption{Testing the value of $\max_{u=s,t;i\leq k}\Vert T_i(\mathbf{P})\mathbf{e}_u\Vert_1$ with varying $k$}\label{fig:assump_cheby}\vspace{-0.2cm}
\end{figure}

\begin{figure}\vspace{-0.3cm}
    \centering
    \subfigure[Social Networks, $\Vert \hat{\mathbf{v}}_i\Vert_1$]{
		\includegraphics[scale=0.25]{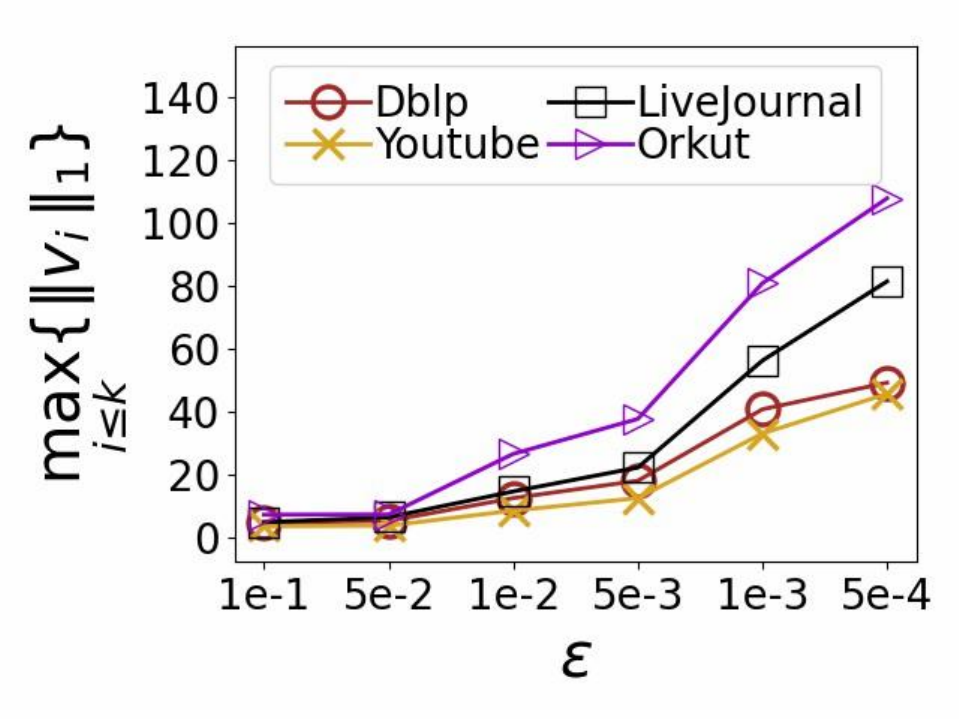}}
     \subfigure[Road Networks, $\Vert \hat{\mathbf{v}}_i\Vert_1$]{
		\includegraphics[scale=0.25]{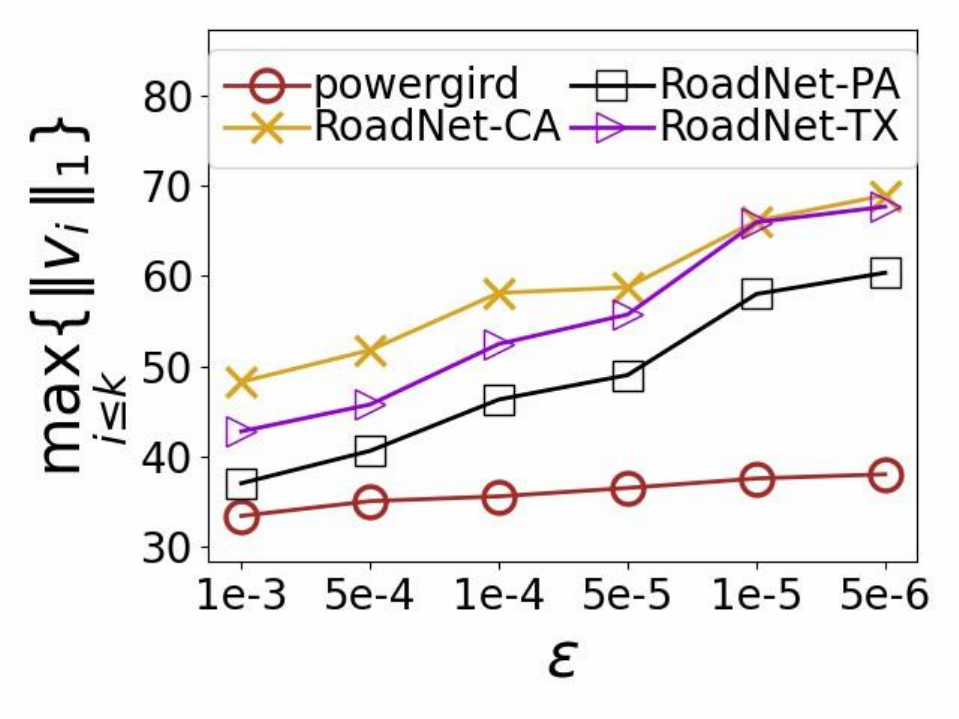}}

    \subfigure[Social Networks, $\Vert \mathcal{A}\hat{\mathbf{v}}_i\Vert_1$]{
		\includegraphics[scale=0.25]{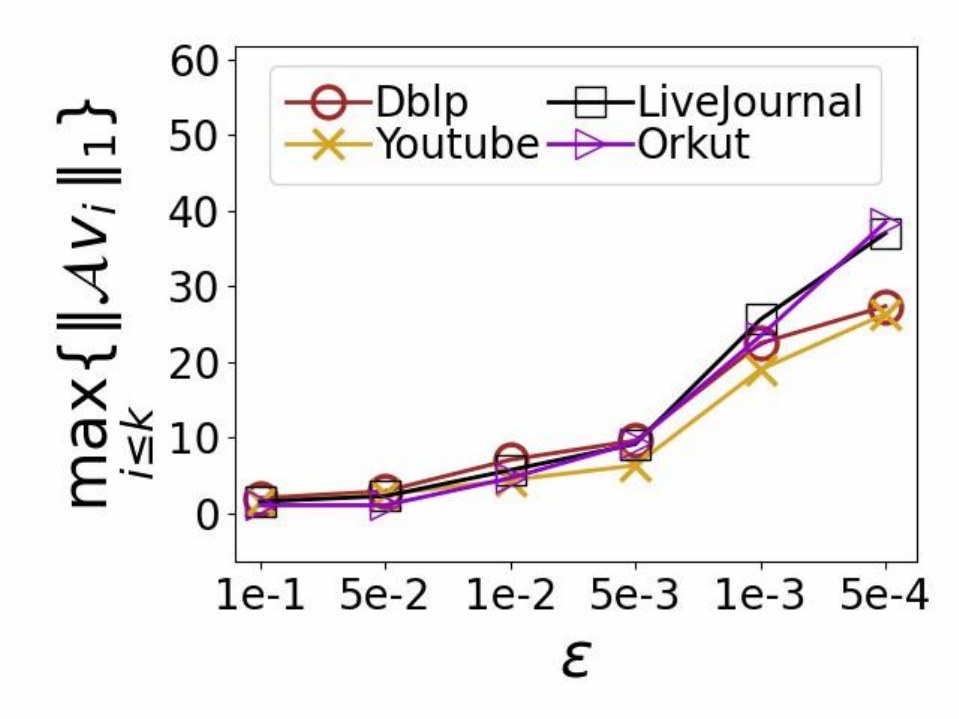}}
     \subfigure[Road Networks, $\Vert \mathcal{A}\hat{\mathbf{v}}_i\Vert_1$]{
		\includegraphics[scale=0.25]{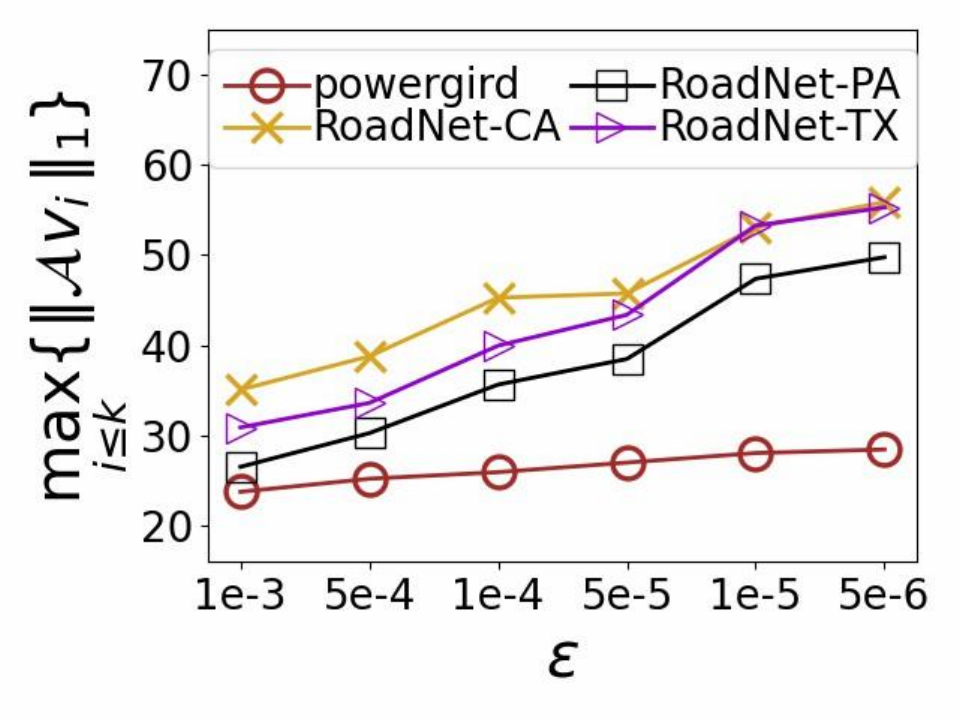}}
    \caption{Testing the value of $\max_{i\leq k} \Vert \hat{\mathbf{v}}_i\Vert_1$ and $\max_{i\leq k}\Vert \mathcal{A}\hat{\mathbf{v}}_i\Vert_1$ with varing $\epsilon$}\label{fig:assump_vi}\vspace{-0.2cm}
\end{figure}

\subsection{Case Study: Robust Routing on Road Networks}

One interesting application of the proposed methods is the robust routing on road networks. One recent research ~\cite{sinop2023robust} showing that generating a set of alternates routes via electric flows have the advantage in terms of: (i) \textit{Stretch}. Each alternate path does not too much worse than the shortest path. (ii) \textit{Diversity}. The alternate paths sufficiently differs from each other. (iii) \textit{Robustness}. Under random edge deletions, there will be a good route among the alternates with high probability. Specifically, their algorithms left the electric flow computation as a crucial subroutine, which is closely related to our RD computation problem.

\begin{definition}\label{def:electric_flow}
    Given two nodes $s,t$, the potential vector $\bm{\phi}\in \mathbb{R}^n$ is defined by $\bm{\phi}=\mathbf{L}^\dagger (\mathbf{e}_s-\mathbf{e}_t)$. The electric flow vector $\mathbf{f}\in \mathbb{R}^m$ is defined by $\mathbf{f}(e)=\bm{\phi}(u)-\bm{\phi}(v)$ for each edge $e=(u,v)\in \mathcal{E}$.
\end{definition}

\begin{algorithm}[t!]
\small
	\SetAlgoLined
	\KwIn{$\mathcal{G},s,t, k,l$}
$\hat{\bm{\phi}}\leftarrow $ Lanczos Iteration ($\mathcal{G},s,t,k$) with final output modified by $\hat{\bm{\phi}}=\sqrt{\frac{1}{d_s}+\frac{1}{d_t}}\mathbf{D}^{-1/2}\mathbf{V}(\mathbf{I}-\mathbf{T})^{-1}\mathbf{e}_1$\;
Compute $\hat{\mathbf{f}}$ by $\hat{\mathbf{f}}(e)=\hat{\bm{\phi}}(u)-\hat{\bm{\phi}}(v)$ for any $e=(u,v)\in \mathcal{E}$\;
\For{$i\in [2l]$}
{
    Find a path $\mathcal{P}$ that maximizes the minimum flow from $\hat{\mathbf{f}}$ along its edges, and remove it from $\hat{\mathbf{f}}$\;
    If no $\mathcal{P}$ exists, break\;
}
	\KwOut{$l$ paths with the smallest cost}	
    \caption{Robust routing using Lanczos Iteration.}\label{algo:routing}
\end{algorithm}

We show that the proposed Lanczos Algorithm can be directly generalized to compute $\bm{\phi}$ (and thus computing $\mathbf{f}$) by only a single line adjustment, see Algorithm \ref{algo:routing}. After this modification, we prove that the Lanczos Algorithm computes an accurate approximation of $\mathbf{f}$ \footnote{We just provide the error analysis by Lanczos Algorithm for simplicity, though the proposed Lanczos Push algorithm can also be modified to compute the electric flow.}.

\begin{theorem}\label{thm:electric_flow}
    Algorithm \ref{algo:routing} computes $\hat{\mathbf{f}}$ such that $\Vert \mathbf{f}-\hat{\mathbf{f}}\Vert_1\leq \epsilon$ when setting the iteration number $k=O(\sqrt{\kappa}\log \frac{m}{\epsilon})$.
\end{theorem}

\comment{
\begin{proof}
    \textcolor{blue}{We define $\mathbf{\Pi}=\mathbf{I}-\frac{1}{n}\mathbf{11}^T$ as projection matrix on the image of $\mathbf{L}$. By the definition of the potential vector $\bm{\phi}$, we have the following formula using same strategy as Eq. \ref{equ:approx_formula}:}
   \begin{equation}
\begin{aligned}
    \bm{\phi}&=\mathbf{L}^{\dagger}(\mathbf{e}_s-\mathbf{e}_t)\\
    &=\mathbf{\Pi}\mathbf{D}^{-1/2}(\mathbf{I}-\mathcal{A})^{\dagger}\mathbf{D}^{-1/2}(\mathbf{e}_s-\mathbf{e}_t)\\
    &=\mathbf{\Pi D}^{-1/2}(\mathbf{I}-\mathcal{A})^{\dagger} (\frac{\mathbf{e}_s}{\sqrt{d_s}}-\frac{\mathbf{e}_t}{\sqrt{d_t}})\\
    &\approx \sqrt{\frac{1}{d_s}+\frac{1}{d_t}}\mathbf{\Pi D}^{-1/2}\mathbf{V}(\mathbf{I}-\mathbf{T})^{-1}\mathbf{V}\mathbf{v}_1\\
    &=\sqrt{\frac{1}{d_s}+\frac{1}{d_t}}\mathbf{\Pi D}^{-1/2}\mathbf{V}(\mathbf{I}-\mathbf{T})^{-1}\mathbf{e}_1
\end{aligned}
\end{equation}
\textcolor{blue}{By definition $\mathbf{f}(e)=\bm{\phi}(u)-\bm{\phi}(v)=(\mathbf{e}_u-\mathbf{e}_v)^T\mathbf{L}^\dagger(\mathbf{e}_s-\mathbf{e}_t)$, we can approximate:}
\begin{equation}
    \begin{aligned}
        \hat{\mathbf{f}}(e)&=\sqrt{\frac{1}{d_s}+\frac{1}{d_t}}(\mathbf{e}_u-\mathbf{e}_v)^T\mathbf{\Pi D}^{-1/2}\mathbf{V}(\mathbf{I}-\mathbf{T})^{-1}\mathbf{e}_1\\
        &=\sqrt{\frac{1}{d_s}+\frac{1}{d_t}}(\mathbf{e}_u-\mathbf{e}_v)^T\mathbf{ D}^{-1/2}\mathbf{V}(\mathbf{I}-\mathbf{T})^{-1}\mathbf{e}_1
    \end{aligned}
\end{equation}
\textcolor{blue}{Which is equivalent to our formula in Line 1-2 in Algorithm \ref{algo:routing}. Using the same way as the proof of Theorem \ref{thm:lanczos_err}, we can control the error $| \hat{\mathbf{f}}(e)- \mathbf{f}(e)|\leq \epsilon/m$ when setting the iteration number $k=O(\kappa \log \frac{\kappa m}{\epsilon})=O(\kappa \log \frac{m}{\epsilon})$. As a result, $\Vert \mathbf{f}-\hat{\mathbf{f}}\Vert_1\leq \epsilon$. This finishes the proof.}
\end{proof}
}

For the next step, we use the same strategy as ~\cite{sinop2023robust} to compute $l$ robust paths using electric flow $\hat{\mathbf{f}}$: we iteratively find a path that maximizes the minimum flow from $s$ to $t$, then remove the corresponding minimum flow value from the edges it pass through. We repeat this process $2l$ times and finally find $l$ paths with the smallest cost. We provide an running example on Boston extracted from OpenStreetMap ~\cite{openstreetmap}. We randomly choose two node index $s,t$. We choose $l=10$ alternative routes and $k=150$ for Lanczos algorithm, we compare the electric flow-based routing with SOTA shortest path routing methods: \Penalty and \plateau in terms of three metrics following ~\cite{sinop2023robust}: (i) Stretch: The average length of the alternative paths over the length of $s,t$-shortest path; (ii) Diversity: The average pairwise Jaccard similarity among all alternative paths; (iii) Robustness: The probability that there is an alternative path after random edge deletions. Specifically, we delete each edge independently with probability $0.01$ in our experiments. The experimental results are reported in Table \ref{tab:routing}. As can be seen, the electric flow based routing output paths with better robustness and stretch comparing with \Penalty and with better robustness and diversity comparing with \plateau. Moreover, the runtime using Lanczos algorithm is comparable with the runtime of the SOTA methods. We note that in previous studies, the main bottleneck for the electric flow-based routing is the computation of the potential $\bm{\phi}$ ~\cite{sinop2023robust}. Our algorithm overcomes this bottleneck, enabling a fast and practical solution for alternative routing.

\begin{table}[t!]
\small
	\caption{A comparison over different methods for robust routing} \vspace{-0.2cm}
	\scalebox{1}{
		\begin{tabular}{c|c|c|c}
			\toprule
			\multicolumn{1}{c|}{Metric/Methods} & \multicolumn{1}{c|}{Electric Flow}&\multicolumn{1}{c|}{\Penalty}&\plateau\\
			\midrule
			Stretch& 1.15& 1.64& 1.09\\
             Diversity& 0.93& 0.98& 0.51\\
	       Robustness& 0.95& 0.85&0.52\\
                Runtime & 0.039& 0.036&0.099\\
            \bottomrule	
		\end{tabular}
	}
\end{table}\vspace{-0.2cm}\label{tab:routing}

\section{ Related Work}


\stitle{RD estimation with applications.} There is a long line of research regarding the estimation of RD. The current algorithms for RD estimation can be broadly classified in two categories: (i) RD computation for multiple vertex pairs ~\cite{spielman2008graph,chu2020graph,jambulapati2018efficient,li2023new,dwaraknath2024towards,yang2025improved}; (ii) RD computation for single source or single vertex pair ~\cite{peng2021local,liao2023resistance,yang2023efficient,liao2024efficient,cui2025mixing}. More specifically, for the first problem, the SOTA theoretical result is the $\tilde{O}(m+(n+|S|)/\epsilon^{1.5})$ time algorithm for $\epsilon$-approximates $S$ pairs of RD value for general undirected graphs ~\cite{chu2020graph}. For expander graphs, the time complexity is recently reduced to $\tilde{O}(\sqrt{nm}/\epsilon+|S|)$ while maintaining the same error guarantee ~\cite{yang2025improved}.
For the second problem, the SOTA algorithms ~\cite{peng2021local,yang2023efficient} estimate $r_\mathcal{G}(s,t)$ with $\epsilon$-absolute error guarantee in $\tilde{O}(1/\epsilon^2)$ time for a given pair of nodes $s,t$ for expander graphs. In this paper, we focus on improving the efficiency of the local algorithm on non-expander graphs (i.e., $\kappa$ is not $O(1)$), so our research has independent interest over these previous studies. 
The RD computation is widely used as a sub-procedure in graph sparsification ~\cite{spielman2008graph}, graph clustering ~\cite{alev2017graph}, counting random spanning trees ~\cite{chu2020graph,li2023new}, robust routing of road networks ~\cite{gao2023robust} and understanding oversquashing in GNNs \cite{black2023understanding}. 

\stitle{Lanczos iteration.} Lanczos iteration is a powerful algorithm both theoretically and experimentally. In numerical analysis, understanding the stability of Lanczos iteration in finite precision is a key issue. For example, ~\cite{meurant2006lanczos,musco2018stability} analyze the error propagation of Lanczos iteration in finite precision for matrix function approximation, ~\cite{paige1980accuracy} studies the error bound for eigenvalue approximation by Lanczos iteration, ~\cite{Chen2023StabilityOT} studies the stability of Lanczos algorithm under regular spectral distributions. However, in this paper, in contrast to just analyzing rounding errors produced by Lanczos iteration, we actively create errors by our subset Lanczos recurrence (Eq. \ref{equ:subset_lz_recurrence}) to make our algorithm implemented locally. As a result, our proposed techniques have independent interest over these previous researches. Recently, Lanczos iteration has also been applied in various graph algorithms. For example, ~\cite{orecchia2012approximating} use Lanczos iteration to compute the Heat Kernel PageRank value, ~\cite{bonchi2012fast} use Lanczos iteration to compute the Katz scores and commute times, ~\cite{tsitsulin2020just} use stochastic Lanczos quadrature to estimate the Von Neumann graph entropy, ~\cite{chen2021analysis,bhattacharjee2024improved} use stochastic Lanczos quadrature to design the linear time algorithm to approximate the spectral density of a symmetric matrix, and ~\cite{shi2017local} applies Lanczos iteration on a subgraph to computes the extreme eigenvector for local clustering. However, all these previous researches directly apply the Lanczos algorithm as a sub-procedure, while we provide the first algorithm that modifies the Lanczos algorithm by implementing it locally. Finally, note that although our algorithm is primarily designed to compute the RD value, it is also possible to be generalized to more applications. These generalizations are left to future work.

\comment{

\stitle{Local graph algorithms.} With the advent of large scale datasets, traditional linear time algorithms that require access to the entire graph have become infeasible. Therefore, more and more people are starting to study local graph algorithms. For example, ~\cite{liu2020strongly,fountoulakis2020p,fountoulakis2023flow,fountoulakis2019variational} study the local algorithms for conductance based local graph clustering, ~\cite{lofgren16bidirection,bressan2018sublinear,wang2024revisiting,lofgren2013personalized,chen2023accelerating,bautista2022local,wei2024approximating} study the local algorithms for PageRank computation, ~\cite{cohen2018approximating,jin2023moments} study the local algorithms for graph spectral density approximation, ~\cite{andoni2018solving} study the local algorithms for solving Laplacian systems. However, since their problem differs from ours, their techniques cannot be directly applied to our local RD computation problem.

}

\vspace{-0.3cm}
\section{Conclusion}
In this paper, we propose two efficient algorithms for RD computation, i.e., the Lanczos iteration and Lanczos Push. Specifically, we first show that the classical Lanczos iteration can be used to compute the RD value. We prove that Lanczos iteration can be implemented in $\tilde{O}(\sqrt{\kappa} m)$ time to approximate the RD value with $\epsilon$-absolute error guarantee, which is $\sqrt{\kappa}$ times faster than the traditional Power Method. Then, we develop a novel subset Lanczos recurrence technique to achieve a local algorithm for RD computation. We prove that under some modest assumptions, our Lanczos Push algorithm can be implemented in $\tilde{O}(\kappa^{2.75}C_1C_2/\epsilon)$ time to achieve an $\epsilon$-absolute error approximation of RD value, which achieves improvement of $\kappa^{0.25}$ over SOTA methods. Extensive experimental results demonstrate that our algorithms significantly outperform existing SOTA methods for RD computation, especially for large graphs and graphs with larger condition number $\kappa$.

\balance

\newpage

\bibliographystyle{plain}
\bibliography{ref}

\appendix



\section{Omitting proofs}

\subsection{Proof of Eq. ~\ref{equ:er_laylor_expansion}}

\begin{theorem}{(Eq. \ref{equ:er_laylor_expansion} restate)}
    \begin{align*}
r_\mathcal{G}(s,t)&=(\mathbf{e}_s-\mathbf{e}_t)^T\mathbf{L}^{\dagger}(\mathbf{e}_s-\mathbf{e}_t)\\
&=\frac{1}{2}(\mathbf{e}_s-\mathbf{e}_t)^T\sum_{k=0}^{+\infty}{\mathbf{D}^{-1}\left(\frac{1}{2}\mathbf{I}+\frac{1}{2}\mathbf{P}\right)^k}(\mathbf{e}_s-\mathbf{e}_t)
\end{align*}
\end{theorem}

\begin{proof}
    We use the similar argument as Lemma 4.3 in ~\cite{peng2021local}. First, we note that $\mathbf{L}=\mathbf{D}-\mathbf{A}=\mathbf{D}^{1/2}(\mathbf{I}-\mathcal{A})\mathbf{D}^{1/2}$, where $\mathcal{A}$ is the normalized adjacency matrix. Therefore, we have:
 \begin{align*}
    \mathbf{L}^\dagger&=\mathbf{D}^{-1/2}(\mathbf{I}-\mathcal{A})^\dagger\mathbf{D}^{-1/2}
    =\mathbf{D}^{-1/2}\sum_{j=2}^{n}{\frac{1}{1-\lambda_j(\mathcal{A})}\mathbf{u}_j\mathbf{u}_j^T}\mathbf{D}^{-1/2}\\
    &=\frac{1}{2}\mathbf{D}^{-1/2}\sum_{j=2}^{n}{\sum_{l=0}^{\infty}{\left(\frac{1}{2}+\frac{1}{2}\lambda_j(\mathcal{A})\right)^l}\mathbf{u}_j\mathbf{u}_j^T}\mathbf{D}^{-1/2}\\
    &=\frac{1}{2}\mathbf{D}^{-1/2}\sum_{l=0}^{\infty}{\sum_{j=2}^{n}{\left(\frac{1}{2}+\frac{1}{2}\lambda_j(\mathcal{A})\right)^l}\mathbf{u}_j\mathbf{u}_j^T}\mathbf{D}^{-1/2}\\
    &=\frac{1}{2}\mathbf{D}^{-1/2}\sum_{l=0}^{\infty}{\left(\left(\frac{1}{2}\mathbf{I}+\frac{1}{2}\mathcal{A}\right)^l-\mathbf{u}_1\mathbf{u}_1^T\right)}\mathbf{D}^{-1/2}\\
 \end{align*}
 Where $\mathbf{u}_j$ denotes the eigenvector of $\mathcal{A}$ corresponding to the eigenvalue $\lambda_j(\mathcal{A})$. Next, we notice the fact that $\mathbf{u}_1=\mathbf{D}^{1/2}\mathbf{1}\perp \mathbf{D}^{-1/2}(\mathbf{e}_s-\mathbf{e}_t)$ for any $s,t\in \mathcal{V}$, where $\mathbf{1}$ is the all-one vector. Therefore,
 \begin{align*}
     r_\mathcal{G}(s,t)&=(\mathbf{e}_s-\mathbf{e}_t)^T\mathbf{L}^{\dagger}(\mathbf{e}_s-\mathbf{e}_t)\\
     &=\frac{1}{2}(\mathbf{e}_s-\mathbf{e}_t)^T\mathbf{D}^{-1/2}\sum_{l=0}^{\infty}{\left(\left(\frac{1}{2}\mathbf{I}+\frac{1}{2}\mathcal{A}\right)^l-\mathbf{u}_1\mathbf{u}_1^T\right)}\mathbf{D}^{-1/2}(\mathbf{e}_s-\mathbf{e}_t)\\
     &=\frac{1}{2}(\mathbf{e}_s-\mathbf{e}_t)^T\mathbf{D}^{-1/2}\sum_{l=0}^{+\infty}{\left(\frac{1}{2}\mathbf{I}+\frac{1}{2}\mathcal{A}\right)^l}\mathbf{D}^{-1/2}(\mathbf{e}_s-\mathbf{e}_t)\\
     &=\frac{1}{2}(\mathbf{e}_s-\mathbf{e}_t)^T\mathbf{D}^{-1}\sum_{l=0}^{+\infty}{\left(\frac{1}{2}\mathbf{I}+\frac{1}{2}\mathbf{P}\right)^l}(\mathbf{e}_s-\mathbf{e}_t)
 \end{align*}
 This finishes the proof.
\end{proof}

\subsection{Proof of Theorem \ref{thm:pm_guarantee}}

\begin{theorem}{(Theorem \ref{thm:pm_guarantee} restate)}
    When setting $l=2\kappa\log \frac{\kappa}{\epsilon}$, the approximation error holds that $|\overline{r}_\mathcal{G}(s,t)-r_\mathcal{G}(s,t)|\leq \epsilon$. In addition, the time complexity of \powermethod is $O(lm)=O(\kappa m \log \frac{\kappa}{\epsilon})$.
\end{theorem}

\begin{proof}
    By $\overline{r}_\mathcal{G}(s,t)=\frac{1}{2}(\mathbf{e}_s-\mathbf{e}_t)^T\sum_{k=0}^{l}{\mathbf{D}^{-1}\left(\frac{1}{2}\mathbf{I}+\frac{1}{2}\mathbf{P}\right)^k}(\mathbf{e}_s-\mathbf{e}_t)$ is the output of \powermethod, the approximation error holds that:
    \begin{align*}
        &|\overline{r}_\mathcal{G}(s,t)-r_\mathcal{G}(s,t)|\\
        &= \left|\frac{1}{2}(\mathbf{e}_s-\mathbf{e}_t)^T\mathbf{D}^{-1}\sum_{l=L+1}^{\infty}{\left(\frac{1}{2}\mathbf{I}+\frac{1}{2}\mathbf{P}\right)^l}(\mathbf{e}_s-\mathbf{e}_t)\right|\\
        &=\left|\frac{1}{2}(\mathbf{e}_s-\mathbf{e}_t)^T\mathbf{D}^{-1/2}\sum_{l=L+1}^{\infty}{\left(\frac{1}{2}\mathbf{I}+\frac{1}{2}\mathcal{A}\right)^l}\mathbf{D}^{-1/2}(\mathbf{e}_s-\mathbf{e}_t)\right|\\
        &=\left|\frac{1}{2}(\mathbf{e}_s-\mathbf{e}_t)^T\mathbf{D}^{-1/2}\sum_{l=L+1}^{\infty}{\sum_{j=2}^{n}{\left(\frac{1}{2}+\frac{1}{2}\lambda_j(\mathcal{A})\right)^l}\mathbf{u}_j\mathbf{u}_j^T}\mathbf{D}^{-1/2}(\mathbf{e}_s-\mathbf{e}_t)\right|
    \end{align*}
    Where $\mathbf{u}_j$ denotes the eigenvector of $\mathcal{A}$ corresponding to the eigenvalue $\lambda_j(\mathcal{A})$. Next, we denote $\alpha_j=\mathbf{u}_j^T\mathbf{D}^{-1/2}(\mathbf{e}_s-\mathbf{e}_t)$, and $\alpha=(\alpha_1,...,\alpha_n)^T=(\mathbf{u}_1,...,\mathbf{u}_n)^T\mathbf{D}^{-1/2}(\mathbf{e}_s-\mathbf{e}_t)$. Therefore, we can bound
    \begin{align*}
    \Vert\alpha\Vert_2&=\Vert(\mathbf{u}_1,...,\mathbf{u}_n)\mathbf{D}^{-1/2}(\mathbf{e}_s-\mathbf{e}_t)\Vert_2
    \leq \Vert\mathbf{D}^{-1/2}(\mathbf{e}_s-\mathbf{e}_t)\Vert_2\leq \sqrt{2}
    \end{align*}
    And thus $\sum_{j=1}^{n}{\alpha_j^2}\leq 2$. Therefore, 
    \begin{align*}
       |\overline{r}_\mathcal{G}(s,t)-r_\mathcal{G}(s,t)|&=\left|\frac{1}{2}\sum_{l=L+1}^{\infty}{\sum_{j=2}^{n}{\left(\frac{1}{2}+\frac{1}{2}\lambda_j(\mathcal{A})\right)^l}}\alpha_j^2\right|\\
        &\leq \sum_{l=L+1}^{\infty}{\left(\frac{1}{2}+\frac{1}{2}\lambda_2(\mathcal{A})\right)^l}\\
    \end{align*}
    Where $\lambda_2(\mathcal{A})$ is the second largest eigenvalue of $\mathcal{A}$. By the property of the eigenvalues of normalized Laplacian matrix, $\lambda_2(\mathcal{L})=1-\lambda_2(\mathcal{A})$. Therefore,
    \begin{align*}
        |\overline{r}_\mathcal{G}(s,t)-r_\mathcal{G}(s,t)|
        &\leq \sum_{l=L+1}^{\infty}{\left(1-\frac{1}{2}\lambda_2(\mathcal{L})\right)^l}
        \leq \left(1-\frac{1}{2}\lambda_2(\mathcal{L})\right)^L\frac{2}{\lambda_2(\mathcal{L})} \\
        &\leq \left(1-\frac{1}{2}\lambda_2(\mathcal{L})\right)^L\kappa(\mathcal{L}) \leq \epsilon
    \end{align*}
    
    The final inequality holds when setting $L=\log \frac{\epsilon}{\kappa(\mathcal{L})}/ \log \left(1-\frac{1}{2}\lambda_2(\mathcal{L})\right)$. We notice the fact $\log \left(1-\frac{1}{2}\lambda_2(\mathcal{L})\right)\leq  -\frac{1}{2}\lambda_2(\mathcal{L})$ and $\kappa(\mathcal{L})\leq \frac{2}{\lambda_2(\mathcal{L})}$, so we set $L\geq 2\kappa(\mathcal{L}) \log \frac{n}{\epsilon}$ is enough. Finally, we note that for each iteration of \powermethod we need to compute the matrix-vector multiplication $\left(\frac{1}{2}\mathbf{I}+\frac{1}{2}\mathbf{P}\right) \mathbf{r}$, which requires $O(m)$ operations. Therefore the total time complexity of \powermethod is $O(lm)=O(\kappa m \log \frac{\kappa}{\epsilon})$.
\end{proof}

\subsection{Proof of Lemma \ref{lem:lz_polynomial_correct}}

\begin{lemma}{(Lemma \ref{lem:lz_polynomial_correct} restate)}
    Given $\mathbf{V}=[\mathbf{v}_1,...,\mathbf{v}_k]$ is the $n\times k$ dimension orthogonal matrix spanning the subspace $\mathcal{K}_k(\mathbf{v}_1, \mathcal{A})=span\langle\mathbf{v}_1, \mathcal{A}\mathbf{v}_1,..., \mathcal{A}^k\mathbf{v}_1\rangle$ and $\mathbf{T}=\mathbf{V}^T\mathcal{A}\mathbf{V}$. Then for any polynomial $p_k$ with degree at most $k$, we have $p_k(\mathcal{A})\mathbf{v}_1=\mathbf{V}p_k(\mathbf{T})\mathbf{V}^T\mathbf{v}_1$.
\end{lemma}

\begin{proof}
    Since $\mathcal{A}^i\mathbf{v}_1\in \mathcal{K}_k(\mathbf{v}_1, \mathcal{A})$ for $i\leq k$ and $\mathbf{VV}^T$ is the projection matrix onto the subspace $\mathcal{K}_k(\mathbf{v}_1, \mathcal{A})$, we have $\mathcal{A}^i\mathbf{v}_1=\mathbf{VV}^T\mathcal{A}^i\mathbf{v}_1$. Therefore
    \begin{align*}
        \mathcal{A}^i\mathbf{v}_1&=(\mathbf{VV}^T)\mathcal{A}(\mathbf{VV}^T)\mathcal{A}... \mathcal{A}(\mathbf{VV}^T)\mathbf{v}_1\\
        &=\mathbf{V}(\mathbf{V}^T\mathcal{A}\mathbf{V})...(\mathbf{V}^T\mathcal{A}\mathbf{V})\mathbf{V}^T\mathbf{v}_1=\mathbf{V}\mathbf{T}^i\mathbf{V}^T\mathbf{v}_1.
    \end{align*}
    holds for any $i\leq k$. Thus $p_k(\mathcal{A})\mathbf{v}_1=\mathbf{V}p_k(\mathbf{T})\mathbf{V}^T\mathbf{v}_1$ for any polynomial $p_k$ with degree at most $k$.
\end{proof}

\subsection{Proof of Lemma \ref{lem:lz_quadratic_err}}

\begin{lemma}{(Lemma \ref{lem:lz_quadratic_err} restate)}
    Let $\mathbf{V}\in \mathbb{R}^{n\times k}$, $\mathbf{T}\in \mathbb{R}^{k\times k}$ defined in Lemma \ref{lem:lz_polynomial_correct}. Let $\mathcal{P}_k$ be the set of all polynomials with degree $\leq k$. Then, the error satisfy:
\begin{align*}
    |\mathbf{v}_1^T&\mathbf{V}(\mathbf{I}-\mathbf{T})^{-1}\mathbf{V}^T\mathbf{v}_1-\mathbf{v}_1^T(\mathbf{I}-\mathcal{A})^\dagger\mathbf{v}_1|\\
   & \leq 2 \min_{p\in \mathcal{P}_k}{\max_{x \in [\lambda_{min}(\mathcal{A}),\lambda_{2}(\mathcal{A})]}{|1/(1-x)-p(x)|}}
    \end{align*}
    
    Where $\lambda_{min}(\mathcal{A})$ and $\lambda_{2}(\mathcal{A})$ denotes the smallest and second largest eigenvalue of $\mathcal{A}$, respectively.
\end{lemma}

\begin{proof}

We note that for any polynomial $p$ of degree $\leq k$, the following inequality holds:
\begin{equation}\label{equ:lanczos_bound}
\begin{aligned}
    |\mathbf{v}_1^T\mathbf{V}&(\mathbf{I}-\mathbf{T})^{-1}\mathbf{V}^T\mathbf{v}_1-\mathbf{v}_1^T(\mathbf{I}-\mathcal{A})^\dagger\mathbf{v}_1| =  |\mathbf{e}_1^T(\mathbf{I}-\mathbf{T})^{-1}\mathbf{e}_1-\mathbf{v}_1^T(\mathbf{I}-\mathcal{A})^\dagger\mathbf{v}_1| \\
    & \leq |\mathbf{v}_1^T (\mathbf{I}-\mathcal{A})^\dagger \mathbf{v}_1 - \mathbf{v}_1^T p(\mathcal{A}) \mathbf{v}_1 | 
    + | \mathbf{v}_1^T p(\mathcal{A}) \mathbf{v}_1 -\mathbf{e}_1^T p(\mathbf{T})\mathbf{e}_1  | \\
    &+ |  \mathbf{e}_1^T p(\mathbf{T})\mathbf{e}_1 -  \mathbf{e}_1^T (\mathbf{I}-\mathbf{T})^{-1}\mathbf{e}_1| \\
\end{aligned}
\end{equation}

For the second term of the right hand side of Eq. (\ref{equ:lanczos_bound}), by Lemma \ref{lem:lz_polynomial_correct} this is $0$. Now we bound the first term and third term of the above inequality. We notice that $\mathcal{A} $ has the eigen-decomposition $\mathcal{A}=\sum_{i=1}^{n}{\lambda_i \mathbf{u}_i\mathbf{u}_i^T}$, where $\lambda_i$ is the $i$-th largest eigenvalue of $\mathcal{A}$ and $\mathbf{u}_i$ is the corresponding eigenvector. By the definition of pseudo-inverse, $(\mathbf{I}-\mathcal{A})^\dagger$ has the eigen-decomposition $(\mathbf{I}-\mathcal{A})^\dagger=\sum_{i=2}^{n}{\frac{1}{1-\lambda_i} \mathbf{u}_i\mathbf{u}_i^T}$, and $p(\mathcal{A})$ has the eigen-decomposition $p(\mathcal{A})=\sum_{i=1}^{n}{p(\lambda_i) \mathbf{u}_i\mathbf{u}_i^T}$. By the fact that the first eigenvector $\mathbf{u}_1=\mathbf{D}^{1/2} \mathbf{1}\perp \mathbf{v}_1$, so the first term of the second line of Equ. (\ref{equ:lanczos_bound}) can be expressed below:
\begin{align*}
    |\mathbf{v}_1^T (\mathbf{I}-\mathcal{A})^\dagger \mathbf{v}_1 - \mathbf{v}_1^T p(\mathcal{A}) \mathbf{v}_1 | = |\sum_{i=2}^{n}{(\frac{1}{1-\lambda_i}-p(\lambda_i))(\mathbf{v}_1^T\mathbf{u}_i)^2  } |
\end{align*}
Next, we define another vector $c=\mathbf{v}_1^T (\mathbf{u}_1,...,\mathbf{u}_n)$. We notice that $\Vert c \Vert_2 = \Vert \mathbf{v}_1^T (\mathbf{u}_1,...,\mathbf{u}_n) \Vert_2=\Vert \mathbf{v}_1  \Vert_2 = 1$, and thus $\sum_{i=1}^{n}{c(i)^2}=\sum_{i=1}^{n}{(\mathbf{v}_1^T\mathbf{u}_i)^2}=1$. Therefore, we can bound the above equality:
\begin{align*}
    |\sum_{i=2}^{n}{(\frac{1}{1-\lambda_i}-p(\lambda_i))(\mathbf{v}_1^T\mathbf{u}_i)^2  } | \leq \max_{x \in [\lambda_{min}(\mathcal{A}),\lambda_{2}(\mathcal{A})]}{|1/(1-x)-p(x)|}
\end{align*}
Using the similar argument on $\mathbf{T}$, we have:
\begin{align*}
    |  \mathbf{e}_1^T p(\mathbf{T})\mathbf{e}_1 -  \mathbf{e}_1^T (\mathbf{I}-\mathbf{T})^{-1}\mathbf{e}_1| \leq \max_{x \in [\lambda_{min}(\mathbf{T}),\lambda_{max}(\mathbf{T})]}{|1/(1-x)-p(x)|}
\end{align*}
Finally, we observe that $\mathbf{T}$ is the projection of $\mathcal{A}$ onto the subspace $\mathcal{K}_k(\mathcal{A},\mathbf{v}_1)=span\langle \mathbf{v}_1,\mathcal{A}\mathbf{v}_1,...,\mathcal{A}^k\mathbf{v}_1\rangle$ with $\mathbf{v}_1$ orthogonal to the eigenspace of the largest eigenvalue $\lambda_1(\mathcal{A})$. So $\lambda_{min}(\mathbf{T})\geq \lambda_{min}(\mathcal{A})$ and $\lambda_{max}(\mathbf{T})\leq \lambda_2 (\mathcal{A})$. Therefore we have that:
\begin{align*}
    \max_{x \in [\lambda_{min}(\mathbf{T}),\lambda_{max}(\mathbf{T})]}&{|1/(1-x)-p(x)|} \\
    &\leq \max_{x \in [\lambda_{min}(\mathcal{A}),\lambda_{2}(\mathcal{A})]}{|1/(1-x)-p(x)|}
\end{align*}
Putting these inequalities together, we have the result of Lemma \ref{lem:lz_quadratic_err}:
\begin{align*}
    |\mathbf{v}_1^T&\mathbf{V}(\mathbf{I}-\mathbf{T})^{-1}\mathbf{V}^T\mathbf{v}_1-\mathbf{v}_1^T(\mathbf{I}-\mathcal{A})^\dagger\mathbf{v}_1|\\
   & \leq 2 \min_{p\in \mathcal{P}_k}{\max_{x \in [\lambda_{min}(\mathcal{A}),\lambda_{2}(\mathcal{A})]}{|1/(1-x)-p(x)|}}
    \end{align*}
    This finishes the proof. 

\end{proof}

\subsection{Proof of Theorem \ref{thm:lanczos_err}}

\begin{theorem}{(Theorem \ref{thm:lanczos_err} restate)}
    Let $\hat{r}_\mathcal{G}(s,t)$ output by Lanczos iteration (Algorithm \ref{algo:lanczos}). Then the error satisfies
\begin{align*}
    |\hat{r}_\mathcal{G}(s,t)&-r_\mathcal{G}(s,t)|\leq \epsilon
    \end{align*}
    When setting the iteration number $k=O(\sqrt{\kappa}\log \frac{\kappa}{\epsilon})$. Furthermore, the time complexity of Algorithm \ref{algo:lanczos} is $O(km+k^{2.37})=O (\sqrt{\kappa} m \log \frac{\kappa}{\epsilon})$.
\end{theorem}

We observe that $\hat{r}_\mathcal{G}(s,t)$ output by Algorithm \ref{algo:lanczos} satisfies $\hat{r}_\mathcal{G}(s,t)=\left(\frac{1}{d_s}+\frac{1}{d_t}\right)\mathbf{e}_1^T (\mathbf{I}-\mathbf{T})^{-1} \mathbf{e}_1$ and $r_\mathcal{G}(s,t)=\left(\frac{1}{d_s}+\frac{1}{d_t}\right)\mathbf{v}_1(\mathbf{I}-\mathcal{A})^\dagger \mathbf{v}_1$, so by Lemma \ref{lem:lz_quadratic_err}, we have that:
\begin{align*}
    |\hat{r}_\mathcal{G}(s,t)-&r_\mathcal{G}(s,t)|\leq \\
    &\left(\frac{1}{d_s}+\frac{1}{d_t}\right)2 \min_{p\in \mathcal{P}_k}{\max_{x \in [\lambda_{min}(\mathcal{A}),\lambda_{2}(\mathcal{A})]}{|1/(1-x)-p(x)|}}
\end{align*}
So all the remaining works is to bound the right hand side of the above inequalities. To reach this end, we prove that there exist a polynomial of degree at most $k=O(\sqrt{\kappa}\log \frac{\kappa}{\epsilon})$ that $\epsilon$-approximates $1/(1-x)$ in the range $[\lambda_{min}(\mathcal{A}),\lambda_{2}(\mathcal{A})]$.

\begin{theorem}\label{thm:polynomial_approx}
    For any given infinitely differentiable function $f$, parameter $\epsilon \in (0,1)$, matrix $\mathcal{A}\in \mathbb{R}^{n\times n}$ with condition number $\kappa$. Set $k=O(\sqrt{\kappa}\log \frac{\kappa}{\epsilon})$. There exists a polynomial $p_k$ of degree $\leq k$, satisfy: $$\max_{x \in [\lambda_{min}(\mathcal{A}),\lambda_{2}(\mathcal{A})]}{|1/(1-x)-p_k(x)|}< \frac{\epsilon}{2}$$
    Specifically, $p_k(x)=\sum_{l=0}^{k}{c_l T_l(x)}$ is the linear combination of Chebyshev polynomials $T_l(x)$ with coefficients $c_l>0$, and $\sum_{l}{c_l}\leq \kappa \log (\frac{\kappa}{\epsilon})$.
\end{theorem}

We observe that Theorem \ref{thm:lanczos_err} immediately follows by Theorem \ref{thm:polynomial_approx}. Actually, by Theorem \ref{thm:polynomial_approx}, we set the iteration number $k=O(\sqrt{\kappa}\log \frac{\kappa}{\epsilon})$ to reach the $\epsilon$-absolute value guarantee. In each iteration we need $O(m)$ operations and finally we need the additional $O(k^{2.37})$ time to compute $\mathbf{e}_1^T(\mathbf{I}-\mathbf{T})^{-1}\mathbf{e}_1$. So the total time complexity is bounded by $O (km+k^{2.37})=O (\sqrt{\kappa} m \log \frac{\kappa}{\epsilon})$. The remaining parts is the proof for Theorem \ref{thm:polynomial_approx}. To prove Theorem \ref{thm:polynomial_approx}, we use the following polynomial expansion by the first kind of Chebyshev polynomials. We recall that the first kind of Chebyshev polynomials $\{T_l(x)\}_{l\geq 0}$ is generated by the following recurrence:
\begin{equation}
\begin{aligned}
    T_0(x)&=1, T_1(x)= x \\
    T_{l+1}(x)&=2xT_l(x)-T_{l-1}(x),  \ \ \  l\geq 1
\end{aligned}
\end{equation}\label{equ:chebyshev}

\begin{lemma}\label{lem:xk_chebyshev_expansion}
    For any $k\geq 0$, $x^k=\sum_{l\in \mathbb{Z}}{q_k(l)T_{|l|}(x)}$ for $x\in [-1,1]$. Where $q_k(l)$ is the probability that simple random walk on $\mathbb{Z}$ starting at $0$ is at $l$ after $k$ steps.
\end{lemma}
\begin{proof}
    Lemma \ref{lem:xk_chebyshev_expansion} follows by computation. Set $x=\cos \theta$ and $w=e^{i\theta}$, we have that:
    \begin{align*}
        x^k&=[(w+w^{-1})/2]^k=\sum_{l\in \mathbb{Z}}{q_k(l)w^l}\\
        &=\sum_{l\in \mathbb{Z}}{q_k(l)(w^l+w^{-l})/2}=\sum_{l\in \mathbb{Z}}{q_k(l)\cos (l\theta)}\\
        &=\sum_{l\in \mathbb{Z}}{q_k(l)T_{|l|}(x)}
    \end{align*}
    This proves Lemma \ref{lem:xk_chebyshev_expansion}.
\end{proof}

Next, we use Lemma \ref{lem:xk_chebyshev_expansion} to prove Theorem \ref{thm:polynomial_approx}. The sketch of the proof is as follows: we first use the Taylor expansion to express $1/(1-x)$: $1/(1-x)=\frac{1}{2}\sum_{t=0}^{\infty}{(\frac{1}{2}+\frac{1}{2}x)^t}$. Next, we choose a sufficinetly large truncation step $L$, and consider the $L$ step truncation polynomial $\frac{1}{2}\sum_{t=0}^{L}{(\frac{1}{2}+\frac{1}{2}x)^t}$. We prove that $\frac{1}{2}\sum_{t=0}^{L}{(\frac{1}{2}+\frac{1}{2}x)^t}$ is close enough to $1/(1-x)$ when setting $L=\kappa\log (\frac{\kappa}{\epsilon})$. Then, we approximate the degree $L$ polynomial $\frac{1}{2}\sum_{t=0}^{L}{(\frac{1}{2}+\frac{1}{2}x)^t}$ by first $k$ Chebyshev polynomials (i.e. using Lemma \ref{lem:xk_chebyshev_expansion} ), we prove that setting $k=\sqrt{L}\log \frac{\kappa}{\epsilon}$ is enough. To reach this end, we have the following two claims:
    \begin{claim}\label{claim1}
        $\max_{x \in [\lambda_{min}(\mathcal{A}),\lambda_{2}(\mathcal{A})]}{|1/(1-x)-\frac{1}{2}\sum_{t=0}^{L}{(\frac{1}{2}+\frac{1}{2}x)^t}|}< \frac{\epsilon}{4}$ when setting $L=O(\kappa\log (\frac{\kappa}{\epsilon}))$.
    \end{claim}
\begin{proof}
    By the property of Taylor expansion, we know that $1/(1-x)-\frac{1}{2}\sum_{t=0}^{L}{(\frac{1}{2}+\frac{1}{2}x)^t}=\frac{1}{2}\sum_{t=L+1}^{\infty}{(\frac{1}{2}+\frac{1}{2}x)^t}$. Next, for $x \in [\lambda_{min}(\mathcal{A}),\lambda_{2}(\mathcal{A})]$, we can bound $\frac{1}{2}\sum_{t=L+1}^{\infty}{(\frac{1}{2}+\frac{1}{2}x)^t}=\frac{(\frac{1}{2}+\frac{1}{2}x)^L}{1-x}\leq\frac{(\frac{1}{2}+\frac{1}{2}\lambda_2(\mathcal{A}))^L}{1-\lambda_2(\mathcal{A})}=\frac{\kappa}{2}(\frac{1}{2}+\frac{1}{2}\lambda_2(\mathcal{A}))^L \leq \frac{\epsilon}{4}$. The last inequality holds by setting $L=\kappa\log (\frac{\kappa}{\epsilon})$. This proves Claim \ref{claim1}. 
    \end{proof}

    \begin{claim}\label{claim2}
       There exists a polynomial $p_k$ of degree at most $k$, such that $\max_{x \in [-1,1]}{|\frac{1}{2}\sum_{t=0}^{L}{(\frac{1}{2}+\frac{1}{2}x)^t}-p_k(x)|}< \frac{\epsilon}{4}$ when setting $k=\sqrt{2L}\log ^{1/2} (\frac{8L}{\epsilon})$.
    \end{claim}

    \begin{proof}
    First, for each moment $x^t$, we use Lemma \ref{lem:xk_chebyshev_expansion} : $x^t=\sum_{l\in \mathbb{Z}}{q_t(l)T_{|l|}(x)}$. By Azuma-Hoeffding inequality, $\sum_{|l|\geq k}{q_t(l)}\leq 2 \exp (-k^2/2t)$. We choose $k=\sqrt{2L}\log ^{1/2} (\frac{8L}{\epsilon})$, then $\sum_{|l|\geq k}{q_t(l)}\leq \frac{\epsilon}{4L}$. So, we define the degree $k$ Chebyshev truncation polynomial: $q_t(x)=\sum_{|l|\leq k}{q_t(l)T_{|l|}(x)}$, we have:
    $$\max_{x \in [-1,1]}{|x^t-q_t(x)|}
    \leq \sum_{|l|\geq k}{q_t(l)}
    \leq \frac{\epsilon}{4L}$$
    Then, we define the degree $k$ polynomial $p_k$ as $p_k(x)=\sum_{t=0}^{L}{q_t(x)}$, we have:
    \begin{equation}
    \begin{aligned}
    \max_{x \in [-1,1]}{|\frac{1}{2}\sum_{t=0}^{L}{(\frac{1}{2}+\frac{1}{2}x)^t}-p_k(x)|}
    &\leq \sum_{t=0}^{L}{ \max_{x \in [-1,1]}{|x^t-q_t(x)|}} \\
    &\leq L \frac{\epsilon}{4L}
    = \frac{\epsilon}{4}
    \end{aligned}
    \end{equation}
    This proves Claim \ref{claim2}.
    \end{proof}

By Claim \ref{claim1} and Claim \ref{claim2}, we know that $1/(1-x)$ can be approximated by a polynomial $p_k$ with error at most $\frac{\epsilon}{2}$, and $k=\sqrt{2L}\log ^{1/2} (\frac{8L}{\epsilon})=O(\sqrt{\kappa}\log\frac{\kappa}{\epsilon})$. Finally, by $p_k(x)=\sum_{t=0}^{L}{q_t(x)}$ and $q_t(x)$ is the linear combination of Chebyshev polynomials for $t\leq L$, so $p_k(x)$ is the linear combination of Chebyshev polynomials. We consider this expansion $p_k(x)=\sum_{t=0}^{L}{q_t(x)}=\sum_{l=0}^{k}{c_lT_l(x)}$ with coefficients $c_l$. Since $q_t(l)$ is a probability distribution (i.e., $\sum_{l\in \mathbb{Z}}{q_t(l)}=1$ and $q_t(l)\geq 0$), so the coefficients $c_l\geq 0$ and the sum is at most: $\sum_{l=0}^{k}{c_l}\leq \sum_{l=0}^{L}{\sum_{l\in \mathbb{Z}}{q_t(l)}}=L=\kappa \log (\frac{\kappa}{\epsilon})$.

\comment{
\subsection{Proof of Proposition \ref{prop:lanczos_push_guarantee}}
\begin{proposition}{(Proposition \ref{prop:lanczos_push_guarantee} restate)}
    The Lanczos Push algorithm (Algorithm ~\ref{algo:lanczos_local}) computes $\hat{\mathbf{V}}\in \mathbb{R}^{n\times k}$, $\hat{\mathbf{T}}\in \mathbb{R}^{k\times k}$ and an additional vector $\hat{\mathbf{v}}_{k+1}$, a scalar $\hat{\beta}_{k+1}$ such that:
$$\mathcal{A}\hat{\mathbf{V}}=\hat{\mathbf{V}}\hat{\mathbf{T}}+\hat{\beta}_{k+1}\hat{\mathbf{v}}_{k+1}\mathbf{e}_k^T+\mathbf{E}$$
Where, $\mathbf{E}$ is defined as the error matrix with $\mathbf{E}=[\delta_1,\delta_2,...,\delta_k]$ and $\delta_i\in \mathbb{R}^n$ is defined as $\delta_i=(AMV(\mathcal{A},\hat{\mathbf{v}}_i)-\mathcal{A}\hat{\mathbf{v}}_i)-\hat{\alpha}_i\hat{\mathbf{v}}_i|_{\mathcal{V}-S_i}-\hat{\beta}_i\hat{\mathbf{v}}_{i-1}|_{\mathcal{V}-S_{i-1}}$ for each $i\leq k$.
\end{proposition}

\begin{proof}
    Recall that in each iteration of Algorithm \ref{algo:lanczos_local} (i.e., Line 3-17), we perform subset Lanczos recurrence:
    \begin{align*}
    \hat{\beta}_{i+1}\hat{\mathbf{v}}_{i+1}&=AMV(\mathcal{A},\hat{\mathbf{v}}_i)-\hat{\alpha}_i\hat{\mathbf{v}}_i|_{S_i}-\hat{\beta}_i\hat{\mathbf{v}}_{i-1}|_{S_{i-1}}  \\
    &=(\mathcal{A}-\hat{\alpha}_i)\hat{\mathbf{v}}_i-\hat{\beta}_i\hat{\mathbf{v}}_{i-1}+ \delta_i
    \end{align*}
Where $S_i=\{u\in \mathcal{V}: |\hat{\mathbf{v}}_i(u)|>\epsilon d_u\}$ and $\delta_i=(AMV(\mathcal{A},\hat{\mathbf{v}}_i)-\mathcal{A}\hat{\mathbf{v}}_i)-\hat{\alpha}_i\hat{\mathbf{v}}_i|_{\mathcal{V}-S_i}-\hat{\beta}_i\hat{\mathbf{v}}_{i-1}|_{\mathcal{V}-S_{i-1}}$. Since the equation holds for any $i\leq k$, by the definition of $\hat{\mathbf{V}}$, $\hat{\mathbf{T}}$ and $\mathbf{E}$, this is equivalent to:$$\mathcal{A}\hat{\mathbf{V}}=\hat{\mathbf{V}}\hat{\mathbf{T}}+\hat{\beta}_{k+1}\hat{\mathbf{v}}_{k+1}\mathbf{e}_k^T+\mathbf{E}$$
This finishes the proof.
\end{proof}
}



\subsection{Proof of Theorem \ref{thm:lanczos_push_error}}

\begin{theorem}{(Theorem \ref{thm:lanczos_push_error} restate)}
   Given $\epsilon'\in (0,1)$ , setting iteration number $k=\sqrt{\kappa}\log \frac{\kappa}{\epsilon'}$, parameter $\epsilon=\tilde{\Omega}(\frac{\epsilon'\sqrt{d}}{\kappa^{2.25}C_1})$, the error between the approximation $r'_\mathcal{G} (s,t) $ output by Algorithm \ref{algo:lanczos_local} and accurate ER value $r_\mathcal{G} (s,t) $ can be bounded by:
    $$|r_\mathcal{G} (s,t)-r'_\mathcal{G} (s,t)| \leq \epsilon'$$
\end{theorem}

For theoretical analysis, recall that we need assumption \ref{assump:lzpush_error}. Now we turn to the proof for Theorem \ref{thm:lanczos_push_error}. Similar as Eq. \ref{equ:lanczos_bound} in the proof of Theorem \ref{thm:lanczos_err}, we split the error into three terms:
\begin{equation}\label{equ:lzpush_bound}
\begin{aligned}
    |\hat{r}_\mathcal{G}(s,t)-r_\mathcal{G}(s,t)|/(\frac{1}{d_s}+ \frac{1}{d_t}) & \leq |\mathbf{v}_1^T (\mathbf{I}-\mathcal{A})^\dagger \mathbf{v}_1 - \mathbf{v}_1^T p(\mathcal{A}) \mathbf{v}_1 | \\
    &+ | \mathbf{v}_1^T p(\mathcal{A}) \mathbf{v}_1 -\hat{\mathbf{v}}_1^T \hat{\mathbf{V}}p(\hat{\mathbf{T}})\mathbf{e}_1  | \\
    &+ |  \hat{\mathbf{v}}_1^T \hat{\mathbf{V}} p(\hat{\mathbf{T}})\mathbf{e}_1 -  \mathbf{e}_1^T (\mathbf{I}-\hat{\mathbf{T}})^{-1}\mathbf{e}_1| \\
\end{aligned}
\end{equation}

For the first term of the right hand side of Eq. \ref{equ:lzpush_bound}, by the same argument as the proof for Lemma \ref{lem:lz_quadratic_err} this is smaller than $\frac{\epsilon'}{8}$ by setting the iteration number $k=O(\sqrt{\kappa}\log (\frac{\kappa}{\epsilon'}))$. For the third term of the right hand side of Eq. \ref{equ:lzpush_bound}, by Assumption \ref{assump:lzpush_error} and the same argument as the proof for Lemma \ref{lem:lz_quadratic_err}, this is also smaller than $\frac{\epsilon'}{8}$. So we have the error bound:
\begin{align*}
    |\mathbf{v}_1^T (\mathbf{I}-\mathcal{A})^\dagger \mathbf{v}_1-\hat{\mathbf{v}}_1^T \hat{\mathbf{V}}(\mathbf{I}-\hat{\mathbf{T}})^{-1}\mathbf{e}_1| 
    \leq| \mathbf{v}_1^T p(\mathcal{A}) \mathbf{v}_1 -\hat{\mathbf{v}}_1^T \hat{\mathbf{V}}p(\hat{\mathbf{T}})\mathbf{e}_1  | +\frac{\epsilon'}{4}
\end{align*}

Next we focus on the bound to the middle term $| \mathbf{v}_1^T p(\mathcal{A}) \mathbf{v}_1 -\hat{\mathbf{v}}_1^T \hat{\mathbf{V}}p(\hat{\mathbf{T}})\mathbf{e}_1  |$ in Eq. \ref{equ:lzpush_bound}. Recall that $p$ is the polynomial chosen by Lemma \ref{thm:polynomial_approx}, which is expanded by $p(x)=\sum_{l=0}^{k}{c_lT_l(x)}$, where $\{T_l(x)\}$ are the first kind of Chebyshev polynomial series and the coefficients $\sum_l{|c_l|}\leq \tilde{O}(\kappa)$. Therefore the polynomial deviation series can be further expressed by \begin{equation}\label{equ:lzpush_polynomial_error}
    \begin{aligned}
     \mathbf{v}_1^T p(\mathcal{A})\mathbf{v}_1 - \mathbf{v}_1^T\hat{\mathbf{V}}p(\hat{\mathbf{T}})\mathbf{e}_1&=\sum_{l=0}^{k}{c_l \mathbf{v}_1^T(\hat{\mathbf{V}}T_l(\hat{\mathbf{T}})\mathbf{e}_1-T_l(\mathcal{A})\mathbf{v}_1)}
    \end{aligned}
\end{equation}

We define $d_l= \hat{\mathbf{V}}T_l(\hat{\mathbf{T}})\mathbf{e}_1-T_l(\mathcal{A})\mathbf{v}_1$ be the deviation constraint on the Chebyshev basis $T_l(x)$. So the error term can be now simply expressed by:
\begin{align*}
    |\mathbf{v}_1^T p(\mathcal{A})\mathbf{v}_1 - \mathbf{v}_1^T\hat{\mathbf{V}}p(\hat{\mathbf{T}})\mathbf{e}_1|&\leq \sum_{l=0}^{k}|c_l||\mathbf{v}_1^Td_l|
\end{align*}

Now we focus on the bound of $d_l$. For the next step, we follow the same strategy as the analysis for the stability of the Lanczos algorithm ~\cite{musco2018stability}. First, by comparing with Proposition \ref{prop:lanczos_guarantee}, we observe that Lanczos Push also satisfies the similar equation with an additional error term.

\begin{proposition}\label{prop:lanczos_push_guarantee}
    The Lanczos Push algorithm (Algorithm ~\ref{algo:lanczos_local}) computes $\hat{\mathbf{V}}\in \mathbb{R}^{n\times k}$, $\hat{\mathbf{T}}\in \mathbb{R}^{k\times k}$ and an additional vector $\hat{\mathbf{v}}_{k+1}$, a scalar $\hat{\beta}_{k+1}$ such that:
$$\mathcal{A}\hat{\mathbf{V}}=\hat{\mathbf{V}}\hat{\mathbf{T}}+\hat{\beta}_{k+1}\hat{\mathbf{v}}_{k+1}\mathbf{e}_k^T+\mathbf{E}$$
Where, $\mathbf{E}$ is defined as the error matrix with $\mathbf{E}=[\delta_1,\delta_2,...,\delta_k]$ and $\delta_i\in \mathbb{R}^n$ is defined as $\delta_i=(AMV(\mathcal{A},\hat{\mathbf{v}}_i)-\mathcal{A}\hat{\mathbf{v}}_i)-\hat{\alpha}_i\hat{\mathbf{v}}_i|_{\mathcal{V}-S_i}-\hat{\beta}_i\hat{\mathbf{v}}_{i-1}|_{\mathcal{V}-S_{i-1}}$ for each $i\leq k$.
\end{proposition}

\begin{proof}
    Recall that in each iteration of Algorithm \ref{algo:lanczos_local} (i.e., Line 3-17), we perform subset Lanczos recurrence:
    \begin{align*}
    \hat{\beta}_{i+1}\hat{\mathbf{v}}_{i+1}&=AMV(\mathcal{A},\hat{\mathbf{v}}_i)-\hat{\alpha}_i\hat{\mathbf{v}}_i|_{S_i}-\hat{\beta}_i\hat{\mathbf{v}}_{i-1}|_{S_{i-1}}  \\
    &=(\mathcal{A}-\hat{\alpha}_i)\hat{\mathbf{v}}_i-\hat{\beta}_i\hat{\mathbf{v}}_{i-1}+ \delta_i
    \end{align*}
Where $S_i=\{u\in \mathcal{V}: |\hat{\mathbf{v}}_i(u)|>\epsilon d_u\}$ and $\delta_i=(AMV(\mathcal{A},\hat{\mathbf{v}}_i)-\mathcal{A}\hat{\mathbf{v}}_i)-\hat{\alpha}_i\hat{\mathbf{v}}_i|_{\mathcal{V}-S_i}-\hat{\beta}_i\hat{\mathbf{v}}_{i-1}|_{\mathcal{V}-S_{i-1}}$. Since the equation holds for any $i\leq k$, by the definition of $\hat{\mathbf{V}}$, $\hat{\mathbf{T}}$ and $\mathbf{E}$, this is equivalent to:$$\mathcal{A}\hat{\mathbf{V}}=\hat{\mathbf{V}}\hat{\mathbf{T}}+\hat{\beta}_{k+1}\hat{\mathbf{v}}_{k+1}\mathbf{e}_k^T+\mathbf{E}$$
This finishes the proof.
\end{proof}

Next, we prove that the deviation $d_l$ can be expressed by the linear combination of the second kind of Chebyshev polynomials over $\xi_j$ for $j\leq l$, where $\xi_j$ is defined as $\xi_j=2\mathbf{E}T_j(\hat{\mathbf{T}})\mathbf{e}_1$. Before we formally provide the result, we first recall that the second kind of Chebyshev polynomials $\{U_l(x)\}_{l\geq 0}$ is generated by the following recurrence:
\begin{align*}
    U_0(x)&=1, U_1(x)= 2x \\
    U_{l+1}(x)&=2xU_l(x)-U_{l-1}(x),  \ \ \  l\geq 1
\end{align*}

\begin{lemma}\label{lem:deviation_chebyshev}
    Let $\hat{\mathbf{V}}$, $\hat{\mathbf{T}}$ output by Lanczos Push, and $\mathbf{E}$ be the error matrix defined in Proposition \ref{prop:lanczos_push_guarantee}. We define $\xi_l=2\mathbf{E}T_l(\hat{\mathbf{T}})\mathbf{e}_1$ for each integer $l\leq k$, and $d_l= \hat{\mathbf{V}}T_l(\hat{\mathbf{T}})\mathbf{e}_1-T_l(\mathcal{A})\mathbf{v}_1$ be the deviation produced by Lanczos Push constraint on $T_l(\mathcal{A})$. Then, $d_l$ can be expressed by $d_l=-\frac{1}{2}U_{l-1}(\mathcal{A})\xi_1-\sum_{j=2}^{l-1}{U_{l-j}(\mathcal{A})\xi_j}$ for each $l\leq k$, where $U_l(x)$ is the second kind of Chebyshev polynomials.
\end{lemma}

\stitle{Proof of Lemma \ref{lem:deviation_chebyshev}.}
    The key observation here is that $d_l$ satisfy the three term recurrence, and the recurrence coefficients are the same as Chebyshev polynomials. To formalize this observation, we state the following Fact:
    \begin{fact}\label{fact:dl_recurrence}
        $d_{l+1}=2 \mathcal{A}d_l-d_{l-1} - \xi_l$ for $l\leq k-1$.
    \end{fact}

\begin{proof}
By the difinition of $d_{l+1}$ and Proposition \ref{prop:lanczos_push_guarantee}, we have:
\begin{align*}
    d_{l+1}&= \hat{\mathbf{V}}T_{l+1}(\hat{\mathbf{T}})\mathbf{e}_1- T_{l+1}(\mathcal{A})\mathbf{v}_1\\
    &= \hat{\mathbf{V}}
    (2\hat{\mathbf{T}}T_{l}(\hat{\mathbf{T}})-T_{l-1}(\hat{\mathbf{T}}))
  \mathbf{e}_1
    -(2\mathcal{A}T_{l}(\mathcal{A})-T_{l-1}(\mathcal{A}))
    \mathbf{v}_1 \\
    &=2(\mathcal{A}\hat{\mathbf{V}}-\hat{\beta}_{k+1}\hat{\mathbf{v}}_{k+1}\mathbf{e}_k^T-\mathbf{E}) T_l(\mathbf{T})\mathbf{e}_1 - 2\mathcal{A}T_k(\mathcal{A})\mathbf{v}_1-d_{l-1}\\
    &=2\mathcal{A}d_{l} - d_{l-1}-2\hat{\beta}_{k+1}\hat{\mathbf{v}}_{k+1}\mathbf{e}_k^TT_l(\mathbf{T})\mathbf{e}_1 - 2\mathbf{E}T_l(\mathbf{T})\mathbf{e}_1
\end{align*}
We notice that $\mathbf{e}_k^TT_l(\mathbf{T})\mathbf{e}_1=0$ for $l\leq k-1$ and $\xi_l=2\mathbf{E}T_l(\hat{\mathbf{T}})\mathbf{e}_1$ by definition. So we have $d_{l+1}=2 \mathcal{A}d_l-d_{l-1} - \xi_l$ for $l\leq k-1$. This completes the proof of Fact \ref{fact:dl_recurrence} . 
\end{proof}

The next observation is that the recurrence of $d_l$ (i.e., Fact \ref{fact:dl_recurrence}) exactly satisfies the recurrence of the second kind of Chebyshev polynomials. Specifically, we claim that $d_l=-U_{l-1}(\mathcal{A})\xi_1-\sum_{j=2}^{l}{U_{l-j}(\mathcal{A})\xi_j}$ exactly satisfy the recurrence stated in Fact \ref{fact:dl_recurrence}.

\begin{fact}\label{fact:dl_value}
    $d_l=-\frac{1}{2}U_{l-1}(\mathcal{A})\xi_1-\sum_{j=2}^{l-1}{U_{l-j}(\mathcal{A})\xi_j}$ for $1\leq l\leq k$.
\end{fact}

\begin{proof}
We use induction method to prove Fact \ref{fact:dl_value}. For $l'=0$, $d_0=0$. For $l'=1$, we have $d_1=-\frac{1}{2}\xi_1$, so Fact \ref{fact:dl_value} clearly holds. We now assume that for $l'\leq l$, Fact \ref{fact:dl_value} holds. Then, for $l'=l+1$, by Fact \ref{fact:dl_recurrence}, we have the following equation:
\begin{align*}
    d_{l+1}&=2 \mathcal{A}d_l-d_{l-1} - \xi_l \\
    &=-\mathcal{A}U_{l-1}(\mathcal{A})\xi_1 -\sum_{j=2}^{l-1}{ 2 \mathcal{A}U_{l-j}(\mathcal{A})\xi_j}\\
    &+\frac{1}{2}U_{l-2}(\mathcal{A})\xi_1+\sum_{j=2}^{l-2}{U_{l-j-1}(\mathcal{A})\xi_j}-\xi_l \\
    &=-\frac{1}{2}U_{l}(\mathcal{A})\xi_1-\sum_{j=2}^{l-2}{( 2 \mathcal{A}U_{l-j}(\mathcal{A})-U_{l-j-1}(\mathcal{A}))\xi_j} -2\mathcal{A}\xi_{l-1}-\xi_l \\
    &=-\frac{1}{2}U_{l}(\mathcal{A})\xi_1-\sum_{j=2}^{l-2}{U_{l-j+1}(\mathcal{A})\xi_j}- U_1(\mathcal{A})\xi_{l-1}-U_0(\mathcal{A})\xi_l \\
    &=-\frac{1}{2}U_{l}(\mathcal{A})\xi_1-\sum_{j=2}^{l}{U_{l+1-j}(\mathcal{A})\xi_j}
\end{align*}
So Fact \ref{fact:dl_value} holds for $l'=l+1$. This finishes the proof of Fact \ref{fact:dl_value}. 
\end{proof}

Lemma \ref{lem:deviation_chebyshev} immediately follows by Fact \ref{fact:dl_value}. 

Recall that in Eq. \ref{equ:lzpush_polynomial_error}, we express this term by the following form:
    \begin{align*}
     \mathbf{v}_1^T p(\mathcal{A})\mathbf{v}_1 - \hat{\mathbf{v}}_1^T\hat{\mathbf{V}}p(\hat{\mathbf{T}})\mathbf{e}_1&=\sum_{l=0}^{k}{c_l \mathbf{v}_1^Td_l}
    \end{align*}
Where $d_l=-\frac{1}{2}U_{l-1}(\mathcal{A})\xi_1-\sum_{j=2}^{l-1}{U_{l-j}(\mathcal{A})\xi_j}$ by Lemma \ref{lem:deviation_chebyshev} and $\xi_l=2\mathbf{E}T_l(\hat{\mathbf{T}})\mathbf{e}_1$ by definition. Again, by definition, $\mathbf{E}=[\delta_1,\delta_2,...,\delta_k]$ with $\delta_i=(AMV(\mathcal{A},\hat{\mathbf{v}}_i)-\mathcal{A}\hat{\mathbf{v}}_i)-\hat{\alpha}_i\hat{\mathbf{v}}_i|_{\mathcal{V}-S_i}-\hat{\beta}_i\hat{\mathbf{v}}_{i-1}|_{\mathcal{V}-S_{i-1}}$ for each $i\leq k$. In our Algorithm \ref{algo:lanczos_local}, we choose $S_i=\{u\in \mathcal{V}:|\hat{\mathbf{v}}_i(u)|> \epsilon d_u\}$. Therefore, to get the upper bound of $d_l$, we should first give the upper bound of $\delta_i$.

\begin{lemma}\label{lem:delta_bound}
    $|\delta_i(u)|\leq 3\epsilon d_u$ for all $i\leq k$, $u\in \mathcal{V}$.
\end{lemma}
\begin{proof}
    Recall that in Line 4-8 of Algorithm \ref{algo:lanczos_local}, we only perform the local matrix-vector multiplication $AMV(\mathcal{A},\hat{\mathbf{v}}_i)$ on edge $(u,v)\in \mathcal{E}$ with $|\hat{\mathbf{v}}_{i}(u)|>\epsilon \sqrt{d_ud_v}$, so the error $ |AMV(\mathcal{A},\hat{\mathbf{v}}_i)(u)-\mathcal{A}\hat{\mathbf{v}}_i(u)|$ only cause from $(u,v)\in \mathcal{E}$ with $|\hat{\mathbf{v}}_{i}(u)|\leq \epsilon \sqrt{d_ud_v}$. Additionally, since we choose $S_i=\{u\in \mathcal{V}:|\hat{\mathbf{v}}_i(u)|>\epsilon d_u\}$ in each iteration $i\leq k$, so $|\hat{\mathbf{v}}_i|_{\mathcal{V}-S_i}(u)|\leq \epsilon d_u$ for all $i\leq k$, $u\in \mathcal{V}$. Therefore, we have:
    \begin{align*}
        |\delta_i(u)|&\leq|AMV(\mathcal{A},\hat{\mathbf{v}}_i)(u)-\mathcal{A}\hat{\mathbf{v}}_i(u)|\\
        &+|\hat{\alpha}_i\hat{\mathbf{v}}_i|_{\mathcal{V}-S_i}(u)|
        +|\hat{\beta}_i\hat{\mathbf{v}}_{i-1}|_{\mathcal{V}-S_{i-1}}(u)| \\
        & \leq \sum_{v\in \mathcal{N}(u)}{\frac{\epsilon \sqrt{d_ud_v}}{\sqrt{d_ud_v}}} +\epsilon d_u +\epsilon d_u 
        \leq 3\epsilon d_u
    \end{align*}
    Where the second inequality holds because $\hat{\alpha}_l$, $\hat{\beta}_l$ bounded by constants by Assumption \ref{assump:lzpush_error}, and we let $\hat{\alpha}_l$, $\hat{\beta}_l\leq 1$ just for simplicity. This proves Lemma \ref{lem:delta_bound}.
\end{proof}
Next, we use Lemma \ref{lem:delta_bound} and the definition $\xi_l=2\mathbf{E}T_l(\hat{\mathbf{T}})\mathbf{e}_1$ to provide the upper bound for $\xi_l$.

\begin{lemma}\label{lem:xi_bound}
    $|\xi_l(u)|\leq 6\sqrt{k}\epsilon d_u$ for all $l\leq k$, $u\in \mathcal{V}$.
\end{lemma}

\begin{proof}
    By definition, $\xi_l=2\mathbf{E}T_l(\hat{\mathbf{T}})\mathbf{e}_1$ and $\mathbf{E}=[\delta_1,\delta_2,...,\delta_k]$. We define $c_l=T_l(\hat{\mathbf{T}})e_1$, then $\Vert c_l \Vert_2=\Vert T_l(\hat{\mathbf{T}})\mathbf{e}_1 \Vert_2\leq \Vert \mathbf{e}_1 \Vert_2=1$. Where the first inequality holds by the fact that $\Vert T_l(\hat{\mathbf{T}}) \Vert_2 \leq 1$ since $\lambda(\hat{\mathbf{T}})\subset [\lambda_{min}(\mathcal{A}), \lambda_{2}(\mathcal{A})]\subset [-1, 1]$ by Assumption \ref{assump:lzpush_error} and $\max_{x\in [-1,1]}{|T_l(x)|}\leq 1$ by the property of Chebyshev polynomials. Therefore, by $\xi_l=2\mathbf{E}c_l$ and $|\delta_i(u)|\leq 3\epsilon d_u$, we have:
    \begin{align*}
        |\xi_l(u)|&=|\sum_{j=1}^{k}{2c_l(j)\delta_j(u)}| \leq 6\epsilon d_u \Vert c_l \Vert_1 \\
        &\leq 6\sqrt{k}\epsilon d_u \Vert c_l \Vert_2 \leq 6\sqrt{k}\epsilon d_u 
    \end{align*}
    This proves Lemma \ref{lem:xi_bound}.
\end{proof}

Next, we use Lemma \ref{lem:xi_bound} and the equation $d_l=-\frac{1}{2}U_{l-1}(\mathcal{A})\xi_1-\sum_{j=2}^{l-1}{U_{l-j}(\mathcal{A})\xi_j}$ from Lemma \ref{lem:deviation_chebyshev} to prove the upper bound for $d_l(s)$ and $d_l(t)$. However, to reach this end, we first need the upper bound on the second kind of Chebyshev polynomials $\{U_l(x)\}_{l\leq k}$. 

\begin{lemma}\label{lem:bound_Ul}
    $\Vert \mathbf{D}^{1/2}U_l(\mathbf{P})\mathbf{e}_u\Vert_1\leq  Ck$ for $u=s,t$, $\forall l\leq k$.
\end{lemma}
\begin{proof}
    We prove this Lemma by a property that states the relationship between $\{T_l(x)\}$ and $\{U_l(x)\}$: $U_l(x)=xU_{l-1}(x)+T_l(x)$. Using this property and by definition $C_1=\max_{u=s,t;l\leq k}{\Vert \mathbf{D}^{1/2}T_l(\mathbf{P})\mathbf{e}_u\Vert_1}$, for $u=s,t$ we have:
    \begin{align*}
        \Vert \mathbf{D}^{1/2}U_l(\mathbf{P})\mathbf{e}_u\Vert_1 &=\Vert \mathbf{D}^{1/2}\mathbf{P}U_{l-1}(\mathbf{P})\mathbf{e}_u+\mathbf{D}^{1/2}T_l(\mathbf{P})\mathbf{e}_u\Vert_1 \\
        &\leq \Vert \mathbf{P} \Vert_1\Vert \mathbf{D}^{1/2}U_{l-1}(\mathbf{P}) \mathbf{e}_u\Vert_1+\Vert \mathbf{D}^{1/2}T_l(\mathbf{P})\mathbf{e}_u\Vert_1 \\
        &=\Vert \mathbf{D}^{1/2}U_{l-1}(\mathbf{P})\mathbf{e}_u \Vert_1+\Vert \mathbf{D}^{1/2}T_l(\mathbf{P})\mathbf{e}_u\Vert_1 \\
        &\leq \sum_{j=1}^{l}{\Vert \mathbf{D}^{1/2}T_j(\mathbf{P})\mathbf{e}_u\Vert_1} \\
        &\leq C_1l\leq C_1k
    \end{align*}
    Where the second inequality holds by recursively repeat this process. This proves Lemma \ref{lem:bound_Ul}.
\end{proof}

Armed with Lemma \ref{lem:bound_Ul}, we are now ready to prove the upper bound of $d_l$.

\begin{lemma}\label{lem:bound_dl}
    $|d_l(u)|\leq 6C_1 k^{5/2} \epsilon \sqrt{d_u}$ for $u=s,t$ and $l\leq k$.
\end{lemma}
\begin{proof}
    By Lemma \ref{lem:deviation_chebyshev}, we have $d_l=-\frac{1}{2}U_{l-1}(\mathcal{A})\xi_1-\sum_{j=2}^{l-1}{U_{l-j}(\mathcal{A})\xi_j}$. Therefore,
    \begin{align*}
        |d_l(s)|&=\left|\frac{1}{2}U_{l-1}(\mathcal{A})\xi_1(s)+\sum_{j=2}^{l-1}{U_{l-j}(\mathcal{A})\xi_j(s)}\right|\\
        &\leq \sum_{j=1}^{l-1}{\left|U_{l-j}(\mathcal{A})\xi_j(s)\right|} \\
        & \leq\sum_{j=1}^{l-1}{\sum_{u\in \mathcal{V}}{\left| \mathbf{e}_s^T U_{l-j}(\mathcal{A})\mathbf{e}_u\right||\xi_j(u)|}} \\
        & \leq  \sum_{j=1}^{l-1}{ \sum_{u\in \mathcal{V}}{\frac{\sqrt{d_s}}{\sqrt{d_u}}\left| \mathbf{e}_u^T U_{l-j}(\mathbf{P})\mathbf{e}_s\right||\xi_j(u)|}} \\
        & \leq 6\sqrt{k}\epsilon \sum_{j=1}^{l-1}{ \sum_{u\in \mathcal{V}}{\sqrt{d_sd_u}\left| \mathbf{e}_u^T U_{l-j}(\mathbf{P})\mathbf{e}_s\right|}} \\
        & \leq 6\sqrt{k}\epsilon \sum_{j=1}^{l-1}{\sqrt{d_s} \Vert \mathbf{D}^{1/2}U_{l-j}(\mathbf{P})\mathbf{e}_s\Vert_1}\leq 6C_1 k^{5/2} \epsilon \sqrt{d_s}
    \end{align*}
    Where the last inequality holds by $\Vert \mathbf{D}^{1/2}U_l(\mathbf{P})\mathbf{e}_s\Vert_1\leq C_1 k $ and $l\leq k$. By the same argument, we can also prove that $|d_l(t)|\leq 6C_1 k^{5/2} \epsilon \sqrt{d_t}$. This finishes the proof of Lemma \ref{lem:bound_dl}.
\end{proof}

Finally, by $\mathbf{v}_1^T p(\mathcal{A})\mathbf{v}_1 - \hat{\mathbf{v}}_1^T\hat{\mathbf{V}}p(\hat{\mathbf{T}})\mathbf{e}_1=\sum_{l=0}^{k}{c_l \mathbf{v}_1^Td_l}$, and $\mathbf{v}_1=\hat{\mathbf{v}}_1=\frac{\mathbf{e}_s}{\sqrt{d_s}}-\frac{\mathbf{e}_t}{\sqrt{d_t}}/\sqrt{\frac{1}{d_s}+ \frac{1}{d_t}}$, we have the following inequality:
\begin{align*}
    |\mathbf{v}_1^T p(\mathcal{A})\mathbf{v}_1 &- \hat{\mathbf{v}}_1^T\hat{\mathbf{V}}p(\hat{\mathbf{T}})\mathbf{e}_1|=|\sum_{l=0}^{k}{c_l \mathbf{v}_1^Td_l}|\\
    &\leq \sum_{l=0}^{k}{c_l \left|\frac{d_l(s)}{\sqrt{d_s}}-\frac{d_l(t)}{\sqrt{d_t}}\right| /\sqrt{\frac{1}{d_s}+ \frac{1}{d_t}}} \\
    &\leq 6C_1 k^{5/2} \epsilon\sum_{l=0}^{k}{c_l  /\sqrt{\frac{1}{d_s}+ \frac{1}{d_t}}} \\
    &\leq 6C_1 k^{5/2} \epsilon \sum_{l=0}^{k}{2c_l\sqrt{d}} \\
    & = 12C_1 k^{5/2} \epsilon  \left(\sum_{l=0}^{k}{c_l}\right)\sqrt{d}
\end{align*}
Where $d=\min\{d_s,d_t\}$ by our definition. Next, by $\sum_{l=0}^{k}{c_l}\leq \kappa \log (\frac{\kappa}{\epsilon})$ and $k=\sqrt{\kappa}\log (\frac{\kappa}{\epsilon})$, $|\mathbf{v}_1^T p(\mathcal{A})\mathbf{v}_1 - \hat{\mathbf{v}}_1^T\hat{\mathbf{V}}p(\hat{\mathbf{T}})\mathbf{e}_1|$ can be bounded by:
\begin{align*}
|\mathbf{v}_1^T p(\mathcal{A})\mathbf{v}_1 - \hat{\mathbf{v}}_1^T\hat{\mathbf{V}}p(\hat{\mathbf{T}})\mathbf{e}_1|
&\leq 12 C_1 \kappa^{2.25} \epsilon \sqrt{d} \log^{7/2} (\frac{\kappa}{\epsilon})\\
&=\tilde{O}(C_1\kappa^{2.25} \epsilon \sqrt{d})
\end{align*}
Therefore, by setting $\epsilon=\tilde{\Omega}(\frac{\epsilon'\sqrt{d}}{\kappa^{2.25}C_1})$, we can make $|\mathbf{v}_1^T p(\mathcal{A})\mathbf{v}_1 - \hat{\mathbf{v}}_1^T\hat{\mathbf{V}}p(\hat{\mathbf{T}})\mathbf{e}_1|\leq \frac{\epsilon'}{4}d$. Finally, by Eq. \ref{equ:lzpush_bound}, we have:
\begin{align*}
    |\hat{r}_\mathcal{G}(s,t)-r_\mathcal{G}(s,t)|/(\frac{1}{d_s}+ \frac{1}{d_t}) & \leq |\mathbf{v}_1^T (\mathbf{I}-\mathcal{A})^\dagger \mathbf{v}_1 - \mathbf{v}_1^T p(\mathcal{A}) \mathbf{v}_1 | \\
    &+ | \mathbf{v}_1^T p(\mathcal{A}) \mathbf{v}_1 -\hat{\mathbf{v}}_1^T \hat{\mathbf{V}}p(\hat{\mathbf{T}})\mathbf{e}_1  | \\
    &+ |  \hat{\mathbf{v}}_1^T \hat{\mathbf{V}} p(\hat{\mathbf{T}})\mathbf{e}_1 -  \mathbf{e}_1^T (\mathbf{I}-\hat{\mathbf{T}})^{-1}\mathbf{e}_1| \\
    & \leq \frac{\epsilon'}{4}+\frac{\epsilon'}{4}d
\end{align*}
And equivalently,
\begin{align*}
    |\hat{r}_\mathcal{G}(s,t)-r_\mathcal{G}(s,t)| \leq (\frac{\epsilon'}{4}+\frac{\epsilon'}{4}d)\frac{2}{d} \leq \epsilon'
\end{align*}
This finishes the proof.\hfill\qed

%
\subsection{Proof of Theorem \ref{thm:lanczos_push_runtime}.}

\begin{theorem}{(Theorem \ref{thm:lanczos_push_runtime} restate)}
    The total operations by Algorithm \ref{algo:lanczos_local} is bounded by $O(\frac{1}{\epsilon}\sum_{i\leq k}{(\Vert \hat{\mathbf{v}}_i\Vert_1+\Vert \mathcal{A}\hat{\mathbf{v}}_i^+\Vert_1+\Vert \mathcal{A}\hat{\mathbf{v}}_i^-\Vert_1)}+k^{2.37})$. When setting $k=\sqrt{\kappa}\log \frac{\kappa}{\epsilon'}$, $\epsilon=\tilde{\Omega}(\frac{\epsilon'\sqrt{d}}{\kappa^{2.25}C_1})$, this is $\tilde{O}(\kappa^{2.75}C_1C_2/(\epsilon'\sqrt{d}))$.
    \end{theorem}
    
We provide the proof for the time complexity of the Lanczos Push. Recall that in each iteration $i\leq k$, we perform the subset Lanczos recurrence: $\hat{\beta}_{i+1}\hat{\mathbf{v}}_{i+1}=AMV(\mathcal{A},\hat{\mathbf{v}}_i)-\hat{\alpha}_i\hat{\mathbf{v}}_i|_{S_i}-\hat{\beta}_i\hat{\mathbf{v}}_{i-1}|_{S_{i-1}}$. So the time complexity analysis for $i^{th}$ step is split into two parts: (i) bound the operation numbers for computing $AMV(\mathcal{A},\hat{\mathbf{v}}_i)$; (ii) bound the operation numbers for computing $\hat{\alpha}_i,\hat{\beta}_i$ and other addition/subtraction operations.

\stitle{Step 1.} For the computation of $AMV(\mathcal{A},\hat{\mathbf{v}}_i)$ (Line 4-8 in Algorithm \ref{algo:lanczos_local}), recall we only perform the matrix-vector multiplication $\mathcal{A}\hat{\mathbf{v}}_i$ on edge $(u,v)$ with $|\hat{\mathbf{v}}_{i}(u)|>\epsilon \sqrt{d_ud_v}$ and $\Vert \mathcal{A}\hat{\mathbf{v}}_i \Vert_1 \leq C'$. For our analysis, we define $E_i=\{(u,v)\in \mathcal{E}:\hat{\mathbf{v}}_i(u)\geq \epsilon\sqrt{d_ud_v}\}$ be the subset of edges that has been searched for the computation of $AMV(\mathcal{A},\hat{\mathbf{v}}_i)$. Moreover, we split $\hat{\mathbf{v}}_i=\hat{\mathbf{v}}_i^+-\hat{\mathbf{v}}_i^-$, where $\hat{\mathbf{v}}_i^+$ and $\hat{\mathbf{v}}_i^-$ denotes the positive and negative parts of $\hat{\mathbf{v}}_i$, respectively. Similarly, we define $E_i^+=\{(u,v)\in \mathcal{E}:\hat{\mathbf{v}}_i^+(u)\geq \epsilon\sqrt{d_ud_v}\}$ and $E_i^-=\{(u,v)\in \mathcal{E}:\hat{\mathbf{v}}_i^-(u)\geq \epsilon\sqrt{d_ud_v}\}$. Clearly $|E_i|=|E_i^+|+|E_i^-|$. Next we bound the size of $E_i^+$ and $E_i^-$ respectively. We note that:
\begin{align*}
    \Vert \mathcal{A}\hat{\mathbf{v}}_i^+\Vert_1&=\sum_{u\in\mathcal{V}}{|\mathcal{A}\hat{\mathbf{v}}_i^+(u)|}=\sum_{u\in\mathcal{V}}{\left|\sum_{v\in \mathcal{N}(u)}{\frac{\hat{\mathbf{v}}_i^+(u)}{\sqrt{d_ud_v}}}\right|}\\
    &=\sum_{u\in\mathcal{V}}{\sum_{v\in \mathcal{N}(u)}{\frac{\hat{\mathbf{v}}_i^+(u)}{\sqrt{d_ud_v}}}}=2\sum_{(u,v)\in \mathcal{E}}{\frac{\hat{\mathbf{v}}_i^+(u)}{\sqrt{d_ud_v}}}\\
    &\geq 2\sum_{(u,v)\in E_i^+}{\frac{\hat{\mathbf{v}}_i^+(u)}{\sqrt{d_ud_v}}}\geq 2\epsilon |E_i^+|.
\end{align*}
So $|E_i^+|\leq O(\frac{\Vert \mathcal{A}\hat{\mathbf{v}}_i^+\Vert_1}{\epsilon})$. Similarly, $|E_i^-|\leq O(\frac{\Vert \mathcal{A}\hat{\mathbf{v}}_i^-\Vert_1}{\epsilon})$. So the total operations for computing $AMV(\mathcal{A},\hat{\mathbf{v}}_i)$ is $|E_i|=|E_i^+|+|E_i^-|\leq O(\frac{\Vert \mathcal{A}\hat{\mathbf{v}}_i^+\Vert_1+\Vert \mathcal{A}\hat{\mathbf{v}}_i^-\Vert_1}{\epsilon})$.

\stitle{Step 2.} We observe that the number of the other operations (Line 9-17 in Algorithm \ref{algo:lanczos_local}) can be bounded by $O(|S_i|+|S_{i-1}| + supp\{\hat{\mathbf{v}}_{i+1}\})$, since we only work on the subsets $S_i,S_{i-1}$ and $ supp\{\hat{\mathbf{v}}_{i+1}\}$, where $ supp\{\hat{\mathbf{v}}_{i+1}\}$ denotes the set of non-zero entries of $\hat{\mathbf{v}}_{i+1}$. We set $S_i= \{u\in \mathcal{V}: |\hat{\mathbf{v}}_i(u)|> \epsilon d_u\}$, so we can upper bound $|S_i|\leq \frac{\Vert \hat{\mathbf{v}}_i\Vert_1}{\epsilon}$. Similarly, $|S_{i-1}|\leq \frac{\Vert \hat{\mathbf{v}}_{i-1}\Vert_1}{\epsilon}$. In addition, we observe that $ supp\{\hat{\mathbf{v}}_{i+1}\}\leq |E_i|+|S_i|+|S_{i-1}|$ by the subset Lanczos iteration \ref{equ:subset_lz_recurrence}. Therefore, the total operations of Algorithm \ref{algo:lanczos_local} in each iteration $i$ is bounded by $O(|E_i|+|S_i|+|S_{i-1}| + supp\{\hat{\mathbf{v}}_{i+1}\})=O(\frac{1}{\epsilon}(\Vert \hat{\mathbf{v}}_i\Vert_1+\Vert \mathcal{A}\hat{\mathbf{v}}^+_i\Vert_1+\Vert \mathcal{A}\hat{\mathbf{v}}^-_i\Vert_1))$. So the number of operations of Algorithm \ref{algo:lanczos_local} in total $k$ iterations is bounded by $O(\frac{1}{\epsilon}\sum_{i\leq k}{(\Vert \hat{\mathbf{v}}_i\Vert_1+\Vert \mathcal{A}\hat{\mathbf{v}}_i^+\Vert_1+\Vert \mathcal{A}\hat{\mathbf{v}}_i^-\Vert_1)})$. Finally, the computation of the final output $\left(\frac{1}{d_s}+\frac{1}{d_t}\right)\hat{\mathbf{v}}_1^T \hat{\mathbf{V}} (\mathbf{I}-\hat{\mathbf{T}})^{-1}\mathbf{e}_1$ takes $O(k^{2.37})$ time. This finishes the proof. \hfill\qed

\subsection{Proof of Corollary \ref{coro:lzpush_time_error}}

\begin{corollary}{(Corollary \ref{coro:lzpush_time_error} restate)}
    The runtime of Algorithm \ref{algo:lanczos_local} is $\tilde{O}(\kappa^{2.75}\sqrt{nm}/(\epsilon\sqrt{d}))$ to compute the $\epsilon$-absolute error approximation of $r_\mathcal{G} (s,t) $. Specifically, $C_1\leq \sqrt{m}$ and $C_2\leq 3\sqrt{n}$.
\end{corollary}

First, by the property of Chebyshev polynomials, we have that $\max_{x\in [-1,1]}{|T_i(x)|}=1$ for any $i\in \mathbb{N}^+$. By the fact that the eigenvalues of the probability transition matrix $\mathbf{P}$ relies in $[-1,1]$, we have $\Vert T_i(\mathbf{P})\Vert_2\leq 1$. We denote $\mathbf{x}_i=T_i(\mathbf{P})\mathbf{e}_u$, then clearly $\Vert \mathbf{x}_i\Vert_2=\Vert T_i(\mathbf{P})\mathbf{e}_u\Vert_2\leq1$. By Cauchy-Schwarz inequality, $\Vert \mathbf{D}^{1/2}T_i(\mathbf{P})\mathbf{e}_u\Vert_1=\sum_{u\in \mathcal{V}}{|{d_u}^{1/2}\mathbf{x_i}(u)|} \leq (\sum_{u\in \mathcal{V}}{d_u})^{1/2}(\sum_{u\in \mathcal{V}}{\mathbf{x}_i(u)^2})^{1/2}\leq \sqrt{m}$. As a result, $C_1\leq \sqrt{m}$.

    In each iteration we maintain $\Vert \hat{\mathbf{v}}_i\Vert_2=1$. So again, by Cauchy-Schwarz inequality and by $\Vert \mathcal{A}\Vert_2\leq 1$, we have $\Vert \mathcal{A}\hat{\mathbf{v}}_i^+\Vert_1\leq \sqrt{n}\Vert \mathcal{A}\hat{\mathbf{v}}_i^+\Vert_2\leq \sqrt{n}\Vert \hat{\mathbf{v}}_i^+\Vert_2\leq \sqrt{n}\Vert \hat{\mathbf{v}}_i\Vert_2=\sqrt{n}$, and the same upper bound also holds for $\Vert \mathcal{A}\hat{\mathbf{v}}_i^-\Vert_1$. As a result, $C_2\leq 3\sqrt{n}$. So the runtime bound of Algorithm \ref{algo:lanczos_local} can be simplified by $\tilde{O}(\kappa^{2.75}\sqrt{mn}/(\epsilon\sqrt{d}))$.

\subsection{Proof of Lemma \ref{lem:kappa_G1}}

\begin{lemma}{(Lemma \ref{lem:kappa_G1} restate)}
   If $\mathcal{G}_1$ is constructed in the manner described in section \ref{sec:lower-bounds}, we have $\kappa (\mathcal{G}_1)= \Theta(n)$.
\end{lemma}

   \comment{
   For our analysis, we first provide the definition of graph conductance and the classical Cheeger's inequality.

\begin{definition}
    Given a graph $\mathcal{G}=(\mathcal{V},\mathcal{E})$, the conductance of a subset $S\subset V$ is defined as $\phi(S)\triangleq \frac{|\partial(S)|}{vol(S)}$, where $\partial(S)=\{(u,v)\in \mathcal{E}: u\in S,v\in \mathcal{V}-S \}$ denotes the set of cut edge of the subset $S$, and $vol(S)=\sum_{u\in S}{d_u}$ denotes the volumn of the subset $S$. The conductance of graph $\mathcal{G}$ is defined as $\phi_\mathcal{G}\triangleq\min_{0< vol(S)\leq m/2}{\phi(S)}$.
\end{definition}

The classical Cheeger's inequality ~\cite{cheeger1970lower} shows that the second eigenvalue of the normalized Laplacian matrix $\mathcal{L}$ is closely related to the conductance of the graph $\mathcal{G}$.

\begin{theorem}\label{thm:cheeger}
    The following inequality holds: $\frac{\lambda_2}{2}\leq \phi_\mathcal{G}\leq \sqrt{2\lambda_2}$, where $\lambda_2$ is the second eigenvalue of  $\mathcal{L}$.
\end{theorem}
}

    The proof is inspired by Lemma 4.1 in ~\cite{andoni2018solving}. Recall that the condition number is defined as $\kappa=\frac{2}{\lambda_2}$, where $\lambda_2$ is the second smallest eigenvalue of the normalized Laplacian matrix. First, the upper bound of $\lambda_2$ is given by Cheeger's inequality. By the definition of conductance, we have $\frac{2}{3n}=\phi (V(H_1))\geq\phi_{\mathcal{G}_1}\geq \frac{1}{2}\lambda_2$, where the first equality holds because Remanujan graph is $3$ regular graph and $(s,t)$ is the only cut edge between $V(H_1)$ and $V(H_2)$. Therefore, we have $\lambda_2\leq \frac{4}{3n}=O(\frac{1}{n})$. Next, we prove that $\lambda_2\geq \Omega(\frac{1}{n})$. To reach this end, we consider the Laplacian matrix of graph $\mathcal{G}_1$. By the construction, we have
$L_{\mathcal{G}_1}=L_{H_1}+L_{H_2}+\delta_{s,t}\delta_{s,t}^T$, where $\delta_{s,t}$ is the indicator vector that takes value $1$ at $s$, $-1$ at $t$ and $0$ otherwise. Subsequently, for any vector $x\perp \mathbf{1}$, we split $x$ by $x=x_1+x_2$, where $x_1$ only takes value in $u\in V(H_1)$ and $x_2$ only takes value in $u\in V(H_2)$ respectively. Since $x\perp \mathbf{1}$, we define $m=\frac{2}{n}\sum_{u\in V(H_1)}{x_1(u)}=-\frac{2}{n}\sum_{u\in V(H_2)}{x_2(u)}$. Next, we let $x_1'=x_1-m\mathbb{I}_{V(H_1)}$ and $x_2'=x_2+m\mathbb{I}_{V(H_2)}$ ($\mathbb{I}_{V(H_1)}$ is the indicator vector that takes value $1$ for $u\in V(H_1)$ and $0$ otherwise). Now we suppose the opposite that for any small constant $c$, there exist sufficiently large $n$ such that $\lambda_2(L_{\mathcal{G}_1})\leq \frac{c}{n}$. Then there exists some $x\perp \mathbf{1}$ with $\Vert x \Vert_2=1$, such that $x^T L_{\mathcal{G}_1}x\leq \frac{c}{n}$. Therefore, we can obtain the following equation.
    \begin{equation}
        x^T L_{\mathcal{G}_1}x=x_1^T L_{H_1}x_1+ x_2^T L_{H_2} x_2 + (x(s)-x(t))^2.
    \end{equation}
    Then, the three terms of the right hand side of the above equation all $\leq \frac{c}{n}$. Using the fact that $H_1$ is an expander ($\lambda_2(L_{H_1})\geq \frac{1}{2}$ by $H_1$ Ramanujan graph), we have:
\begin{equation}
    \frac{1}{2}\Vert x_1'\Vert_2^2\leq x_1'^TL_{H_1}x_1'=x_1^TL_{H_1}x_1\leq \frac{c}{n}.
\end{equation}
Thus, $\Vert x_1'\Vert_2\leq \sqrt{\frac{2c}{n}}$. And similarly, $\Vert x_2'\Vert_2\leq \sqrt{\frac{2c}{n}}$. However, since $\Vert x\Vert_2=1$, we have:
\begin{equation}
    1=\Vert x\Vert_2\leq \Vert x_1'\Vert_2+\Vert x_2'\Vert_2+m\sqrt{n}.
\end{equation}
Therefore, we have $m\geq \frac{1}{2\sqrt{n}}$ when setting $c\leq \frac{1}{32}$. However, on the other hand, we have the following inequation.
\begin{equation}\small
    (x(s)-x(t))^2=(x_1'(s)+m-(x_2'(t)-m))^2\geq(2m-2\sqrt{\frac{2c}{n}})^2\geq \frac{1}{4n}.
\end{equation}
which contradicts to $(x(s)-x(t))^2\leq \frac{c}{n}$ since $c\leq \frac{1}{32}$. Therefore, there must exist some constant $c$, such that $\lambda_2(L_{\mathcal{G}_1})\geq \frac{c}{n}$ for any sufficiently large $n$. Therefore, for the normalized Laplacian matrix, we have $\lambda_2(\mathcal{L}_{\mathcal{G}_1})\geq \frac{1}{4}\lambda_2(L_{\mathcal{G}_1})=\Omega(\frac{1}{n})$. Putting it together, we have $\lambda_2=\lambda_2(\mathcal{L}_{\mathcal{G}_1})=\Theta (\frac{1}{n})$. This is equivalent to $\kappa=\Theta (n)$.

\subsection{Proof of Theorem \ref{thm:kappa_lower_bound}}

\begin{theorem}{(Theorem \ref{thm:kappa_lower_bound} restate)}
    Any algorithm that approximates $r_\mathcal{G}(s,t)$ with absolute error $\epsilon\leq 0.01$ with success probability $\geq 0.6$ requires $\Omega(\kappa)$ queries.
\end{theorem}

This statement holds immediately by combining with Theorem \ref{thm:lower_bound_1} and Lemma \ref{lem:kappa_G1}. If there is an algorithm that approximates $r_\mathcal{G}(s,t)$ with absolute error $\epsilon\leq 0.01$, then this algorithm distinguishes $\mathcal{G}_1$ from $\mathcal{G}_2$ (since the $s,t$-RD value of $\mathcal{G}_1$ and $\mathcal{G}_2$ differ at least $0.01$, by Theorem \ref{thm:lower_bound_1}). However, by Theorem \ref{thm:lower_bound_1}, any algorithm that distinguishes $\mathcal{G}_1$ from $\mathcal{G}_2$ requires $\Omega(n)$ queries, and $\kappa=\Theta(n)$ by Lemma \ref{lem:kappa_G1}. So any algorithm that requires $\Omega(\kappa)$ queries to distinguish $\mathcal{G}_1$ from $\mathcal{G}_2$.

\subsection{Proof of Theorem \ref{thm:electric_flow}}

\begin{theorem}{(Theorem \ref{thm:electric_flow} restate)}
    Algorithm \ref{algo:routing} computes $\hat{\mathbf{f}}$ such that $\Vert \mathbf{f}-\hat{\mathbf{f}}\Vert_1\leq \epsilon$ when setting the iteration number $k=O(\sqrt{\kappa}\log \frac{m}{\epsilon})$.
\end{theorem}

We define $\mathbf{\Pi}=\mathbf{I}-\frac{1}{n}\mathbf{11}^T$ as projection matrix on the image of $\mathbf{L}$. By the definition of the potential vector $\bm{\phi}$, we have the following formula using same strategy as Eq. \ref{equ:approx_formula}:
   \begin{equation}
\begin{aligned}
    \bm{\phi}&=\mathbf{L}^{\dagger}(\mathbf{e}_s-\mathbf{e}_t)\\
    &=\mathbf{\Pi}\mathbf{D}^{-1/2}(\mathbf{I}-\mathcal{A})^{\dagger}\mathbf{D}^{-1/2}(\mathbf{e}_s-\mathbf{e}_t)\\
    &=\mathbf{\Pi D}^{-1/2}(\mathbf{I}-\mathcal{A})^{\dagger} (\frac{\mathbf{e}_s}{\sqrt{d_s}}-\frac{\mathbf{e}_t}{\sqrt{d_t}})\\
    &\approx \sqrt{\frac{1}{d_s}+\frac{1}{d_t}}\mathbf{\Pi D}^{-1/2}\mathbf{V}(\mathbf{I}-\mathbf{T})^{-1}\mathbf{V}\mathbf{v}_1\\
    &=\sqrt{\frac{1}{d_s}+\frac{1}{d_t}}\mathbf{\Pi D}^{-1/2}\mathbf{V}(\mathbf{I}-\mathbf{T})^{-1}\mathbf{e}_1
\end{aligned}
\end{equation}
By definition $\mathbf{f}(e)=\bm{\phi}(u)-\bm{\phi}(v)=(\mathbf{e}_u-\mathbf{e}_v)^T\mathbf{L}^\dagger(\mathbf{e}_s-\mathbf{e}_t)$, we can approximate:
\begin{equation}
    \begin{aligned}
        \hat{\mathbf{f}}(e)&=\sqrt{\frac{1}{d_s}+\frac{1}{d_t}}(\mathbf{e}_u-\mathbf{e}_v)^T\mathbf{\Pi D}^{-1/2}\mathbf{V}(\mathbf{I}-\mathbf{T})^{-1}\mathbf{e}_1\\
        &=\sqrt{\frac{1}{d_s}+\frac{1}{d_t}}(\mathbf{e}_u-\mathbf{e}_v)^T\mathbf{ D}^{-1/2}\mathbf{V}(\mathbf{I}-\mathbf{T})^{-1}\mathbf{e}_1
    \end{aligned}
\end{equation}
Which is equivalent to our formula in Line 1-2 in Algorithm \ref{algo:routing}. Using the same way as the proof of Theorem \ref{thm:lanczos_err}, we can control the error $| \hat{\mathbf{f}}(e)- \mathbf{f}(e)|\leq \epsilon/m$ when setting the iteration number $k=O(\kappa \log \frac{\kappa m}{\epsilon})=O(\kappa \log \frac{m}{\epsilon})$. As a result, $\Vert \mathbf{f}-\hat{\mathbf{f}}\Vert_1\leq \epsilon$. This finishes the proof.


\end{document}

%% file: command.tex
\newcommand{\ignore}[1]{}
\newcommand{\nop}[1]{}
\newcommand{\eat}[1]{}
\newcommand{\kw}[1]{{\ensuremath{\mathsf{#1}}}\xspace}
\newcommand{\kwnospace}[1]{{\ensuremath {\mathsf{#1}}}}
\newcommand{\stitle}[1]{\vspace{1ex} \noindent{\bf #1}}
\long\def\comment#1{}
\newcommand{\eop}{\hspace*{\fill}\mbox{$\Box$}}

\newtheorem{property}{Property}
\newtheorem{fact}{Fact}
\newtheorem{assumption}{Assumption}
\newtheorem{claim}{Claim}

\newcommand{\rank}{\kw{rank}}
\newcommand{\push}{\kw{Push}}
\newcommand{\truncatepush}{\kw{Push}}
\newcommand{\hkrelax}{\kw{Hk\ Relax}}
\newcommand{\hkpush}{\kw{Hk\ Push}}
\newcommand{\agp}{\kw{AGP}}
\newcommand{\tea}{\kw{TEA}}
\newcommand{\teaplus}{\kw{TEA+}}
\newcommand{\ppr}{\kw{PPR}}
\newcommand{\ssppr}{\kw{SSPPR}}
\newcommand{\hkpr}{\kw{HKPR}}
\newcommand{\powerpush}{\kw{PwPush}}
\newcommand{\pwpush}{\kw{PowerPush}}
\newcommand{\powerpushsor}{\kw{PwPushSOR}}
\newcommand{\ltwocheb}{\kw{ChebyPower}}

\newcommand{\chebpush}{\kw{ChebyPush}}
\newcommand{\powermethod}{\kw{PM}}
\newcommand{\lanczos}{\kw{Lz}}
\newcommand{\lzpush}{\kw{LzPush}}
\newcommand{\bipush}{\kw{BiPush}}
\newcommand{\geer}{\kw{GEER}}
\newcommand{\bisper}{\kw{BiSPER}}
\newcommand{\fastrd}{\kw{FastRD}}

\newcommand{\Penalty}{\kw{Penalty}}
\newcommand{\plateau}{\kw{Plateau}}

\newcommand{\rw}{\kw{RW}}
\newcommand{\lv}{\kw{LV}}
\newcommand{\lewalk}{\kw{LE\textrm{-}Walk}}

\newcommand{\dblp}{\kw{Dblp}}
\newcommand{\wdblp}{\kw{weight\textrm{-}Dblp}}
\newcommand{\asskitter}{\kw{As\textrm{-}Skitter}}
\newcommand{\orkut}{\kw{Orkut}}
\newcommand{\youtube}{\kw{Youtube}}
\newcommand{\livejournal}{\kw{LiveJournal}}
\newcommand{\wlivejournal}{\kw{weight\textrm{-}LiveJournal}}
\newcommand{\roadca}{\kw{RoadNet\textrm{-}CA}}
\newcommand{\roadpa}{\kw{RoadNet\textrm{-}PA}}
\newcommand{\roadtx}{\kw{RoadNet\textrm{-}TX}}
\newcommand{\powergrid}{\kw{powergrid}}
\newcommand{\pokec}{\kw{Pokec}}
\newcommand{\twitter}{\kw{Twitter}}
\newcommand{\friendster}{\kw{Friendster}}
\newcommand{\er}{\kw{ER}}
\newcommand{\ba}{\kw{BA}}